\documentclass{article} 
\usepackage{iclr2027_conference,times}

\usepackage{hyperref}
\usepackage{url}

\usepackage{amsfonts} 
\usepackage{amsmath}
\usepackage{amssymb}
\usepackage{mathtools}
\usepackage{amsthm}
\usepackage{graphicx}
\usepackage{subcaption}
\usepackage{wrapfig}
\usepackage{algorithm}
\usepackage{algorithmic}

\title{Partial AUC Maximization from \\ Positive-unlabeled Data}

\author{
\textbf{Atsutoshi Kumagai},
\textbf{Tomoharu Iwata},
\textbf{Taishi Nishiyama},
\textbf{Hiroshi Takahashi}, \\
\textbf{Kazuki Adachi},
\textbf{Yasuhiro Fujiwara} \\
NTT, Inc. \\
\texttt{atsutoshi.kumagai@ntt.com}
}

\newtheorem{theorem}{Theorem}[section]
\newtheorem{lemma}[theorem]{Lemma}
\newtheorem{proposition}[theorem]{Proposition}

\newtheorem{assumption}[theorem]{Assumption}
\theoremstyle{remark}

\iclrfinalcopy 
\begin{document}

\maketitle

\lhead{Preprint}

\begin{abstract}
The partial area under the receiver operating characteristic curve (pAUC) is an important performance metric for binary classification that summarizes true positive rates within a specific range of false positive rates (FPRs).
Classifiers that achieve high pAUC need to be obtained in many real-world applications such as cybersecurity, medical care, and advertising. Although many methods for maximizing the pAUC have been proposed, they typically require both labeled positive and negative data for training. However, in practice, labeled negative data are often difficult to collect due to privacy concerns or the need for high expertise to annotate them. 
In this paper, we propose a method for maximizing the pAUC from positive and unlabeled (PU) data without negative data.
Within an empirical risk minimization framework, we show that the pAUC, including its FPR-dependent thresholds, can be represented using only the positive and marginal densities, and derive an empirical estimator from PU data.
A classifier is then trained by maximizing the derived smoothed empirical pAUC estimator.
We experimentally demonstrate the effectiveness of the proposed method with ten real-world datasets.
\end{abstract}

\section{Introduction}
\label{sec:intro}

The area under the receiver operating characteristic (ROC)
curve (AUC) is a widely used evaluation metric for binary classification, especially in class-imbalanced scenarios where positive data are much smaller than negative data
\citep{bradley1997use,cortes2003auc}.
The ROC curve is obtained by plotting the true positive rate (TPR) 
against the false positive rate (FPR) at various threshold settings.
By maximizing the AUC, we can obtain classifiers that achieve high average TPRs over all FPR values from zero to one
\citep{yang2022auc}.

\begin{wrapfigure}[14]{r}{60mm}
  \centering
  \includegraphics[width=0.98\linewidth]{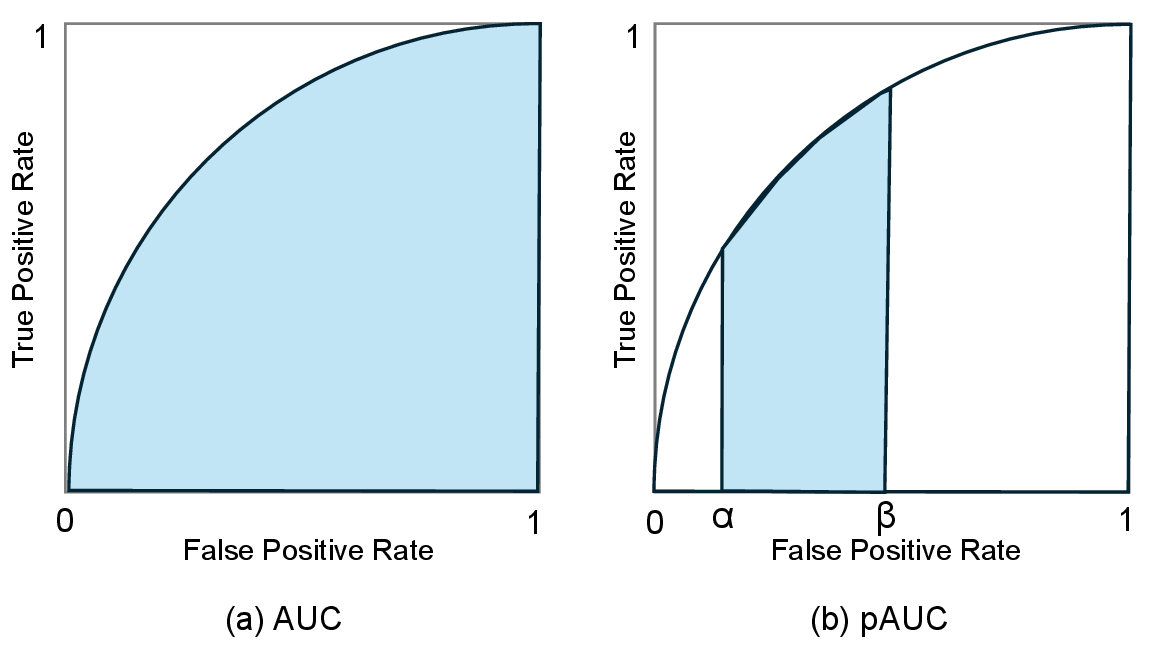}
  \caption{The comparison between the AUC and pAUC: (a) AUC. (b) pAUC within FPR range $[\alpha, \beta]$. }
  \label{fig:pauc}
\end{wrapfigure}
In many real-world applications, 
we are interested in achieving high TPRs within a specific range of FPRs.
For example, in cybersecurity, 
detection systems often need a low FPR since falsely detecting benign software or communications as malicious intensely hinders the work of end users and operators,
but they still require high TPRs to reliably detect truly malicious activities~\citep{nishiyama2024partial}.
In medical care, a low FPR needs to be maintained to avoid unnecessary biopsies, which are invasive and impose physical and psychological burdens on patients as well as financial costs
\citep{baker2001proposed,iwata2020semi}.
In advertising, a low FPR is also important since irrelevant advertisements substantially degrade the user experience.
In such applications, a partial AUC (pAUC), which is the partial area under the ROC curve within a specific FPR range, 
is a more appropriate metric than the AUC
\citep{yang2022auc}.
Figure~\ref{fig:pauc} illustrates the difference between the AUC and pAUC.

Although many methods for maximizing the pAUC have been proposed
\citep{narasimhan2013structural,ueda2018partial,iwata2020semi,yao2022large,zhu2022auc},
they require both labeled positive and negative (PN) data for training.
However, in practice, labeled negative data are often difficult to collect.
For example, in cybersecurity, 
although some malicious data (positive data) can be obtained from public blocklists, 
benign data (negative data) are often unavailable due to privacy reasons, and identifying benign data within unlabeled data requires a high level of expertise
\citep{mirsky2018kitsune,kumagaiauc,kumagaipositive}.
In medical care and advertising, although positive data can be obtained from biopsy-confirmed cases or clicked ads, data that are not biopsy-confirmed cases or clicked ads cannot be regarded as negative because they may contain false negatives, such as early-stage patients incorrectly classified as healthy or advertisements that users are interested in but happened not to click
\citep{bekker2020learning}.
Therefore, such data are better treated as unlabeled data.
In this setting, most existing pAUC maximization methods cannot be applied.

The only exception is the work by Iwata et al.
\citeyearpar{iwata2020semi}, 
which proposed a semi-supervised learning method for maximizing the pAUC and presented an approximation of the pAUC from positive and unlabeled (PU) data.
This method sorts the unlabeled instances by their current classifier scores and heuristically treats all instances within a specified rank range as negatives, assuming that the selected instances are indeed negative.
However, the resulting objective generally does not coincide with the pAUC. Moreover, when this heuristic assumption is violated, its performance can degrade.

In this paper, we propose an approach for maximizing the pAUC from PU data by directly reformulating the pAUC without relying on this heuristic assumption. Our approach builds on the empirical risk minimization (ERM) framework of PU learning,
which allows us to rewrite the risk (i.e., the objective for classifier training) using only the positive and marginal densities
\citep{sugiyama2022machine,du2015convex}.
However, unlike standard classification and AUC objectives used in prior PU learning work
\citep{kiryo2017positive,sakai2018semi,xie2018semi}, simply applying this risk rewriting to the pAUC is insufficient because the pAUC additionally depends on FPR-dependent score thresholds, which are determined by the unknown negative density.
We show that the FPR, and consequently the FPR-dependent thresholds required for the pAUC, can also be represented using only the positive and marginal densities, thereby enabling a direct reformulation of the pAUC.
Based on this reformulation, we derive an empirical estimator of the pAUC from PU data.
By maximizing the derived smoothed empirical pAUC estimator, we can train classifiers that achieve high pAUC even without negative data.
The proposed method can use any differentiable classifiers such as linear classifiers,
kernel-based classifiers, and neural network-based classifiers, which is preferable in practical use.

Our main contributions are summarized as follows: (1) We show that the pAUC, including its FPR-dependent thresholds, can be reformulated using only the positive and marginal densities, and derive an empirical estimator of the pAUC from PU data. 
(2)
We propose a method for learning classifiers that achieve high pAUC by maximizing the derived smoothed empirical pAUC estimator with PU data. (3) We experimentally demonstrate the effectiveness of the proposed method against existing PU learning, AUC maximization, and pAUC maximization methods with ten real-world datasets.

\section{Related Work}
\label{sec:related}

Many methods for maximizing the AUC have been proposed
\citep{yang2022auc,ying2016stochastic,yuan2021large,zhang2023federated,yuan2021compositional}.
Although these methods perform well, particularly on class-imbalanced data
\citep{yuan2021compositional,yuan2021large}, 
they require labeled PN data for training and are not designed for maximizing the pAUC.
Since classifiers achieving high AUC are not necessarily optimal with respect to the pAUC
\citep{narasimhan2013structural,narasimhan2013svmpauctight},
pAUC needs to be directly maximized.

PU learning methods have been proposed for training binary classifiers from only PU data
\citep{bekker2020learning,elkan2008learning}.
A representative approach is the ERM, which rewrites the risk using only PU data~\citep{sugiyama2022machine,du2015convex}.
Some methods use this approach to maximize the AUC with PU data
\citep{sakai2018semi,xie2024weakly,xie2018semi}.
However, extending this approach to the pAUC is nontrivial because, unlike the AUC, the pAUC depends on FPR-dependent score thresholds determined by the negative density. 
Thus, rewriting the pairwise terms alone is insufficient.
We show that these thresholds can also be represented using the positive and marginal densities, enabling the pAUC to be reformulated from PU data.

Many pAUC maximization methods have been proposed
\citep{narasimhan2013structural,narasimhan2013svmpauctight,iwata2020semi,yao2022large,zhu2022auc,ueda2018partial,zhou2025stochastic}.
Although these methods are promising, they require labeled PN data for training.
Iwata et al.
\citeyearpar{iwata2020semi} 
presented an approximation of the pAUC from PU data by heuristically treating all unlabeled instances whose ranks fall within a specified range as negative. However, the resulting objective generally does not coincide with the pAUC. In contrast, the proposed method reformulates the pAUC from PU data without relying on this heuristic. Although Xie et al.
\citeyearpar{xie2024weakly}
also use score-based instance selection, their method targets the full AUC rather than the pAUC.

\section{Preliminaries}
\label{sec:prelim}

In this section, we briefly explain AUC and pAUC.
Let input instance ${\bf x} \in \mathcal{X} $ and its class label $y \in \{-1,+1\}$ be associated with probability density $p({\bf x},y)$,
where $+1$ and $-1$ correspond to PN classes, respectively.
The conditional probability densities for PN classes are denoted as
$p^{{\rm p}} ({\bf x}) := p({\bf x} | y=+1)$ and $p^{{\rm n}} ({\bf x}) := p({\bf x} | y=-1)$.
Let $s: \mathcal{X} \to \mathbb{R} $ be a score function that represents the relative tendency of an instance to be positive.
The classifier is obtained by the score function with threshold $t$: $y= {\rm sign} ( s({\bf x}) - t )$, where ${\rm sign}$ denotes a sign function.

The TPR with threshold $t$ is defined as the probability that the score of positive data is higher than threshold $t$:
\begin{align}
\label{eq:tpr}
{\rm TPR} (t) := \mathbb{E} _{{\bf x} \sim p^{{\rm p}}} \left[ I(s({\bf x}) > t)  \right],
\end{align}
where ${\mathbb E}_{{\bf x} \sim p} \left[ \cdot \right]$ is the expectation over density $p$, and $I(\cdot)$ is the indicator function; $I(Z) = 1$ if $Z$ is true and $I(Z)=0$ otherwise. Similarly, the FPR with threshold $t$ is defined as the probability that the score of negative data is higher than threshold $t$:
\begin{align}
\label{eq:fpr}
{\rm FPR} (t) := \mathbb{E} _{{\bf x} \sim p^{{\rm n}}} \left[ I(s({\bf x}) > t)  \right].
\end{align}
The AUC is defined as the area under the ROC curve plotted by ${\rm TPR}(t)$ against ${\rm FPR}(t)$ with varying threshold $t$. 
Assuming there are no ties, it can be calculated by
\begin{align}
{\rm AUC} := \int_0^1 {\rm TPR} ({\rm FPR}^{-1}(u))du = \mathbb{E}_{{\bf x}^{{\rm p}} \sim p^{{\rm p}}}  \mathbb{E}_{{\bf x}^{{\rm n}} \sim p^{{\rm n}}} \left[ I(s({\bf x}^{{\rm p}}) > s({\bf x}^{{\rm n}})) \right],
\end{align}
where  
\begin{align}
{\rm FPR}^{-1}(u) = {\rm inf} \{ t \in \mathbb{R} \ : \ {\rm FPR}(t) \le u\}
\end{align}
\citep{yang2022auc}.
${\rm FPR}^{-1}(u)$ returns the lower threshold at which the FPR is no greater than $u$.
The AUC can be regarded as the probability that the score of positive data is higher than that of negative data.

The pAUC between $\alpha$ and $\beta$, where $0 \le \alpha < \beta \le 1$, is defined as the normalized area under
the ROC curve plotted by ${\rm TPR}(t)$ against ${\rm FPR}(t)$ where ${\rm FPR}(t)$ is between $\alpha$ and $\beta$ as follows:
\begin{align}
\label{eq:epauc}
{\rm pAUC} (\alpha, \beta) &:= \frac{1}{\beta-\alpha} \int_{\alpha}^{\beta} {\rm TPR} ({\rm FPR}^{-1}(u))du \nonumber\\
&= \frac{1}{\beta-\alpha} \mathbb{E}_{{\bf x}^{{\rm p}} \sim p^{{\rm p}}}  \mathbb{E}_{{\bf x}^{{\rm n}} \sim p^{{\rm n}}} \left[ I(s({\bf x}^{{\rm p}}) > s({\bf x}^{{\rm n}})) I(s({\bf x} ^{{\rm n}}) \in [t_{\beta}, t_{\alpha}]) \right],
\end{align}
where $t_{\alpha} := {\rm FPR}^{-1}(\alpha)$ and $t_{\beta} := {\rm FPR}^{-1}(\beta)$
\citep{yang2022auc}.
The pAUC between $\alpha$ and $\beta$ is the probability that positive data have higher scores than negative data whose score is between $t_{\beta} = {\rm FPR}^{-1}(\beta)$ and $t_{\alpha} = {\rm FPR}^{-1}(\alpha)$.
Due to the normalization factor $1/(\beta-\alpha)$, the pAUC takes values between zero and one.

Given $N^{{\rm p}}$ positive instances $\{ {\bf x}^{{\rm p}} _1, \dots, {\bf x}^{{\rm p}} _{N^{{\rm p}}} \} $ drawn from $p^{{\rm p}} ({\bf x})$ and $N^{{\rm n}}$ negative instances  $\{ {\bf x}^{{\rm n}} _1, \dots, {\bf x}^{{\rm n}} _{N^{{\rm n}}} \} $ drawn from $p^{{\rm n}} ({\bf x})$,
the empirical estimator of the pAUC is calculated as
\begin{align}
{\widehat {\rm pAUC}} (\alpha, \beta) =& \frac{1}{(\beta - \alpha)N^{{\rm p}}N^{{\rm n}}} \sum_{n=1}^{N^{{\rm p}}} \sum_{m=j_{\alpha}+1}^{j_{\beta}} I(s({\bf x} _n^{{\rm p}}) > s({\bf x} _{(m)}^{{\rm n}})),
\end{align}
where $j_{\alpha} := \alpha N^{{\rm n}}$, $j_{\beta} := \beta N^{{\rm n}}$, and 
${\bf x}_{(m)}^{{\rm n}}$ is the negative instance ranked in the $m$-th position among negative data in a descending order of scores $s({\bf x})$\footnote{For simplicity, we assumed that
$j_{\alpha}$ and $j_{\beta}$ are both integers following the previous studies
\citep{yao2022large,narasimhan2013svmpauctight}.}
\citep{dodd2003partial}.
The empirical pAUC is the empirical probability that positive data
have higher scores than negative data that are ranked between $j_{\alpha}$ and $j_{\beta}$.

\section{Proposed Method}
\label{sec:prop}

We first explain our problem setting in Section \ref{subsec:probset}.
Then, we derive the empirical estimator of the pAUC with PU data in Section \ref{subsec:pauc_estimator}.
Lastly, we describe the pAUC maximization method with the derived pAUC estimator in Section \ref{subsec:pauc_maximization}.

\subsection{Problem Setting}
\label{subsec:probset}

Suppose that we are given a set of positive instances $X^{{\rm p}}$ and a set of unlabeled instances $X$:
\begin{align}
&X^{{\rm p}} := \{ {\bf x}^{{\rm p}} _{n} \}_{n=1}^{N ^{{\rm p}}} \overset{\text{i.i.d.}}{\sim} p^{{\rm p}} ({\bf x}), \ \ X := \{ {\bf x} _{n} \}_{n=1}^{N} \overset{\text{i.i.d.}}{\sim} p ({\bf x}) = \pi p^{{\rm p}} ({\bf x}) + (1-\pi) p^{{\rm n}} ({\bf x}), \label{eq:mar_train2}
\end{align}
where $p ({\bf x})$ is the marginal density, and $\pi := p(y=+1)$ is the positive class-prior.
We assume that $\pi$ is known as in the previous PU learning studies
\citep{du2015convex,kiryo2017positive,iwata2020semi}.
When $\pi$ is unknown, it can be estimated from PU data
\citep{scott2015rate,ramaswamy2016mixture,yao2021rethinking},
as also illustrated in our experiments.
Our goal is to learn score function $s$ that can maximize the pAUC between $\alpha$ and $\beta$ from given PU data $X^{{\rm p}} \cup X$.

\subsection{pAUC Estimator with PU Data}
\label{subsec:pauc_estimator}

We derive the empirical estimator of the pAUC with PU data.
First, from the definition of marginal density $p({\bf x}) = \pi p^{{\rm p}} ({\bf x}) + (1-\pi) p^{{\rm n}} ({\bf x})$ in Eq. \eqref{eq:mar_train2},
the negative density can be represented as
\begin{align}
\label{eq:neg_den}
p^{{\rm n}} ({\bf x}) = \frac{1}{1-\pi} \left[ p({\bf x}) - \pi p^{{\rm p}} ({\bf x}) \right].
\end{align}
By substituting Eq. \eqref{eq:neg_den} into Eq. \eqref{eq:epauc}, we obtain
\begin{align}
\label{eq:epauc_pu}
{\rm pAUC} (\alpha, \beta) =& \frac{1}{(\beta-\alpha)(1-\pi)} \left[ \mathbb{E}_{{\bf x}^{{\rm p}} \sim p^{{\rm p}}} \mathbb{E}_{{\bf x} \sim p} \left[ I(s({\bf x}^{{\rm p}}) > s({\bf x})) I(s({\bf x} ) \in [t_{\beta}, t_{\alpha}]) \right] \right. \nonumber\\
& \left. - \pi \mathbb{E}_{{\bf x}^{{\rm p}} \sim p^{{\rm p}}} \mathbb{E}_{{\bar {\bf x}} ^{{\rm p}} \sim p^{{\rm p}}} \left[ I(s({\bf x}^{{\rm p}}) > s({\bar {\bf x}}^{{\rm p}})) I(s({\bar {\bf x}} ^{{\rm p}}) \in [t_{\beta}, t_{\alpha}]) \right] \right].
\end{align}
In Eq. \eqref{eq:epauc_pu}, each expectation depends on positive or marginal densities; they can be approximated with 
PU data $X^{{\rm p}} \cup X$. However, thresholds $t_{\alpha} = {\rm FPR}^{-1}(\alpha)$ and $t_{\beta} = {\rm FPR}^{-1}(\beta)$ still depend
on negative density $p^{{\rm n}} ({\bf x})$.
Therefore, we next represent thresholds $t_{\alpha} = {\rm FPR}^{-1}(\alpha)$ and $t_{\beta} = {\rm FPR}^{-1}(\beta)$ using only the positive and marginal densities.

From Eq. \eqref{eq:neg_den}, the FPR in Eq. \eqref{eq:fpr} can be also rewritten as
\begin{align}
\label{eq:fpr_pu}
& {\rm FPR} (t) = \mathbb{E} _{{\bf x} \sim p^{{\rm n}}} \left[ I(s({\bf x}) > t)  \right] = \frac{1}{1-\pi} \left[ \mathbb{E} _{{\bf x} \sim p} \left[ I(s({\bf x}) > t)  \right] - \pi \mathbb{E} _{{\bf x}^{{\rm p}} \sim p^{{\rm p}}} \left[ I(s({\bf x}^{{\rm p}}) > t)  \right] \right].
\end{align}
Since $t_{\alpha} = {\rm FPR}^{-1}(\alpha)$ and $t_{\beta} = {\rm FPR}^{-1}(\beta)$, 
Eq. \eqref{eq:fpr_pu} implies that these population thresholds are determined only by the positive and marginal densities.
Thus, Eqs. \eqref{eq:epauc_pu} and \eqref{eq:fpr_pu} show that the pAUC, including its thresholds, can be represented using only the positive and marginal densities.
From Eq. \eqref{eq:fpr_pu}, the empirical estimator of the FPR with threshold $t$ is represented as
\begin{align}
\label{eq:fpr_pu_emp}
& {\widehat {\rm FPR}} (t) = \frac{1}{1-\pi} \left[ \frac{1}{N} \sum_{n=1}^{N} I(s({\bf x}_n) > t) -\frac{\pi}{N^{{\rm p}}} \sum_{n=1}^{N^{{\rm p}}} I(s({\bf x}_n^{{\rm p}}) > t) \right].
\end{align}
Since $t_{\alpha} = {\rm FPR}^{-1}(\alpha) = {\rm inf} \{ t \in \mathbb{R} \ : \ {\rm FPR}(t) \le \alpha \}$, 
we want to estimate the lower threshold at which Eq. \eqref{eq:fpr_pu_emp} is no greater than $\alpha$.
Since the value of Eq. \eqref{eq:fpr_pu_emp} changes discretely when threshold $t$ crosses the score of training instance,
we can determine threshold ${\hat t}_{\alpha}$ by simply examining the scores of training PU data $X^{{\rm p}} \cup X$, as shown in Figure~\ref{fig:fpr}. Specifically, threshold ${\hat t}_{\alpha}$ can be determined as
\begin{align}
\label{eq:th_a_pu}
{\hat t}_{\alpha} = {\rm min} \{ s({\bf x}) \ : \  {\widehat {\rm FPR}} (s({\bf x})) \le \alpha, {\bf x} \in X^{{\rm p}} \cup X \}. 
\end{align}

Similarly, threshold ${\hat t}_{\beta}$ can be determined as
\begin{align}
\label{eq:th_b_pu}
{\hat t}_{\beta} = {\rm min} \{ s({\bf x}) \ : \  {\widehat {\rm FPR}} (s({\bf x})) \le \beta, {\bf x} \in X^{{\rm p}} \cup X \}.
\end{align}
Note that when the value of Eq. \eqref{eq:fpr_pu_emp} was less than zero or greater than one, we clipped the value to zero and one, respectively.

By substituting estimated thresholds ${\hat t}_{\alpha}$ and ${\hat t}_{\beta}$ into Eq. \eqref{eq:epauc_pu},
the empirical estimator of the pAUC can be represented as
\begin{align}
\label{eq:emp_pauc_pu}
{\widehat {\rm pAUC}} (\alpha, \beta) =& \frac{1}{(\beta-\alpha)(1-\pi)}\left[ \sum_{n=1}^{N^{{\rm p}}} \sum_{m=1}^{N} \left[ \frac{I(s({\bf x} _n^{{\rm p}}) > s({\bf x}_m)) I(s({\bf x}_m ) \in [{\hat t} _{\beta}, {\hat t} _{\alpha}])}{N N^{{\rm p}}} \right] \right. \nonumber\\
& \left. - \pi \sum_{n \ne m}^{N^{{\rm p}}, N^{{\rm p}}} \left[ \frac{I(s({\bf x}_n^{{\rm p}}) > s({\bf x}_m^{{\rm p}})) I(s({\bf x}_m ^{{\rm p}}) \in [{\hat t}_{\beta}, {\hat t}_{\alpha}])}{N^{{\rm p}} (N^{{\rm p}}-1)} \right] \right],
\end{align}
where we omitted the case of $n=m$ in the second term to avoid biases following the study
\citep{sakai2018semi}.
The finite-sample properties of this estimator are analyzed in Section~\ref{sec:theoretical_analysis}.

\begin{figure}[t]
  \centering
  \includegraphics[width=0.55\linewidth]{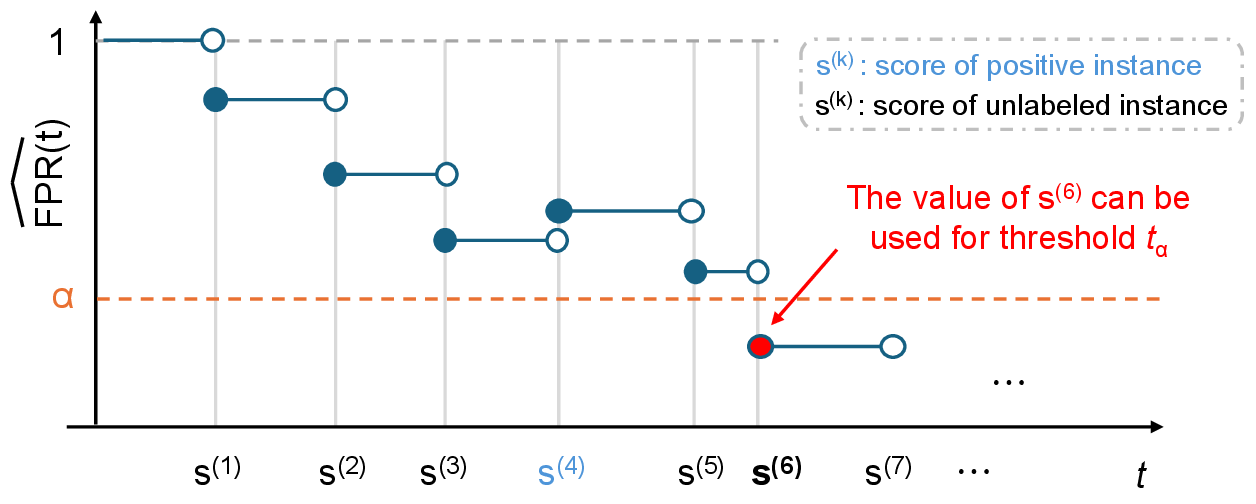}
  \caption{A visual explanation of determining threshold ${\hat t} _{\alpha}$: The vertical and horizontal axes represent ${\widehat {\rm FPR}} (t)$ and threshold $t$, respectively. $s^{(k)}$ denotes the $k$-th smallest score within the scores of all training data $\{ s({\bf x}) \} $, where ${\bf x} \in X^{{\rm p}} \cup X$. From Eq. \eqref{eq:fpr_pu_emp}, ${\widehat {\rm FPR}} (t) = 1$ for any $t$ smaller than $s^{(1)}$ and ${\widehat {\rm FPR}} (t) = 0$ for any $t$ greater than or equal to $s^{(N^{{\rm p}} + N)}$. ${\widehat {\rm FPR}} (t)$ is a discrete function where its value changes only when $t$ crosses one of the score $s^{(k)}$.
  When threshold $t$ exceeds the score of the unlabeled (positive) instance, ${\widehat {\rm FPR}} (t)$ decreases (increases).
  By checking $s^{(1)}, s^{(2)}, \cdots$ in ascending order and identifying the first (minimum) score at which ${\widehat {\rm FPR}} (t) \le \alpha$, we can take the score ($s^{(6)}$ in the figure) as the threshold ${\hat t} _{\alpha}$.}
  \label{fig:fpr}
\end{figure}

\subsection{pAUC Maximization with PU Data}
\label{subsec:pauc_maximization}

We explain the pAUC maximization method using Eq. \eqref{eq:emp_pauc_pu}.
Since the gradient of the indicator function $I(\cdot)$ is zero almost everywhere, 
the empirical pAUC in Eq. \eqref{eq:emp_pauc_pu} cannot be maximized by using gradient descent-based methods.
To make the empirical pAUC differentiable, following the previous studies
\citep{iwata2020semi,kumagaiauc,kumagaipositive},
we used the sigmoid function $\sigma(s({\bf x}) - s({\bf x}')) = 1/(1+{\rm exp}(- s({\bf x}) + s({\bf x}')))$ instead of 
$I(s({\bf x}) > s({\bf x}'))$ in Eq. \eqref{eq:emp_pauc_pu}: 
\begin{align}
\label{eq:emp_pauc_pu_dif}
{\widetilde {\rm pAUC}} (\alpha, \beta) =& \frac{1}{(\beta-\alpha)(1-\pi)}\left[ \sum_{n=1}^{N^{{\rm p}}} \sum_{m=1}^{N} \left[ \frac{\sigma(s({\bf x} _n^{{\rm p}}) - s({\bf x}_m)) I(s({\bf x}_m ) \in [{\hat t} _{\beta}, {\hat t} _{\alpha}])}{N N^{{\rm p}}} \right] \right. \nonumber\\
& \left. -\pi \sum_{n \ne m}^{N^{{\rm p}}, N^{{\rm p}}} \left[ \frac{\sigma (s({\bf x}_n^{{\rm p}}) - s({\bf x}_m^{{\rm p}})) I(s({\bf x}_m ^{{\rm p}}) \in [{\hat t}_{\beta}, {\hat t}_{\alpha}])}{N^{{\rm p}} (N^{{\rm p}}-1)} \right] \right].
\end{align}
Here, we do not replace the indicator function term for instance selection $I(s({\bf x} ) \in [{\hat t} _{\beta}, {\hat t} _{\alpha}])$ with the sigmoid function for simplicity and stability.
This treatment is standard in previous pAUC maximization studies
\citep{ueda2018partial,iwata2020semi,yao2022large}.
Thus, gradients are only propagated through the sigmoid function terms during training.
Note that we can use any surrogate functions such as logistic and ramp functions instead of the sigmoid function
\citep{sakai2018semi}.
Maximizing ${\widetilde {\rm pAUC}} (\alpha, \beta)$ is equivalent to minimizing $1 - {\widetilde {\rm pAUC}} (\alpha, \beta)$.
Therefore,
our loss function for learning $s$ is given as
\begin{align}
\label{eq:abs_loss}
{\mathcal L}(s; \alpha, \beta) = 1 - {\widetilde {\rm pAUC}} (\alpha, \beta).
\end{align}

Algorithm \ref{arg} shows the training procedure of the proposed method with stochastic gradient descent methods.
We first sample PU data from $X^{{\rm p}} \cup X $ (Lines 2--3). 
Then, we calculate thresholds ${\hat t}_{\alpha}$ and ${\hat t}_{\beta}$ by using Eqs. \eqref{eq:th_a_pu} and \eqref{eq:th_b_pu}, respectively (Line 4).
By using the estimated thresholds, 
we calculate the loss in Eq.~\eqref{eq:abs_loss} for learning score function $s$ (Line 5).
Here, estimating the thresholds and evaluating the loss on the same mini-batch is natural, 
because the thresholds themselves are part of the empirical pAUC objective.
Lastly, we update parameters of score function $s$ by using the gradient of the loss (Line 6).

\begin{algorithm}[t]
\caption{Training procedure of the proposed method}
\label{arg}
\begin{algorithmic}[1]
\REQUIRE PU data $X^{{\rm p}} \cup X$, positive class-prior $\pi$, FPR range values $\alpha$ and $\beta$, positive mini-batch size $P$, and unlabeled mini-batch size $U$
\ENSURE Model parameters of score function $s$
\REPEAT
\STATE{Sample unlabeled data with size $U$ from $X$}
\STATE{Sample positive data with size $P$ from $X^{{\rm p}}$}
\STATE{Calculate thresholds ${\hat t}_{\alpha}$ and ${\hat t}_{\beta}$ in Eqs. \eqref{eq:th_a_pu} and \eqref{eq:th_b_pu} with the sampled data}
\STATE{Calculate the loss in Eq. \eqref{eq:abs_loss} on the sampled data}
\STATE{Update parameters of $s$ with the gradient of the loss}
\UNTIL{End condition is satisfied;}
\end{algorithmic}
\end{algorithm}

\section{Experiments}
\label{sec:exp}

\subsection{Data}
\label{subsec:data}

We used eight real-world datasets in the main paper: 
MNIST
\citep{lecun1998gradient}, 
Fashion MNIST (FMNIST)
\citep{xiao2017fashion}, 
SVHN
\citep{netzer2011reading}, 
CIFAR10
\citep{krizhevsky2009learning},
SUSY
\citep{baldi2014searching},
CODRNA
\citep{uzilov2006detection},
IJCNN
\citep{prokhorov2001ijcnn},
and EPSILON
\citep{tsoumakas2008effective}.
These datasets have been commonly used in PU learning studies
\citep{kumagaiauc,kumagaipositive,xie2024weakly,kiryo2017positive,jiang2023positive,sakai2019covariate,hammoudeh2020learning}. 
In Section \ref{sec:medical_advertising},
we used Mammography
\citep{vanschoren2014openml}
and Online-shoppers
\citep{sakar2019real},
which are medical- and advertising-related datasets.

MNIST comprises handwritten images of 10 digit classes, and FMNIST comprises images belonging to 10 fashion categories.
Each image in MNIST and FMNIST is represented as a $28 \times 28$ grayscale image.
SVHN consists of images of 10 printed digits cropped from photographs of house number plates, while CIFAR-10 consists of images of 10 animal and vehicle categories.
Each image in SVHN and CIFAR-10 is represented as a $32 \times 32$ RGB image.
SUSY, CODRNA, IJCNN, and EPSILON are binary tabular datasets taken from the site of LIBSVM
\citep{chang2011libsvm}.
The numbers of features of SUSY, CODRNA, IJCNN, and EPSILON are 18, 8, 22, and 2,000, respectively.
For image datasets, we created binary classification tasks from each dataset following the data creation procedure of previous studies
\citep{kumagaiauc,kumagaipositive,xie2024weakly}.
Specifically, for MNIST and SVHN, we used odd digits as positive and even digits as negative.
For FMNIST, 
we used non-upper garment classes (1, 5, 7, 8, and 9) as positive and the others as negative.
For CIFAR10, we used the vehicle classes (0, 1, 8, and 9) as positive and the animal classes as negative.

For training in each dataset, we used 5,000 unlabeled data and changed the numbers of positive data within $\{ 30, 50, 100\}$.
Additionally, we used 50 positive and 1,000 unlabeled data for validation.
For unlabeled data in training and validation, we changed class-prior $\pi$ within $\{ 0.1, 0.2 \}$.
We used 2,500 positive and 2,500 negative data as test data for evaluation.
Training, validation, and test data did not overlap.
We conducted 10 different experiments for each positive class-prior and labeled positive data size changing the random seeds and evaluated mean test pAUC.

\subsection{Comparison Methods}
\label{subsec:comp}

We compared the proposed method ({Ours) with five methods:
non-negative PU learning method (NNPU)
\citep{kiryo2017positive},
AUC maximization method with PU data (PUAUC)
\citep{xie2018semi,xie2024weakly},
pAUC maximization method (PAUC)
\citep{ueda2018partial},
pAUC maximization method with PU data based on the heuristic scoring assumption (PUpAUCh)
\citep{iwata2020semi},
and 
two-step pAUC maximization method with PU data (2step).
All methods including the proposed method used neural networks for modeling classifiers (score functions).

NNPU is a widely used PU learning method for balanced binary classification, which trains the classifier by minimizing the non-negative classification risk rewritten with PU data.
PUAUC trains the classifier by maximizing the AUC rewritten with PU data.
PAUC maximizes the pAUC by using PU data regarding all unlabeled data as negative.
PUpAUCh maximizes the pAUC from PU data by heuristically treating all unlabeled instances within a specified rank range based on their current scores as negative during training.
2step first trains a classifier by using PUAUC, and then assigns pseudo PN labels to the unlabeled data by using the trained classifier.
Specifically, it selects the top $\pi N$ unlabeled instances according to the classifier scores as pseudo-positive, and assigns the remaining instances as pseudo-negative.
Then, it applied PAUC for the positive and pseudo PN data to train the classifier.
The proposed method, NNPU, PUpAUCh, and 2step used class-prior value $\pi$ for training.
The details of settings such as network architectures are described in section \ref{subsec:settings}.

\begin{table*}[t]
\caption{Test average pAUCs and their standard deviations with $\alpha=0$, $\beta=0.1$ over different labeled positive data sizes $N^{\rm p} = 30$, $50$, and $100$.
Values in bold are not statistically different at the $5\%$ level from the best performing method in each row according to a paired t-test.
`\# Best' row represents the number of results of each method that are the best or comparable to it.}
\label{table:all_results_1}
\centering
\centering
\scalebox{0.78}{
\begin{tabular}{l r r r r r r r}
\hline
Data & $\pi$ & Ours & NNPU & PUAUC & PAUC & PUpAUCh & 2step \\
\hline
MNIST   & 0.1 & \textbf{0.683 $\pm$ 0.074} & 0.522 $\pm$ 0.131  & 0.627 $\pm$ 0.056  & 0.597 $\pm$ 0.070  & 0.675 $\pm$ 0.073  & 0.633 $\pm$ 0.058  \\
        & 0.2 & \textbf{0.633 $\pm$ 0.071} & 0.577 $\pm$ 0.068 & 0.592 $\pm$ 0.054  & 0.473 $\pm$ 0.058 & 0.624 $\pm$ 0.064 & 0.592 $\pm$ 0.054  \\
FMNIST  & 0.1 & \textbf{0.924 $\pm$ 0.025} & 0.756 $\pm$ 0.022 & 0.895 $\pm$ 0.037 & 0.824 $\pm$ 0.064 & \textbf{0.925 $\pm$ 0.028} & 0.907 $\pm$ 0.037 \\
        & 0.2 & \textbf{0.914 $\pm$ 0.028} & 0.770 $\pm$ 0.035 & 0.867 $\pm$ 0.047 & 0.649 $\pm$ 0.119 & \textbf{0.912 $\pm$ 0.036} & 0.873 $\pm$ 0.047 \\
SVHN    & 0.1 & \textbf{0.247 $\pm$ 0.081} & 0.055 $\pm$ 0.004 & 0.208 $\pm$ 0.070 & 0.201 $\pm$ 0.053 & 0.233 $\pm$ 0.078 & 0.202 $\pm$ 0.075 \\
        & 0.2 & \textbf{0.191 $\pm$ 0.073} & 0.054 $\pm$ 0.004 & 0.171 $\pm$ 0.053 & 0.161 $\pm$ 0.032 & 0.152 $\pm$ 0.065 &  0.157 $\pm$ 0.053 \\
CIFAR10 & 0.1 & \textbf{0.478 $\pm$ 0.078} & 0.050 $\pm$ 0.039 & 0.459 $\pm$ 0.078 & 0.362 $\pm$ 0.075 & 0.466 $\pm$ 0.076 & \textbf{0.469 $\pm$ 0.079} \\
        & 0.2 & \textbf{0.433 $\pm$ 0.075} & 0.070 $\pm$ 0.069 & \textbf{0.420 $\pm$ 0.078} & 0.256 $\pm$ 0.043 & \textbf{0.426 $\pm$ 0.075} & \textbf{0.423 $\pm$ 0.079} \\
SUSY    & 0.1 & \textbf{0.454 $\pm$ 0.034} & 0.357 $\pm$ 0.078 & 0.414 $\pm$ 0.051 & 0.366 $\pm$ 0.055 & 0.437 $\pm$ 0.032 & 0.415 $\pm$ 0.052 \\
        & 0.2 & \textbf{0.424 $\pm$ 0.045} & 0.370 $\pm$ 0.058 & 0.389 $\pm$ 0.054 & 0.245 $\pm$ 0.052 & 0.389 $\pm$ 0.047 & 0.385 $\pm$ 0.054 \\
CODRNA  & 0.1 & 0.741 $\pm$ 0.047 & 0.317 $\pm$ 0.266 & 0.749 $\pm$ 0.057 & 0.630 $\pm$ 0.071 & 0.744 $\pm$ 0.050 & \textbf{0.754 $\pm$ 0.054} \\        
        & 0.2 & \textbf{0.716 $\pm$ 0.056} & 0.691 $\pm$ 0.064 & 0.719 $\pm$ 0.062 & 0.495 $\pm$ 0.085 & \textbf{0.720 $\pm$ 0.060} & \textbf{0.726 $\pm$ 0.060} \\
IJCNN   & 0.1 & \textbf{0.620 $\pm$ 0.094} & 0.264 $\pm$ 0.118 & 0.563 $\pm$ 0.087 & 0.525 $\pm$ 0.071 & 0.600 $\pm$ 0.100 & 0.572 $\pm$ 0.090 \\
        & 0.2 & \textbf{0.565 $\pm$ 0.096} & 0.449 $\pm$ 0.088 & 0.525 $\pm$ 0.083 & 0.393 $\pm$ 0.062 & 0.440 $\pm$ 0.194 & 0.524 $\pm$ 0.085 \\
EPSILON   & 0.1 & \textbf{0.193 $\pm$ 0.041} & 0.078 $\pm$ 0.006  & \textbf{0.191 $\pm$ 0.043} & 0.175 $\pm$ 0.034 & \textbf{0.193 $\pm$ 0.044} & \textbf{0.190 $\pm$ 0.046} \\
        & 0.2 & \textbf{0.163 $\pm$ 0.045} & 0.081 $\pm$ 0.010 & \textbf{0.170 $\pm$ 0.039} & 0.139 $\pm$ 0.026 & \textbf{0.162 $\pm$ 0.039} & \textbf{0.172 $\pm$ 0.043} \\
\hline
\# Best   & & 15 & 0 & 3 & 0 & 6 & 6 \\
\hline
\end{tabular}
}
\end{table*}

\subsection{Results}
\label{subsec:results}

Tables \ref{table:all_results_1} and \ref{table:all_results_2} show the average test pAUCs and their standard deviations for two settings $\alpha=0$, $\beta=0.1$, and $\alpha=0$, $\beta=0.2$, respectively.
We used pAUC with $\alpha=0$ since it is often used in previous studies and real-world applications
\citep{narasimhan2013structural,iwata2020semi,zhu2022auc}.
The results with $\alpha=0.1$ and $\beta=0.2$ are also reported in Section \ref{apend:adres}.
The proposed method performed the best or comparably to it in almost all cases (31 out of 32).
NNPU did not work well since it is not designed for both class-imbalanced data and pAUC maximization.
PUAUC outperformed NNPU since maximizing the AUC is appropriate for class-imbalanced data.
However, it was often outperformed by the proposed method by a large margin, which indicates the importance of explicitly maximizing the pAUC.
PAUC did not perform well either even though it aims to maximize the pAUC.
This is because it merely treats all unlabeled data as negative, which leads to very biased results.
PUpAUCh outperformed these methods (NNPU, PUAUC, and PAUC) since it is explicitly designed for maximizing the pAUC with PU data. However, 
it is often outperformed by the proposed method.
Section \ref{apen:iwata_assumption} shows that its heuristic assumption is frequently violated during training.
Although 2step performed comparably to PUpAUCh, its performance was limited by the label noise inherent in the pseudo PN data.
Overall, these results show the effectiveness of our method for maximizing the pAUC from PU data.

\begin{table*}[t!]
\caption{Test average pAUCs and their standard deviations with $\alpha=0$, $\beta=0.2$ over different labeled positive data sizes $N^{\rm p} = 30$, $50$, and $100$.
Values in bold are not statistically different at the $5\%$ level from the best performing method in each row according to a paired t-test.
`\# Best' row represents the number of results of each method that are the best or comparable to it.}
\label{table:all_results_2}
\centering
\scalebox{0.78}{
\begin{tabular}{l r r r r r r r}
\hline
Data & $\pi$ & Ours & NNPU & PUAUC & PAUC & PUpAUCh & 2step \\
\hline
MNIST   & 0.1 & \textbf{0.755 $\pm$ 0.056} & 0.610 $\pm$ 0.127  & 0.725 $\pm$ 0.048  & 0.736 $\pm$ 0.055  & 0.749 $\pm$ 0.055  & 0.730 $\pm$ 0.050  \\
        & 0.2 & \textbf{0.723 $\pm$ 0.065} & 0.677 $\pm$ 0.062 & 0.696 $\pm$ 0.049  & 0.652 $\pm$ 0.062 & \textbf{0.717 $\pm$ 0.056} & 0.697 $\pm$ 0.049  \\
FMNIST  & 0.1 & \textbf{0.943 $\pm$ 0.022} & 0.775 $\pm$ 0.016 & 0.925 $\pm$ 0.028 & 0.926 $\pm$ 0.032 & \textbf{0.940 $\pm$ 0.023} & 0.934 $\pm$ 0.027 \\
        & 0.2 & \textbf{0.934 $\pm$ 0.023} & 0.789 $\pm$ 0.033 & 0.904 $\pm$ 0.034 & 0.853 $\pm$ 0.059 & \textbf{0.931 $\pm$ 0.026} & 0.911 $\pm$ 0.033 \\
SVHN    & 0.1 & \textbf{0.342 $\pm$ 0.092} & 0.105 $\pm$ 0.005 & 0.305 $\pm$ 0.081 & \textbf{0.334 $\pm$ 0.086} & 0.329 $\pm$ 0.089 & 0.313 $\pm$ 0.084 \\
        & 0.2 & \textbf{0.289 $\pm$ 0.090} & 0.104 $\pm$ 0.005 & 0.262 $\pm$ 0.063 & 0.264 $\pm$ 0.052 & 0.221 $\pm$ 0.067 & 0.259 $\pm$ 0.064 \\
CIFAR10 & 0.1 & \textbf{0.610 $\pm$ 0.074} & 0.089 $\pm$ 0.058 & 0.595 $\pm$ 0.072 & 0.569 $\pm$ 0.070 & \textbf{0.600 $\pm$ 0.062} & \textbf{0.608 $\pm$ 0.070} \\
        & 0.2 & \textbf{0.578 $\pm$ 0.075} & 0.116 $\pm$ 0.094 & 0.559 $\pm$ 0.078 & 0.459 $\pm$ 0.070 & \textbf{0.569 $\pm$ 0.066} & \textbf{0.564 $\pm$ 0.077} \\
SUSY    & 0.1 & \textbf{0.553 $\pm$ 0.032} & 0.457 $\pm$ 0.075 & 0.522 $\pm$ 0.041 & 0.528 $\pm$ 0.046 & 0.540 $\pm$ 0.028 & 0.523 $\pm$ 0.042 \\
        & 0.2 & \textbf{0.528 $\pm$ 0.044} & 0.474 $\pm$ 0.051 & 0.499 $\pm$ 0.047 & 0.434 $\pm$ 0.056 & 0.508 $\pm$ 0.039 & 0.500 $\pm$ 0.045 \\
CODRNA  & 0.1 & \textbf{0.853 $\pm$ 0.034} & 0.386 $\pm$ 0.292 & 0.850 $\pm$ 0.041 & 0.830 $\pm$ 0.047 & \textbf{0.855 $\pm$ 0.035} & \textbf{0.856 $\pm$ 0.038} \\
        & 0.2 & \textbf{0.838 $\pm$ 0.039} & 0.791 $\pm$ 0.059 & 0.827 $\pm$ 0.048 & 0.705 $\pm$ 0.072 & \textbf{0.839 $\pm$ 0.039} & \textbf{0.833 $\pm$ 0.047} \\
IJCNN   & 0.1 & \textbf{0.711 $\pm$ 0.079} & 0.367 $\pm$ 0.111 & 0.670 $\pm$ 0.084 & 0.687 $\pm$ 0.073 & 0.687 $\pm$ 0.101 & 0.683 $\pm$ 0.084 \\
        & 0.2 & \textbf{0.659 $\pm$ 0.102} & 0.543 $\pm$ 0.092 & 0.637 $\pm$ 0.087 & 0.573 $\pm$ 0.068 & 0.599 $\pm$ 0.132 & 0.641 $\pm$ 0.089 \\
EPSILON   & 0.1 & \textbf{0.305 $\pm$ 0.058} & 0.128 $\pm$ 0.007 & \textbf{0.298 $\pm$ 0.055} & 0.290 $\pm$ 0.043 & \textbf{0.304 $\pm$ 0.053} & \textbf{0.299 $\pm$ 0.060}  \\
        & 0.2 & \textbf{0.269 $\pm$ 0.046} & 0.132 $\pm$ 0.012 & \textbf{0.260 $\pm$ 0.050} & 0.248 $\pm$ 0.037 & \textbf{0.261 $\pm$ 0.053} & \textbf{0.270 $\pm$ 0.055} \\
\hline
\# Best   & & 16 & 0 & 2 & 1 & 9 & 6 \\
\hline
\end{tabular}
}
\end{table*}

\begin{figure*}[t]
    \centering
    \begin{minipage}{0.24\linewidth}
        \centering
        \includegraphics[width=3.3cm]{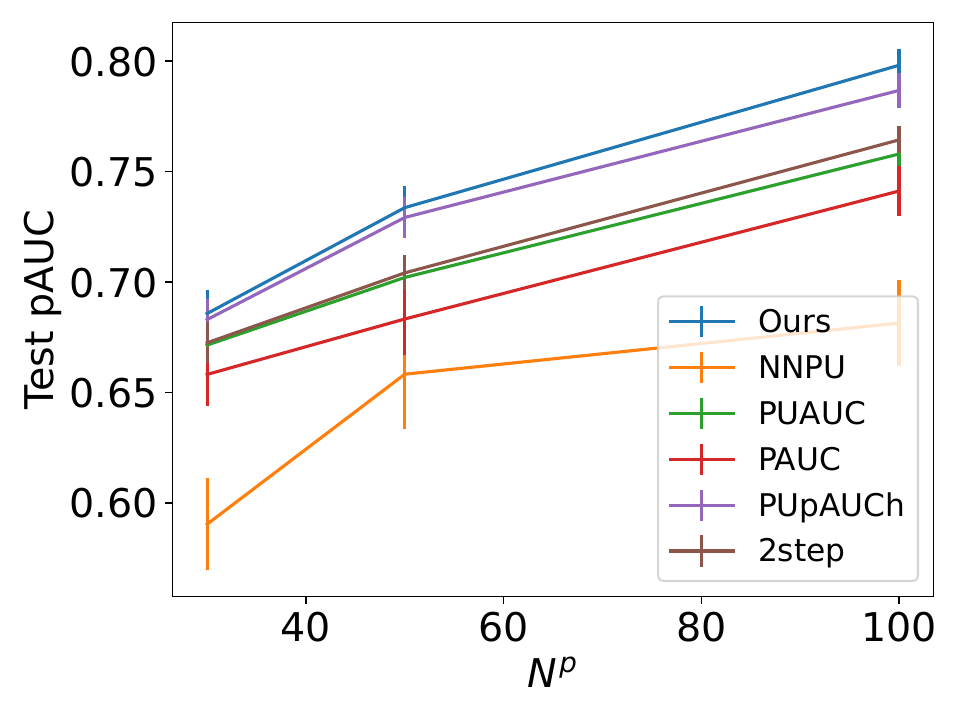}
        \subcaption{MNIST}
    \end{minipage}
    \hfill
    \begin{minipage}{0.24\linewidth}
        \centering
        \includegraphics[width=3.3cm]{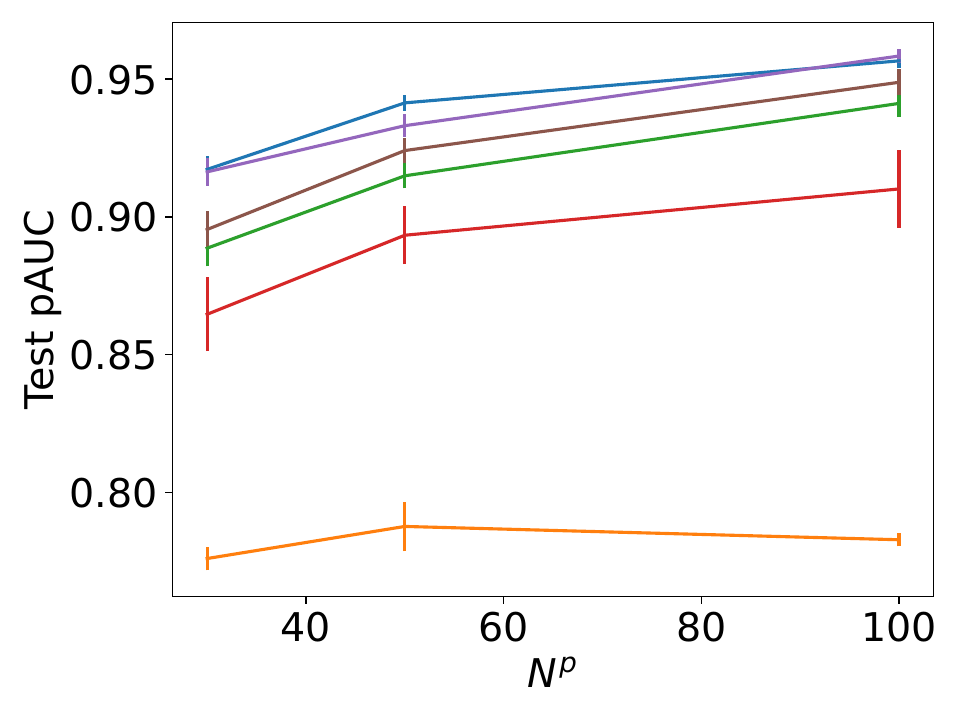}
        \subcaption{FMNIST}
    \end{minipage}
    \hfill
    \begin{minipage}{0.24\linewidth}
        \centering
        \includegraphics[width=3.3cm]{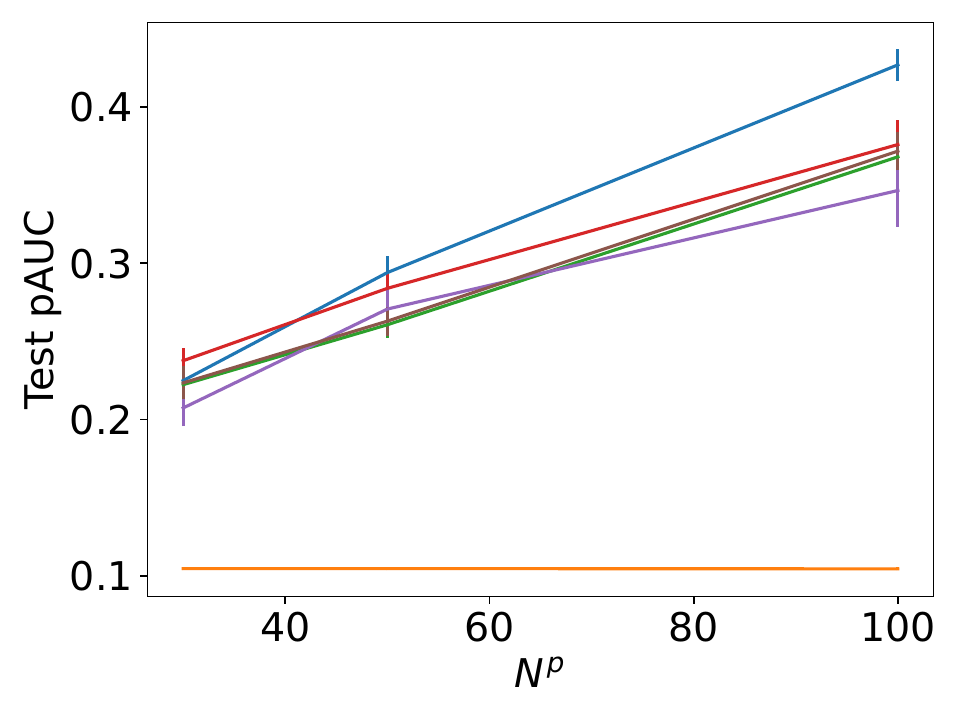}
        \subcaption{SVHN}
    \end{minipage}
    \hfill
    \begin{minipage}{0.24\linewidth}
        \centering
        \includegraphics[width=3.3cm]{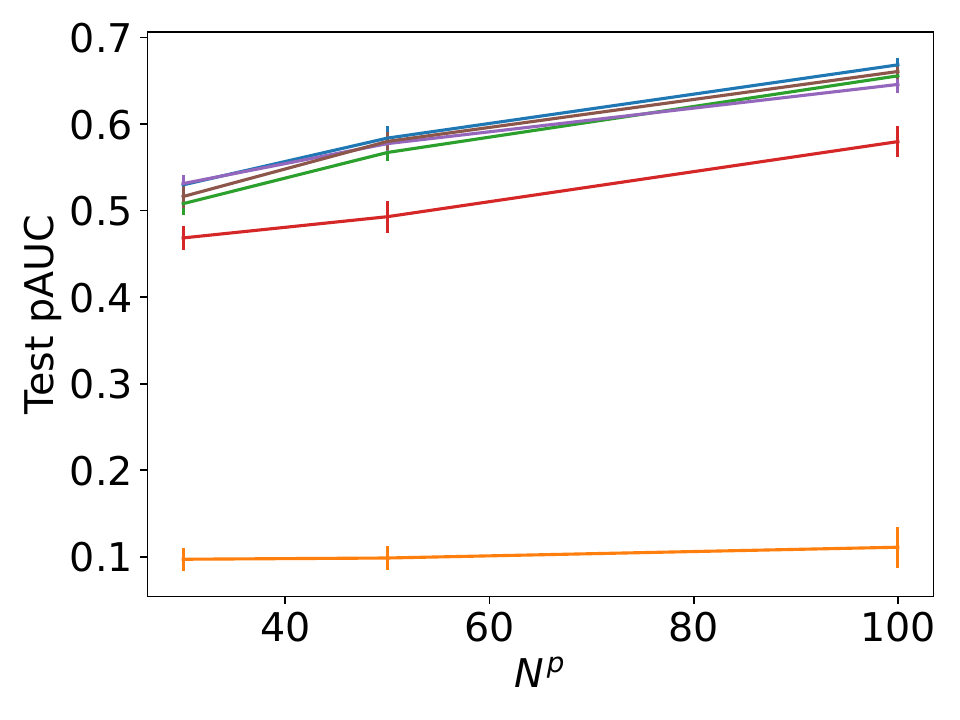}
        \subcaption{CIFAR10}
    \end{minipage}
    \hfill
    \begin{minipage}{0.24\linewidth}
        \centering
        \includegraphics[width=3.3cm]{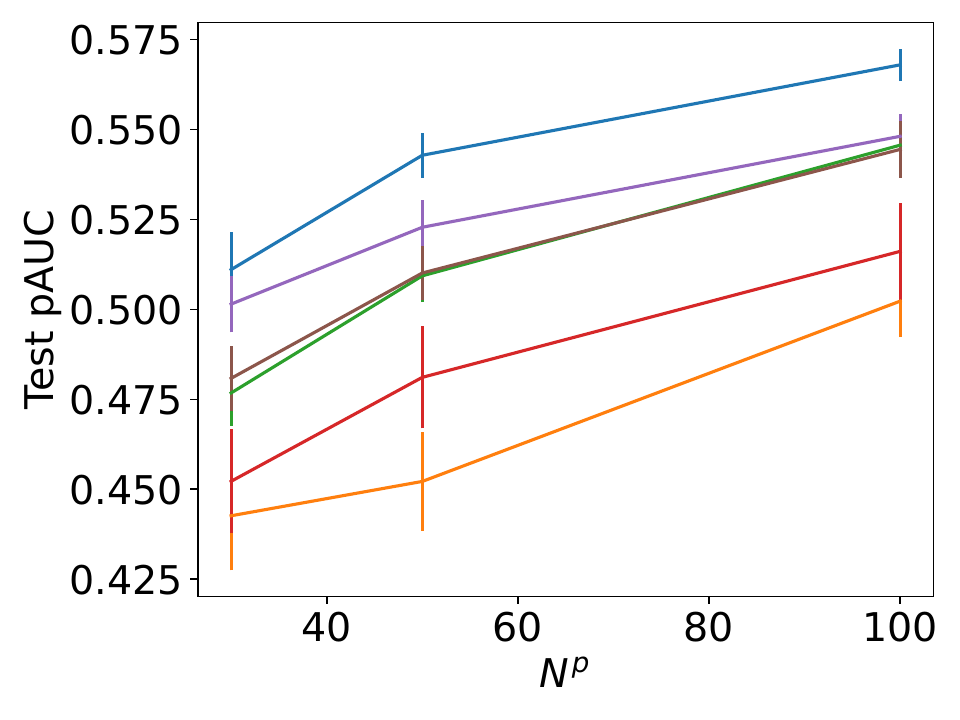}
        \subcaption{SUSY}
    \end{minipage}
    \hfill
    \begin{minipage}{0.24\linewidth}
        \centering
        \includegraphics[width=3.3cm]{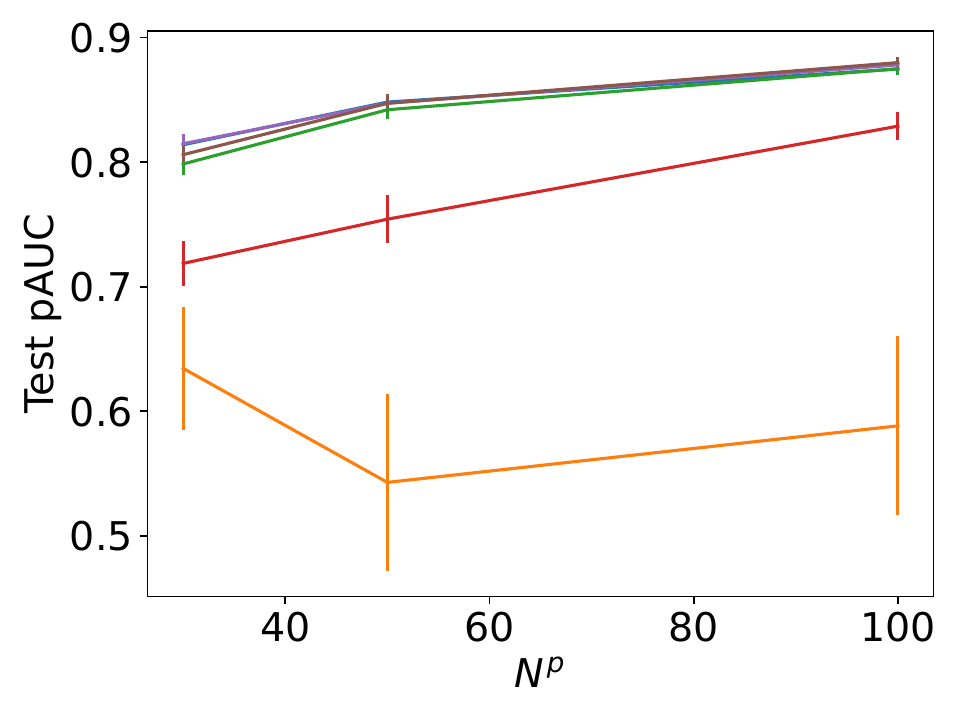}
        \subcaption{CODRNA}
    \end{minipage}
    \hfill
    \begin{minipage}{0.24\linewidth}
        \centering
        \includegraphics[width=3.3cm]{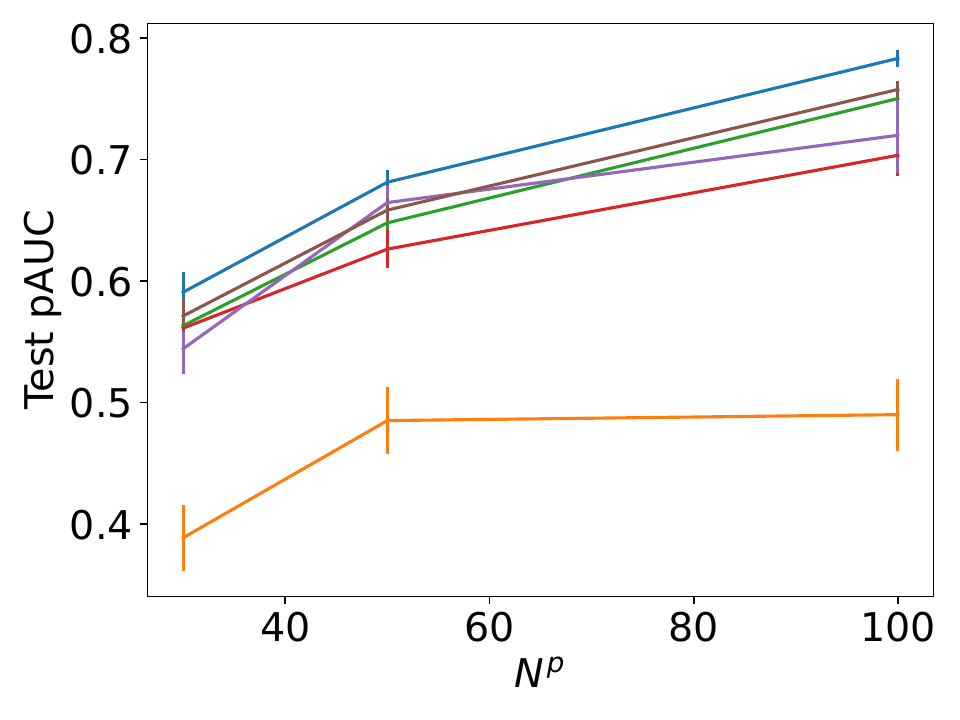}
        \subcaption{IJCNN}
    \end{minipage}
    \hfill
    \begin{minipage}{0.24\linewidth}
        \centering
        \includegraphics[width=3.3cm]{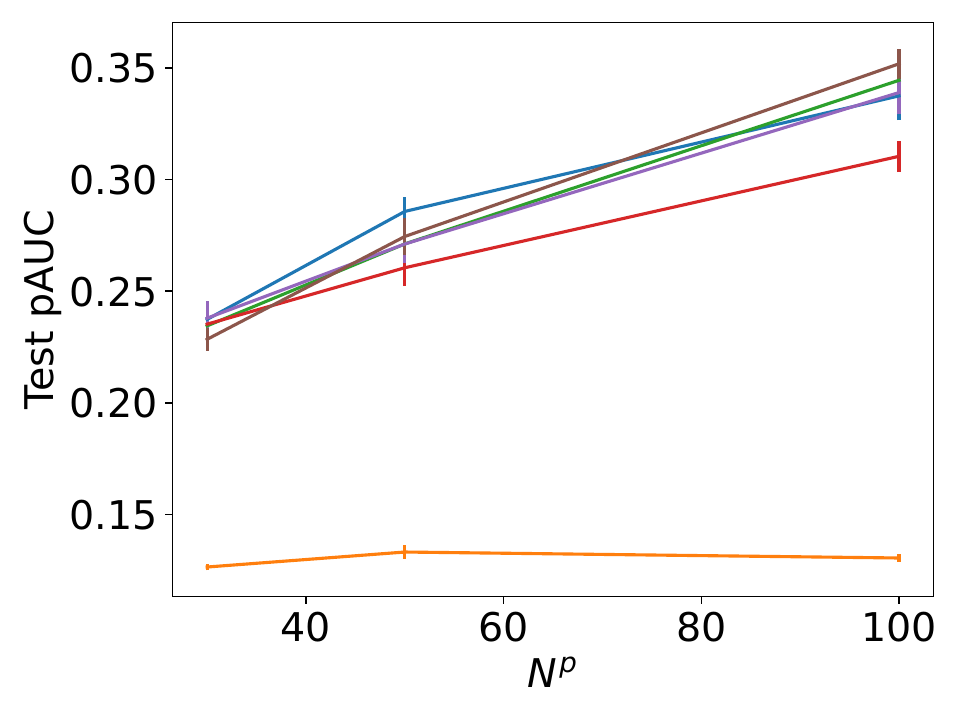}
        \subcaption{EPSILON}
    \end{minipage}
    \caption{The average test pAUCs and their standard errors with $\alpha=0$, $\beta=0.2$ over different class-priors $\pi \in \{0.1, 0.2\}$ when changing the number of labeled positive data $N^{{\rm p}}$.}
    \label{fig:result_pndata}
\end{figure*}

As for the results of each class-prior $\pi$ in Tables \ref{table:all_results_1} and \ref{table:all_results_2}, 
all methods except NNPU tended to perform worse as $\pi$ increased.
This is because these (p)AUC maximization-based methods include a term in their objective functions that encourages the scores of positive data to be higher than those of unlabeled data (e.g., the first term in Eq. \eqref{eq:emp_pauc_pu_dif} for the proposed method).
When $\pi$ is small, the unlabeled data can be regarded as mostly negative and thus the corresponding objective function term becomes more suitable for learning accurate score functions that can rank positive data higher than negative ones.
As a result, the performance tends to improve (degrade) as $\pi$ is decreased (increased).
On the other hand, since increasing $\pi$ reduces the severity of class-imbalance, NNPU tended to perform better.
Nevertheless, since our focus is on the class‐imbalanced problem, these results support the effectiveness of the proposed method. In Section \ref{apend:smaller_larger_priors}, we further confirmed the effectiveness of the proposed method under smaller and larger class-prior values.

Figure \ref{fig:result_pndata} shows the average test pAUCs with their standard errors when changing the number of labeled positive data $N^{{\rm p}}$.
The proposed method consistently performed well across different numbers of labeled positive data.
Since labeled positive data are informative for training, the proposed method tended to perform better as $N^{{\rm p}}$ increased.
In Section \ref{apend:tru}, 
we also evaluated the proposed method with different unlabeled training data sizes $N$ and found that it performed well across different $N$. 

pAUC maximization methods with PU data (the proposed method, PUpAUCh, and 2step) need to know true class-prior $\pi$ for training.
Although we evaluated these methods with true class-priors in our experiments following the previous PU learning studies
\citep{iwata2020semi,kiryo2017positive,du2015convex}, 
there might be some noise in practice.
Therefore, we investigated how the performance of the proposed method changes when using noisy class-priors for training.
Figure \ref{fig:result_noisy} shows the results.
Here, PUAUC is a strong baseline that does not depend on $\pi$ for training.
The proposed method performed well in many cases, even when the class-prior used for training (Input $\pi$) differed from true class-prior $\pi$. These results show that the proposed method robustly works well as long as noise on the class-prior is not so large.

Furthermore, we examined the performance of the proposed method when using the class-priors estimated from PU data.
For this, we used the kernel embedding-based class-prior estimation method~\citep{ramaswamy2016mixture}.
The estimated class-priors of MNIST, FMNIST, SVHN, CIFAR10, SUSY, CODRNA, IJCNN, and EPSILON were $0.214$, $0.208$, $0.096$, $0.382$, $0.226$, $0.159$, $0.113$, and $0.159$, respectively.
Table~\ref{table:all_results_estclass} shows the average test pAUCs and their standard deviations when using the estimated class-priors.
The proposed method still performed well.
These results suggest that the proposed method remains effective and practically applicable even when the class-priors are estimated from data.

\begin{figure*}[t]
    \centering
    \begin{minipage}{0.24\linewidth}
        \centering
        \includegraphics[width=3.3cm]{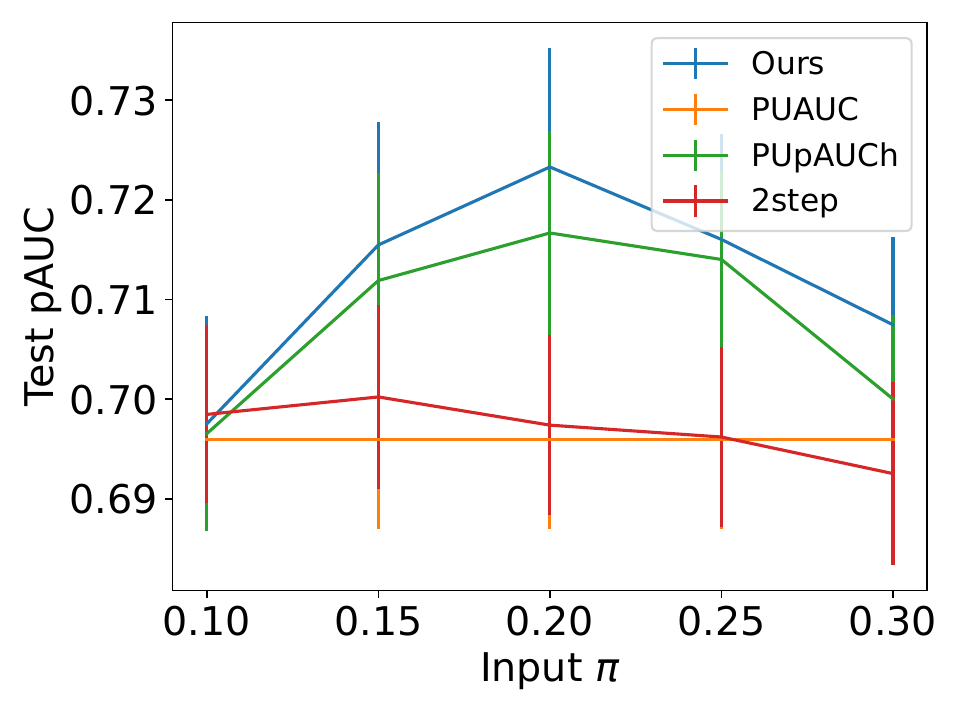}
        \subcaption{MNIST}
    \end{minipage}
    \hfill
    \begin{minipage}{0.24\linewidth}
        \centering
        \includegraphics[width=3.3cm]{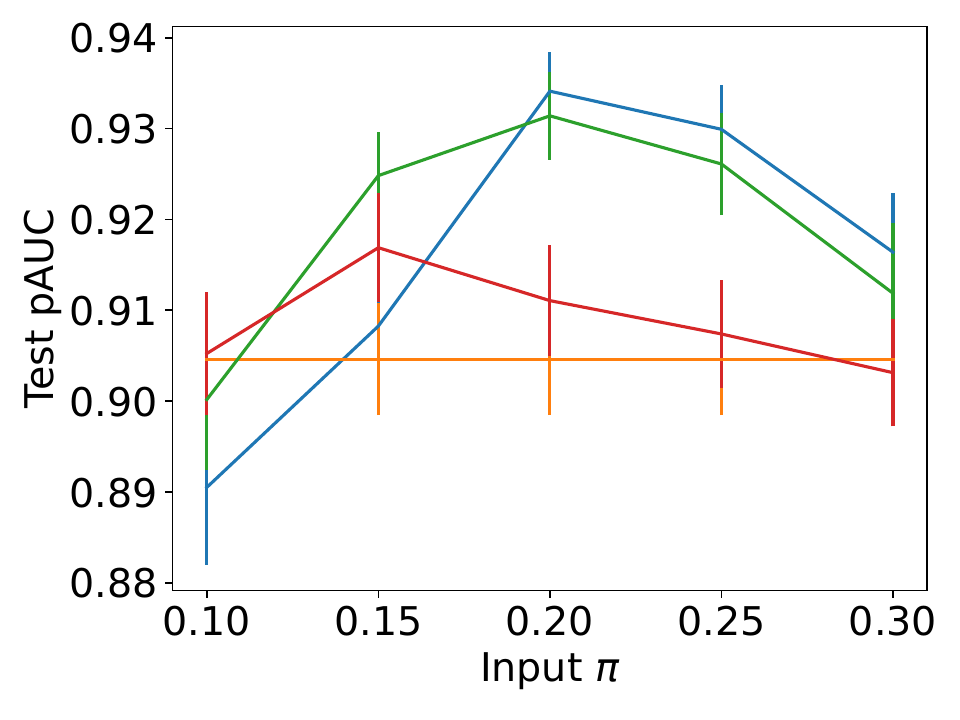}
        \subcaption{FMNIST}
    \end{minipage}
    \hfill
    \begin{minipage}{0.24\linewidth}
        \centering
        \includegraphics[width=3.3cm]{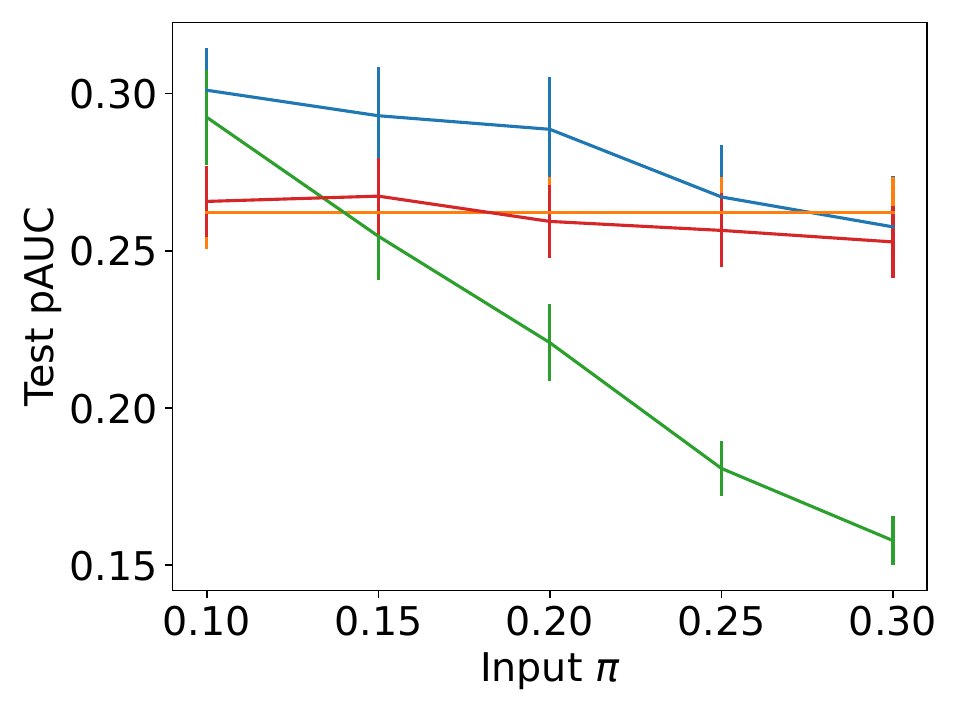}
        \subcaption{SVHN}
    \end{minipage}
    \hfill
    \begin{minipage}{0.24\linewidth}
        \centering
        \includegraphics[width=3.3cm]{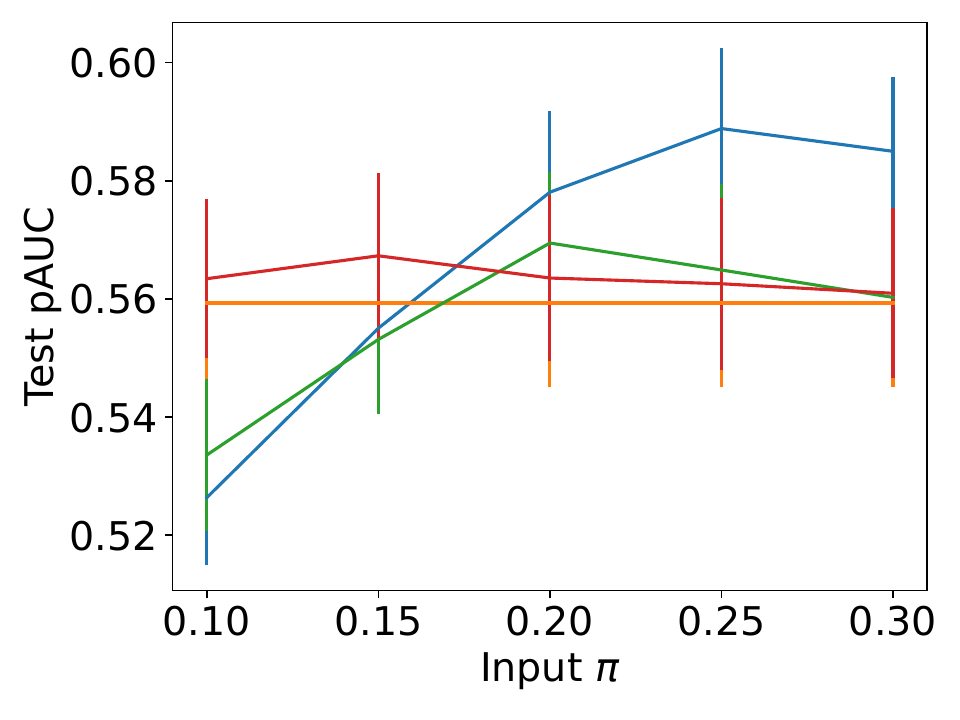}
        \subcaption{CIFAR10}
    \end{minipage}
    \hfill
    \begin{minipage}{0.24\linewidth}
        \centering
        \includegraphics[width=3.3cm]{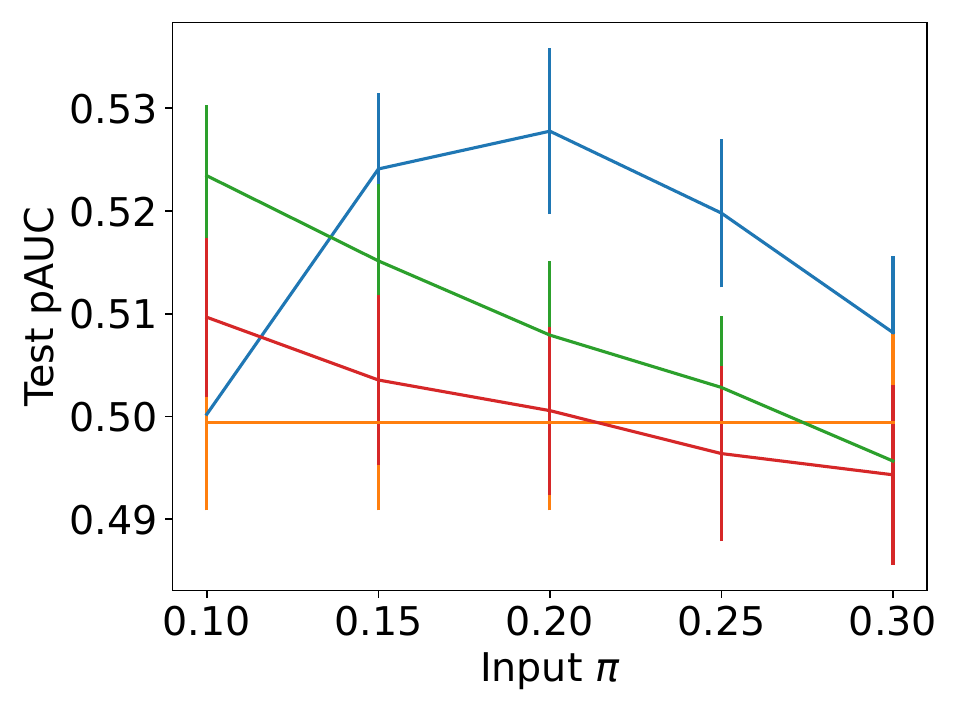}
        \subcaption{SUSY}
    \end{minipage}
    \hfill
    \begin{minipage}{0.24\linewidth}
        \centering
        \includegraphics[width=3.3cm]{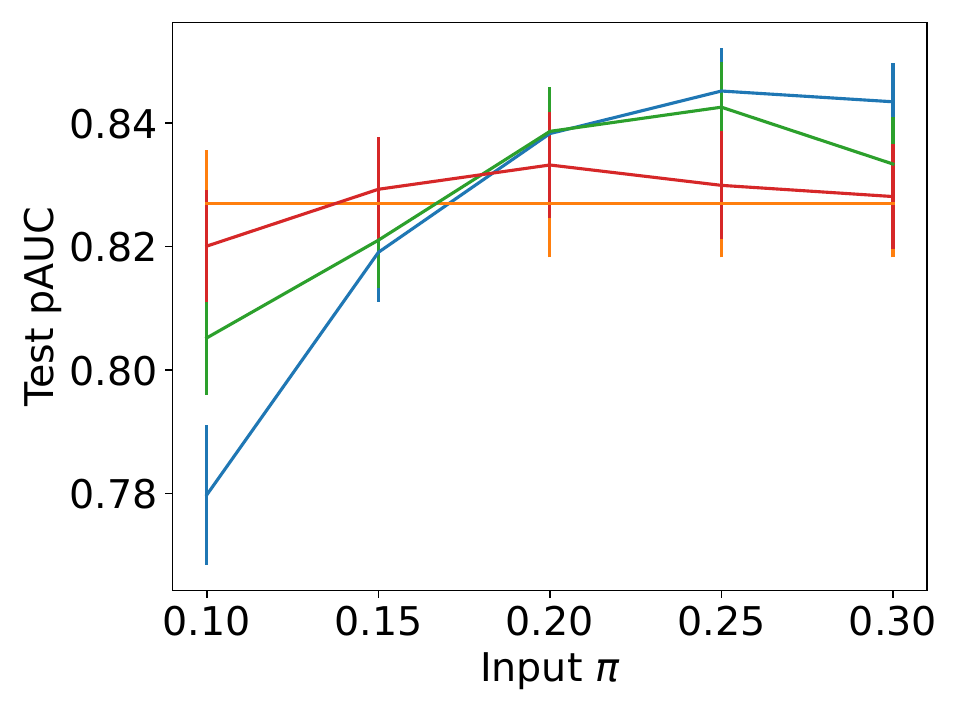}
        \subcaption{CODRNA}
    \end{minipage}
    \hfill
    \begin{minipage}{0.24\linewidth}
        \centering
        \includegraphics[width=3.3cm]{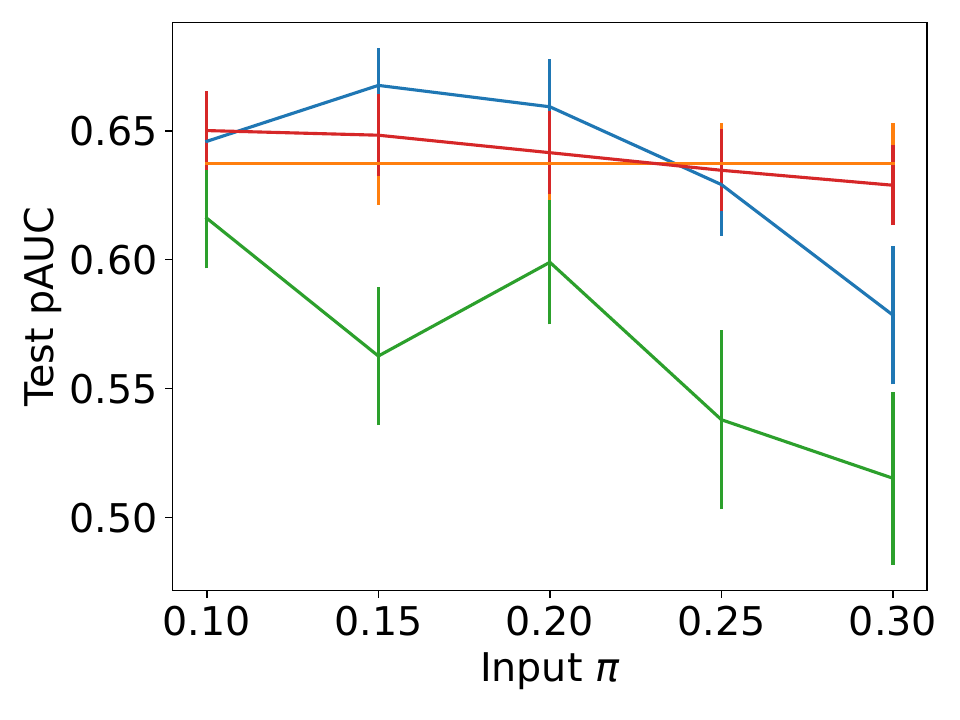}
        \subcaption{IJCNN}
    \end{minipage}
    \hfill
    \begin{minipage}{0.24\linewidth}
        \centering
        \includegraphics[width=3.3cm]{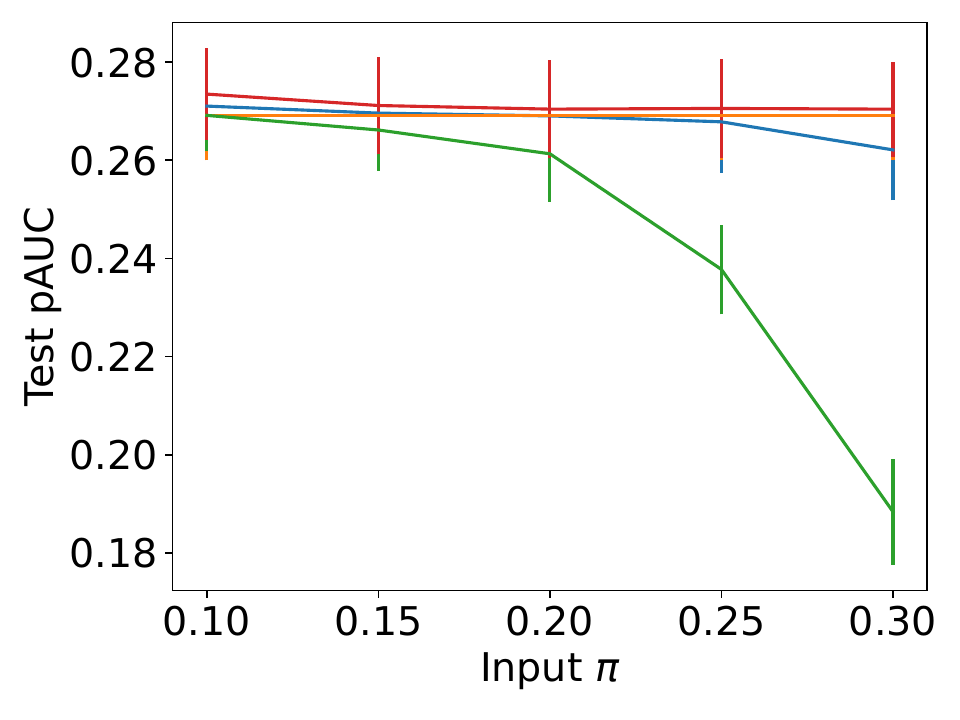}
        \subcaption{EPSILON}
    \end{minipage}
    \caption{The average test pAUCs and their standard errors with $\alpha=0$, $\beta=0.2$ when true class-prior $\pi = 0.2$. 
    `Input $\pi$' means the class-prior value used for training.}
    \label{fig:result_noisy}
\end{figure*}

\begin{table*}[t]
\caption{The results using estimated class-prior when the true class-prior is $\pi=0.2$ 
: test average pAUCs and their standard deviations with $\alpha=0$, $\beta=0.2$ over different labeled positive data sizes $N^{\rm p} = 30$, $50$, and $100$.
Values in bold are not statistically different at the $5\%$ level from the best performing method in each row according to a paired t-test.
`\# Best' row represents the number of results of each method that are the best or comparable to it.}
\label{table:all_results_estclass}
\centering
\scalebox{0.78}{
\begin{tabular}{l r r r r r r r}
\hline
Data & $\pi$ & Ours & NNPU & PUAUC & PAUC & PUpAUCh & 2step \\
\hline
MNIST        & 0.2 & \textbf{0.721 $\pm$ 0.061} & 0.676 $\pm$ 0.065 & 0.696 $\pm$ 0.049  & 0.652 $\pm$ 0.062 & \textbf{0.716 $\pm$ 0.053} & 0.697 $\pm$ 0.049  \\
FMNIST        & 0.2 & \textbf{0.933 $\pm$ 0.024} & 0.808 $\pm$ 0.055 & 0.904 $\pm$ 0.034 & 0.853 $\pm$ 0.059 & \textbf{0.932 $\pm$ 0.026} & 0.910 $\pm$ 0.033 \\
SVHN        & 0.2 & \textbf{0.301 $\pm$ 0.073} & 0.104 $\pm$ 0.005 & 0.262 $\pm$ 0.063 & 0.264 $\pm$ 0.052 & \textbf{0.292 $\pm$ 0.083} & 0.266 $\pm$ 0.063 \\
CIFAR10        & 0.2 & \textbf{0.571 $\pm$ 0.064} & 0.494 $\pm$ 0.078 & \textbf{0.559 $\pm$ 0.078} & 0.459 $\pm$ 0.070 & 0.539 $\pm$ 0.081 & \textbf{0.556 $\pm$ 0.078} \\
SUSY        & 0.2 & \textbf{0.523 $\pm$ 0.044} & 0.488 $\pm$ 0.051 & 0.499 $\pm$ 0.047 & 0.434 $\pm$ 0.056 & 0.505 $\pm$ 0.039 & 0.498 $\pm$ 0.046 \\
CODRNA        & 0.2 & 0.821 $\pm$ 0.043 & 0.639 $\pm$ 0.167 & 0.827 $\pm$ 0.048 & 0.705 $\pm$ 0.072 & \textbf{0.824 $\pm$ 0.043} & \textbf{0.829 $\pm$ 0.047} \\
IJCNN        & 0.2 & \textbf{0.646 $\pm$ 0.043} & 0.198 $\pm$ 0.043 & 0.637 $\pm$ 0.087 & 0.573 $\pm$ 0.068 & 0.616 $\pm$ 0.106 & \textbf{0.650 $\pm$ 0.084} \\
EPSILON        & 0.2 & \textbf{0.274 $\pm$ 0.049} & 0.128 $\pm$ 0.008 & \textbf{0.269 $\pm$ 0.050} & 0.248 $\pm$ 0.037 & \textbf{0.270 $\pm$ 0.046} & \textbf{0.270 $\pm$ 0.053} \\
\hline
\# Best   & & 7 & 0 & 2 & 0 & 5 & 4 \\
\hline
\end{tabular}
}
\end{table*}

\section{Conclusion}

We proposed a method for maximizing the pAUC from PU data by representing the pAUC, including its FPR-dependent thresholds, using only the positive and marginal densities. Based on this reformulation, we derived an empirical estimator whose smoothed version enables classifier training without negative data. Experiments on ten real-world datasets demonstrated its effectiveness compared with existing PU learning, AUC maximization, and pAUC maximization methods.

\bibliography{pu_pauc}

\begin{thebibliography}{49}
\providecommand{\natexlab}[1]{#1}
\providecommand{\url}[1]{\texttt{#1}}
\expandafter\ifx\csname urlstyle\endcsname\relax
  \providecommand{\doi}[1]{doi: #1}\else
  \providecommand{\doi}{doi: \begingroup \urlstyle{rm}\Url}\fi

\bibitem[Baker \& Pinsky(2001)Baker and Pinsky]{baker2001proposed}
Stuart~G Baker and Paul~F Pinsky.
\newblock A proposed design and analysis for comparing digital and analog mammography: special receiver operating characteristic methods for cancer screening.
\newblock \emph{Journal of the American Statistical Association}, 96\penalty0 (454):\penalty0 421--428, 2001.

\bibitem[Baldi et~al.(2014)Baldi, Sadowski, and Whiteson]{baldi2014searching}
Pierre Baldi, Peter Sadowski, and Daniel Whiteson.
\newblock Searching for exotic particles in high-energy physics with deep learning.
\newblock \emph{Nature communications}, 5\penalty0 (1):\penalty0 4308, 2014.

\bibitem[Bekker \& Davis(2020)Bekker and Davis]{bekker2020learning}
Jessa Bekker and Jesse Davis.
\newblock Learning from positive and unlabeled data: a survey.
\newblock \emph{Machine Learning}, 109\penalty0 (4):\penalty0 719--760, 2020.

\bibitem[Bradley(1997)]{bradley1997use}
Andrew~P Bradley.
\newblock The use of the area under the roc curve in the evaluation of machine learning algorithms.
\newblock \emph{Pattern recognition}, 30\penalty0 (7):\penalty0 1145--1159, 1997.

\bibitem[Chang \& Lin(2011)Chang and Lin]{chang2011libsvm}
Chih-Chung Chang and Chih-Jen Lin.
\newblock Libsvm: a library for support vector machines.
\newblock \emph{ACM transactions on intelligent systems and technology (TIST)}, 2\penalty0 (3):\penalty0 1--27, 2011.

\bibitem[Cortes \& Mohri(2003)Cortes and Mohri]{cortes2003auc}
Corinna Cortes and Mehryar Mohri.
\newblock Auc optimization vs. error rate minimization.
\newblock \emph{NeurIPS}, 2003.

\bibitem[Dodd \& Pepe(2003)Dodd and Pepe]{dodd2003partial}
Lori~E Dodd and Margaret~S Pepe.
\newblock Partial auc estimation and regression.
\newblock \emph{Biometrics}, 59\penalty0 (3):\penalty0 614--623, 2003.

\bibitem[Du~Plessis et~al.(2015)Du~Plessis, Niu, and Sugiyama]{du2015convex}
Marthinus Du~Plessis, Gang Niu, and Masashi Sugiyama.
\newblock Convex formulation for learning from positive and unlabeled data.
\newblock In \emph{ICML}, 2015.

\bibitem[Elkan \& Noto(2008)Elkan and Noto]{elkan2008learning}
Charles Elkan and Keith Noto.
\newblock Learning classifiers from only positive and unlabeled data.
\newblock In \emph{SIGKDD}, 2008.

\bibitem[Hammoudeh \& Lowd(2020)Hammoudeh and Lowd]{hammoudeh2020learning}
Zayd Hammoudeh and Daniel Lowd.
\newblock Learning from positive and unlabeled data with arbitrary positive shift.
\newblock In \emph{NeurIPS}, 2020.

\bibitem[Iwata et~al.(2020)Iwata, Fujino, and Ueda]{iwata2020semi}
Tomoharu Iwata, Akinori Fujino, and Naonori Ueda.
\newblock Semi-supervised learning for maximizing the partial auc.
\newblock In \emph{AAAI}, 2020.

\bibitem[Jiang et~al.(2023)Jiang, Xu, Zhao, Yang, Wen, Cao, and Huang]{jiang2023positive}
Yangbangyan Jiang, Qianqian Xu, Yunrui Zhao, Zhiyong Yang, Peisong Wen, Xiaochun Cao, and Qingming Huang.
\newblock Positive-unlabeled learning with label distribution alignment.
\newblock \emph{IEEE Transactions on Pattern Analysis and Machine Intelligence}, 2023.

\bibitem[Kingma \& Ba(2014)Kingma and Ba]{kingma2014adam}
Diederik~P Kingma and Jimmy Ba.
\newblock Adam: a method for stochastic optimization.
\newblock \emph{arXiv preprint arXiv:1412.6980}, 2014.

\bibitem[Kiryo et~al.(2017)Kiryo, Niu, Du~Plessis, and Sugiyama]{kiryo2017positive}
Ryuichi Kiryo, Gang Niu, Marthinus~C Du~Plessis, and Masashi Sugiyama.
\newblock Positive-unlabeled learning with non-negative risk estimator.
\newblock \emph{NeurIPS}, 2017.

\bibitem[Krizhevsky et~al.(2009)Krizhevsky, Hinton, et~al.]{krizhevsky2009learning}
Alex Krizhevsky, Geoffrey Hinton, et~al.
\newblock Learning multiple layers of features from tiny images.
\newblock 2009.

\bibitem[Kumagai et~al.(2024)Kumagai, Iwata, Takahashi, Nishiyama, and Fujiwara]{kumagaiauc}
Atsutoshi Kumagai, Tomoharu Iwata, Hiroshi Takahashi, Taishi Nishiyama, and Yasuhiro Fujiwara.
\newblock Auc maximization under positive distribution shift.
\newblock In \emph{NeurIPS}, 2024.

\bibitem[Kumagai et~al.(2025)Kumagai, Iwata, Takahashi, Nishiyama, Adachi, and Fujiwara]{kumagaipositive}
Atsutoshi Kumagai, Tomoharu Iwata, Hiroshi Takahashi, Taishi Nishiyama, Kazuki Adachi, and Yasuhiro Fujiwara.
\newblock Positive-unlabeled auc maximization under covariate shift.
\newblock In \emph{ICML}, 2025.

\bibitem[LeCun et~al.(1998)LeCun, Bottou, Bengio, and Haffner]{lecun1998gradient}
Yann LeCun, L{\'e}on Bottou, Yoshua Bengio, and Patrick Haffner.
\newblock Gradient-based learning applied to document recognition.
\newblock \emph{Proceedings of the IEEE}, 86\penalty0 (11):\penalty0 2278--2324, 1998.

\bibitem[Massart(1990)]{massart1990tight}
Pascal Massart.
\newblock The tight constant in the dvoretzky-kiefer-wolfowitz inequality.
\newblock \emph{The annals of Probability}, pp.\  1269--1283, 1990.

\bibitem[Mirsky et~al.(2018)Mirsky, Doitshman, Elovici, and Shabtai]{mirsky2018kitsune}
Yisroel Mirsky, Tomer Doitshman, Yuval Elovici, and Asaf Shabtai.
\newblock Kitsune: an ensemble of autoencoders for online network intrusion detection.
\newblock In \emph{NDSS}, 2018.

\bibitem[Mohri et~al.(2012)Mohri, Rostamizadeh, and Talwalkar]{mohri2012foundations}
Mehryar Mohri, Afshin Rostamizadeh, and Ameet Talwalkar.
\newblock \emph{Foundations of machine learning}, volume~2.
\newblock MIT press Cambridge, 2012.

\bibitem[Narasimhan \& Agarwal(2013{\natexlab{a}})Narasimhan and Agarwal]{narasimhan2013structural}
Harikrishna Narasimhan and Shivani Agarwal.
\newblock A structural svm based approach for optimizing partial auc.
\newblock In \emph{ICML}, 2013{\natexlab{a}}.

\bibitem[Narasimhan \& Agarwal(2013{\natexlab{b}})Narasimhan and Agarwal]{narasimhan2013svmpauctight}
Harikrishna Narasimhan and Shivani Agarwal.
\newblock Svmpauctight: a new support vector method for optimizing partial auc based on a tight convex upper bound.
\newblock In \emph{KDD}, 2013{\natexlab{b}}.

\bibitem[Netzer et~al.(2011)Netzer, Wang, Coates, Bissacco, Wu, Ng, et~al.]{netzer2011reading}
Yuval Netzer, Tao Wang, Adam Coates, Alessandro Bissacco, Baolin Wu, Andrew~Y Ng, et~al.
\newblock Reading digits in natural images with unsupervised feature learning.
\newblock In \emph{NIPS workshop on deep learning and unsupervised feature learning}, volume 2011, pp.\ ~7. Granada, Spain, 2011.

\bibitem[Nishiyama et~al.(2024)Nishiyama, Kumagai, Fujino, and Kamiya]{nishiyama2024partial}
Taishi Nishiyama, Atsutoshi Kumagai, Akinori Fujino, and Kazunori Kamiya.
\newblock Partial auc maximization for security log analysis robust to overfitting and noisy labels.
\newblock In \emph{GLOBECOM}, 2024.

\bibitem[Paszke et~al.(2017)Paszke, Gross, Chintala, Chanan, Yang, DeVito, Lin, Desmaison, Antiga, and Lerer]{paszke2017automatic}
Adam Paszke, Sam Gross, Soumith Chintala, Gregory Chanan, Edward Yang, Zachary DeVito, Zeming Lin, Alban Desmaison, Luca Antiga, and Adam Lerer.
\newblock Automatic differentiation in pytorch.
\newblock 2017.

\bibitem[Prokhorov(2001)]{prokhorov2001ijcnn}
Danil Prokhorov.
\newblock Ijcnn 2001 neural network competition.
\newblock \emph{Slide presentation in IJCNN}, 1\penalty0 (97):\penalty0 38, 2001.

\bibitem[Ramaswamy et~al.(2016)Ramaswamy, Scott, and Tewari]{ramaswamy2016mixture}
Harish Ramaswamy, Clayton Scott, and Ambuj Tewari.
\newblock Mixture proportion estimation via kernel embeddings of distributions.
\newblock In \emph{ICML}, 2016.

\bibitem[Sakai \& Shimizu(2019)Sakai and Shimizu]{sakai2019covariate}
Tomoya Sakai and Nobuyuki Shimizu.
\newblock Covariate shift adaptation on learning from positive and unlabeled data.
\newblock In \emph{AAAI}, 2019.

\bibitem[Sakai et~al.(2018)Sakai, Niu, and Sugiyama]{sakai2018semi}
Tomoya Sakai, Gang Niu, and Masashi Sugiyama.
\newblock Semi-supervised auc optimization based on positive-unlabeled learning.
\newblock \emph{Machine Learning}, 107:\penalty0 767--794, 2018.

\bibitem[Sakar et~al.(2019)Sakar, Polat, Katircioglu, and Kastro]{sakar2019real}
C~Okan Sakar, S~Olcay Polat, Mete Katircioglu, and Yomi Kastro.
\newblock Real-time prediction of online shoppers’ purchasing intention using multilayer perceptron and lstm recurrent neural networks.
\newblock \emph{Neural Computing and Applications}, 31\penalty0 (10):\penalty0 6893--6908, 2019.

\bibitem[Scott(2015)]{scott2015rate}
Clayton Scott.
\newblock A rate of convergence for mixture proportion estimation, with application to learning from noisy labels.
\newblock In \emph{AISTATS}, 2015.

\bibitem[Sugiyama et~al.(2022)Sugiyama, Bao, Ishida, Lu, and Sakai]{sugiyama2022machine}
Masashi Sugiyama, Han Bao, Takashi Ishida, Nan Lu, and Tomoya Sakai.
\newblock \emph{Machine learning from weak supervision: An empirical risk minimization approach}.
\newblock MIT Press, 2022.

\bibitem[Tsoumakas et~al.(2008)Tsoumakas, Katakis, and Vlahavas]{tsoumakas2008effective}
Grigorios Tsoumakas, Ioannis Katakis, and Ioannis Vlahavas.
\newblock Effective and efficient multilabel classification in domains with large number of labels.
\newblock In \emph{Proc. ECML/PKDD 2008 workshop on mining multidimensional data}, 2008.

\bibitem[Ueda \& Fujino(2018)Ueda and Fujino]{ueda2018partial}
Naonori Ueda and Akinori Fujino.
\newblock Partial auc maximization via nonlinear scoring functions.
\newblock \emph{arXiv preprint arXiv:1806.04838}, 2018.

\bibitem[Uzilov et~al.(2006)Uzilov, Keegan, and Mathews]{uzilov2006detection}
Andrew~V Uzilov, Joshua~M Keegan, and David~H Mathews.
\newblock Detection of non-coding rnas on the basis of predicted secondary structure formation free energy change.
\newblock \emph{BMC bioinformatics}, 7\penalty0 (1):\penalty0 173, 2006.

\bibitem[Vanschoren et~al.(2014)Vanschoren, Van~Rijn, Bischl, and Torgo]{vanschoren2014openml}
Joaquin Vanschoren, Jan~N Van~Rijn, Bernd Bischl, and Luis Torgo.
\newblock Openml: networked science in machine learning.
\newblock \emph{ACM SIGKDD Explorations Newsletter}, 15\penalty0 (2):\penalty0 49--60, 2014.

\bibitem[Xiao et~al.(2017)Xiao, Rasul, and Vollgraf]{xiao2017fashion}
Han Xiao, Kashif Rasul, and Roland Vollgraf.
\newblock Fashion-mnist: a novel image dataset for benchmarking machine learning algorithms.
\newblock \emph{arXiv preprint arXiv:1708.07747}, 2017.

\bibitem[Xie \& Li(2018)Xie and Li]{xie2018semi}
Zheng Xie and Ming Li.
\newblock Semi-supervised auc optimization without guessing labels of unlabeled data.
\newblock In \emph{AAAI}, 2018.

\bibitem[Xie et~al.(2024)Xie, Liu, He, Li, and Zhou]{xie2024weakly}
Zheng Xie, Yu~Liu, Hao-Yuan He, Ming Li, and Zhi-Hua Zhou.
\newblock Weakly supervised auc optimization: a unified partial auc approach.
\newblock \emph{IEEE Transactions on Pattern Analysis and Machine Intelligence}, 2024.

\bibitem[Yang \& Ying(2022)Yang and Ying]{yang2022auc}
Tianbao Yang and Yiming Ying.
\newblock Auc maximization in the era of big data and ai: a survey.
\newblock \emph{ACM computing surveys}, 55\penalty0 (8):\penalty0 1--37, 2022.

\bibitem[Yao et~al.(2022)Yao, Lin, and Yang]{yao2022large}
Yao Yao, Qihang Lin, and Tianbao Yang.
\newblock Large-scale optimization of partial auc in a range of false positive rates.
\newblock In \emph{NeurIPS}, 2022.

\bibitem[Yao et~al.(2021)Yao, Liu, Han, Gong, Niu, Sugiyama, and Tao]{yao2021rethinking}
Yu~Yao, Tongliang Liu, Bo~Han, Mingming Gong, Gang Niu, Masashi Sugiyama, and Dacheng Tao.
\newblock Rethinking class-prior estimation for positive-unlabeled learning.
\newblock In \emph{ICLR}, 2021.

\bibitem[Ying et~al.(2016)Ying, Wen, and Lyu]{ying2016stochastic}
Yiming Ying, Longyin Wen, and Siwei Lyu.
\newblock Stochastic online auc maximization.
\newblock \emph{NeurIPS}, 2016.

\bibitem[Yuan et~al.(2021{\natexlab{a}})Yuan, Guo, Chawla, and Yang]{yuan2021compositional}
Zhuoning Yuan, Zhishuai Guo, Nitesh Chawla, and Tianbao Yang.
\newblock Compositional training for end-to-end deep auc maximization.
\newblock In \emph{ICLR}, 2021{\natexlab{a}}.

\bibitem[Yuan et~al.(2021{\natexlab{b}})Yuan, Yan, Sonka, and Yang]{yuan2021large}
Zhuoning Yuan, Yan Yan, Milan Sonka, and Tianbao Yang.
\newblock Large-scale robust deep auc maximization: a new surrogate loss and empirical studies on medical image classification.
\newblock In \emph{Proceedings of the IEEE/CVF International Conference on Computer Vision}, pp.\  3040--3049, 2021{\natexlab{b}}.

\bibitem[Zhang et~al.(2023)Zhang, Zhang, Yang, Souvenir, and Gao]{zhang2023federated}
Xinwen Zhang, Yihan Zhang, Tianbao Yang, Richard Souvenir, and Hongchang Gao.
\newblock Federated compositional deep auc maximization.
\newblock \emph{NeurIPS}, 2023.

\bibitem[Zhou et~al.(2025)Zhou, Wang, Thai, and Yang]{zhou2025stochastic}
Linli Zhou, Bokun Wang, My~T. Thai, and Tianbao Yang.
\newblock Stochastic primal-dual double block-coordinate for two-way partial auc maximization.
\newblock \emph{Transactions on Machine Learning Research}, 2025.
\newblock ISSN 2835-8856.

\bibitem[Zhu et~al.(2022)Zhu, Li, Wang, Wu, and Yang]{zhu2022auc}
Dixian Zhu, Gang Li, Bokun Wang, Xiaodong Wu, and Tianbao Yang.
\newblock When auc meets dro: optimizing partial auc for deep learning with non-convex convergence guarantee.
\newblock In \emph{ICML}, 2022.

\end{thebibliography}
\bibliographystyle{iclr2027_conference}

\appendix

\section{Theoretical Analysis}
\label{sec:theoretical_analysis}

In this section, we analyze the finite-sample properties of the proposed
pAUC estimator in Eq. \eqref{eq:emp_pauc_pu}.
We first analyze the estimation error of the FPR estimator in
Eq.~\eqref{eq:fpr_pu_emp} and the FPRs at the estimated thresholds
${\hat t}_{\alpha}$ and ${\hat t}_{\beta}$.
Then, we derive a finite-sample error bound for the empirical
pAUC estimator in Eq.~\eqref{eq:emp_pauc_pu}.
Finally, we analyze the effect of class-prior estimation error.
Throughout this section, we assume
$0\le\alpha<\beta<1$, which covers all settings of interest.

For the analysis, we explicitly write the dependence on score function
$s$ when necessary.
Let $\mathcal{S}$ be a class of score functions.
For $s\in\mathcal{S}$, define
\begin{align}
\overline F_s(t)
&:=
\mathbb{E}_{{\bf x}\sim p}
\left[
I(s({\bf x})>t)
\right],
\label{eq:theory_survival_u}
\\
\overline F_s^{{\rm p}}(t)
&:=
\mathbb{E}_{{\bf x}^{{\rm p}}\sim p^{{\rm p}}}
\left[
I(s({\bf x}^{{\rm p}})>t)
\right],
\label{eq:theory_survival_p}
\end{align}
and their empirical versions as
\begin{align}
\widehat{\overline F}_s(t)
&:=
\frac{1}{N}
\sum_{n=1}^{N}
I(s({\bf x}_n)>t),
\label{eq:theory_survival_u_emp}
\\
\widehat{\overline F}_s^{{\rm p}}(t)
&:=
\frac{1}{N^{{\rm p}}}
\sum_{n=1}^{N^{{\rm p}}}
I(s({\bf x}_n^{{\rm p}})>t).
\label{eq:theory_survival_p_emp}
\end{align}
Then, Eqs.~\eqref{eq:fpr_pu} and \eqref{eq:fpr_pu_emp} can be written as
\begin{align}
{\rm FPR}_s(t)
&=
\frac{
\overline F_s(t)
-
\pi\overline F_s^{{\rm p}}(t)
}
{1-\pi},
\label{eq:theory_fpr_pop}
\\
{\widehat {\rm FPR}}_s(t)
&=
\frac{
\widehat{\overline F}_s(t)
-
\pi\widehat{\overline F}_s^{{\rm p}}(t)
}
{1-\pi}.
\label{eq:theory_fpr_emp}
\end{align}

We use the following assumption.

\begin{assumption}
\label{assump:no_point_mass}
For every $s\in\mathcal{S}$ considered in the analysis, the distribution
of $s({\bf x}^{{\rm n}})$ under $p^{{\rm n}}({\bf x})$ has no point
masses.
\end{assumption}

Assumption~\ref{assump:no_point_mass} serves as a convenient sufficient
condition for handling the data-dependent estimated thresholds in the
finite-sample analysis, as it ensures that ${\rm FPR}_s(t)$ is continuous
in $t$.
We use this continuity below when relating an estimated threshold to the
FPR at that threshold.
When $\alpha=0$, it is possible that no finite threshold achieves
${\rm FPR}_s(t)=0$.
In this case, we set $t_{\alpha,s}=+\infty$, define
${\rm FPR}_s(+\infty)=0$, and interpret
$[t_{\beta,s},t_{\alpha,s}]$ as $[t_{\beta,s},\infty)$.

\subsection{Uniform Convergence of the Score Survival Probabilities}
\label{subsec:theory_uniform}

Define the uniform estimation errors of the marginal and positive
score survival probabilities as
\begin{align}
\Delta_N
&:=
\sup_{s\in\mathcal{S},\,t\in\mathbb{R}}
\left|
\widehat{\overline F}_s(t)
-
\overline F_s(t)
\right|,
\label{eq:delta_u}
\\
\Delta_{N^{{\rm p}}}
&:=
\sup_{s\in\mathcal{S},\,t\in\mathbb{R}}
\left|
\widehat{\overline F}_s^{{\rm p}}(t)
-
\overline F_s^{{\rm p}}(t)
\right|.
\label{eq:delta_p}
\end{align}
Thus, $\Delta_N$ and $\Delta_{N^{{\rm p}}}$ are the largest
estimation errors of the marginal and positive score survival
probabilities, respectively, over all
$s\in\mathcal{S}$ and $t\in\mathbb{R}$.

We first consider a finite score-function class
$\mathcal{S}$ with $|\mathcal{S}|=M$.
For any fixed $s\in\mathcal{S}$, the
Dvoretzky--Kiefer--Wolfowitz (DKW) inequality
\citep{massart1990tight}
gives
\begin{align}
\mathbb{P}
\left(
\sup_{t\in\mathbb{R}}
\left|
\widehat{\overline F}_s(t)
-
\overline F_s(t)
\right|
>
\epsilon
\right)
\le
2\exp(-2N\epsilon^2).
\label{eq:dkw_u_fixed}
\end{align}
Although the DKW inequality is usually stated for cumulative
distribution functions, it can also be directly applied to the
survival probabilities since
$\overline F_s(t)=1-F_s(t)$ and
$\widehat{\overline F}_s(t)=1-\widehat F_s(t)$,
where $F_s$ and $\widehat F_s$ are the population and empirical
cumulative distribution functions of $s({\bf x})$, respectively.

Applying the union bound over all $M$ score functions gives
\begin{align}
\mathbb{P}
\left(
\Delta_N>\epsilon
\right)
\le
2M\exp(-2N\epsilon^2).
\label{eq:dkw_u_union}
\end{align}
Setting the right-hand side to $\delta/2$ gives
\begin{align}
\epsilon_N
=
\sqrt{
\frac{\log(4M/\delta)}{2N}
},
\label{eq:epsilon_u_finite}
\end{align}
and therefore
\begin{align}
\mathbb{P}
\left(
\Delta_N\le\epsilon_N
\right)
\ge
1-\frac{\delta}{2}.
\label{eq:dkw_u_probability}
\end{align}
Applying the same argument to the $N^{{\rm p}}$ positive instances
gives
\begin{align}
\epsilon_{N^{{\rm p}}}
=
\sqrt{
\frac{\log(4M/\delta)}{2N^{{\rm p}}}
},
\end{align}
such that
$\mathbb{P}(\Delta_{N^{{\rm p}}}\le\epsilon_{N^{{\rm p}}})
\ge 1-\delta/2$.
Thus, by the union bound, with probability at least $1-\delta$,
\begin{align}
\Delta_N
&\le
\epsilon_N,
&
\Delta_{N^{{\rm p}}}
&\le
\epsilon_{N^{{\rm p}}}
\label{eq:uniform_conditions}
\end{align}
hold simultaneously.
In particular, for a finite score-function class, the empirical
marginal and positive score survival probabilities converge uniformly
to their population values, with errors of orders
\begin{align}
O\left(
\sqrt{\frac{\log(M/\delta)}{N}}
\right)
\quad\text{and}\quad
O\left(
\sqrt{\frac{\log(M/\delta)}{N^{{\rm p}}}}
\right),
\end{align}
respectively.

The following analysis only requires the uniform bounds in
Eq.~\eqref{eq:uniform_conditions}.
Therefore, the results are not restricted to finite score-function
classes.
For a score function $s$ and a threshold $t$, define
\begin{align}
C_{s,t}
:=
\left\{
{\bf x}\in\mathcal{X}
:
s({\bf x})>t
\right\}.
\end{align}
Consider the collection of these sets over all score functions and
thresholds:
\begin{align}
\mathcal{C}
:=
\left\{
C_{s,t}
:
s\in\mathcal{S},\
t\in\mathbb{R}
\right\}.
\label{eq:threshold_class}
\end{align}
Since
\begin{align}
\overline F_s(t)
&=
\mathbb{P}_{{\bf x}\sim p}
\left(
{\bf x}\in C_{s,t}
\right),
&
\widehat{\overline F}_s(t)
&=
\frac{1}{N}
\sum_{n=1}^{N}
I({\bf x}_n\in C_{s,t}),
\end{align}
we have
\begin{align}
\Delta_N
=
\sup_{C\in\mathcal{C}}
\left|
\frac{1}{N}
\sum_{n=1}^{N}
I({\bf x}_n\in C)
-
\mathbb{P}_{{\bf x}\sim p}
\left(
{\bf x}\in C
\right)
\right|.
\end{align}
Therefore, if $\mathcal{C}$ has finite VC dimension $d$, a standard
VC uniform-convergence bound
\citep{mohri2012foundations}
gives
\begin{align}
\Delta_N
=
O\left(
\sqrt{
\frac{
d\log(2N)+\log(2/\delta)
}{N}
}
\right)
\end{align}
with probability at least $1-\delta/2$.
Applying the same argument to the positive distribution using the
same set class $\mathcal{C}$ and the $N^{{\rm p}}$ positive instances
gives
\begin{align}
\Delta_{N^{{\rm p}}}
=
O\left(
\sqrt{
\frac{
d\log(2N^{{\rm p}})+\log(2/\delta)
}{N^{{\rm p}}}
}
\right)
\end{align}
with probability at least $1-\delta/2$.
Thus, by the union bound, both bounds hold simultaneously with
probability at least $1-\delta$.
Therefore, the following results also apply to infinite score-function
classes when the threshold-set class $\mathcal{C}$ has finite VC
dimension.

\subsection{Accuracy of the PU-based FPR Estimator and Estimated Thresholds}
\label{subsec:theory_fpr}

We first analyze the error of the FPR estimator in
Eq.~\eqref{eq:fpr_pu_emp}.

\begin{theorem}
\label{thm:fpr_uniform}
Suppose that
\begin{align}
\Delta_N
&\le
\epsilon_N,
&
\Delta_{N^{{\rm p}}}
&\le
\epsilon_{N^{{\rm p}}}.
\end{align}
Then,
\begin{align}
\sup_{s\in\mathcal{S},\,t\in\mathbb{R}}
\left|
{\widehat {\rm FPR}}_s(t)
-
{\rm FPR}_s(t)
\right|
\le
\epsilon_{\rm FPR},
\label{eq:fpr_uniform_bound}
\end{align}
where
\begin{align}
\epsilon_{\rm FPR}
:=
\frac{
\epsilon_N
+
\pi\epsilon_{N^{{\rm p}}}
}
{1-\pi}.
\label{eq:epsilon_fpr}
\end{align}
The same bound holds when ${\widehat {\rm FPR}}_s(t)$ is clipped to
$[0,1]$ as in the proposed method.
\end{theorem}

\begin{proof}
Using Eqs.~\eqref{eq:theory_fpr_pop} and
\eqref{eq:theory_fpr_emp}, we have
\begin{align}
&
\left|
{\widehat {\rm FPR}}_s(t)
-
{\rm FPR}_s(t)
\right|
\nonumber\\
&=
\frac{1}{1-\pi}
\left|
\left(
\widehat{\overline F}_s(t)
-
\overline F_s(t)
\right)
-
\pi
\left(
\widehat{\overline F}_s^{{\rm p}}(t)
-
\overline F_s^{{\rm p}}(t)
\right)
\right|
\nonumber\\
&\le
\frac{1}{1-\pi}
\left[
\left|
\widehat{\overline F}_s(t)
-
\overline F_s(t)
\right|
+
\pi
\left|
\widehat{\overline F}_s^{{\rm p}}(t)
-
\overline F_s^{{\rm p}}(t)
\right|
\right]
\nonumber\\
&\le
\frac{
\Delta_N
+
\pi\Delta_{N^{{\rm p}}}
}{
1-\pi
}
\nonumber\\
&\le
\frac{
\epsilon_N
+
\pi\epsilon_{N^{{\rm p}}}
}{
1-\pi
}
=
\epsilon_{\rm FPR}.
\end{align}
Since this holds for every $s\in\mathcal{S}$ and
$t\in\mathbb{R}$, by taking the supremum, we obtain
Eq.~\eqref{eq:fpr_uniform_bound}.
Since ${\rm FPR}_s(t)\in[0,1]$, clipping
${\widehat {\rm FPR}}_s(t)$ to $[0,1]$ cannot increase the estimation
error.
Thus, the same bound also holds for the clipped estimator.
\end{proof}

For $\gamma\in\{\alpha,\beta\}$, let
${\hat t}_{\gamma,s}$ denote the estimated threshold obtained by
Eq.~\eqref{eq:th_a_pu} or Eq.~\eqref{eq:th_b_pu}, respectively,
for score function $s$.
Since $\alpha<\beta$,
\begin{align}
\left\{
t:
{\widehat {\rm FPR}}_s(t)\le\alpha
\right\}
\subseteq
\left\{
t:
{\widehat {\rm FPR}}_s(t)\le\beta
\right\}.
\end{align}
Therefore,
${\hat t}_{\beta,s}\le{\hat t}_{\alpha,s}$ even though
${\widehat {\rm FPR}}_s(t)$ is not necessarily monotonic in $t$.

\begin{theorem}
\label{thm:threshold_fpr}
Suppose that Assumption~\ref{assump:no_point_mass} holds,
$0\le\alpha<\beta<1$, and
\begin{align}
\Delta_N
&\le
\epsilon_N,
&
\Delta_{N^{{\rm p}}}
&\le
\epsilon_{N^{{\rm p}}}.
\end{align}
Then, for every $s\in\mathcal{S}$ and
$\gamma\in\{\alpha,\beta\}$,
\begin{align}
\left|
{\rm FPR}_s({\hat t}_{\gamma,s})
-
\gamma
\right|
\le
\epsilon_{\rm FPR}.
\label{eq:estimated_threshold_fpr_bound}
\end{align}
\end{theorem}

\begin{proof}
We fix $s$ and omit its subscript for simplicity.
Since $0\le\gamma<1$, clipping
$\widehat{\rm FPR}_s(t)$ to $[0,1]$ does not change the set of
thresholds satisfying $\widehat{\rm FPR}_s(t)\le\gamma$.
From the definition of ${\hat t}_{\gamma}$ in
Eqs.~\eqref{eq:th_a_pu} and \eqref{eq:th_b_pu},
\begin{align}
{\widehat {\rm FPR}}({\hat t}_{\gamma})
\le
\gamma.
\end{align}
Therefore, Theorem~\ref{thm:fpr_uniform} gives
\begin{align}
{\rm FPR}({\hat t}_{\gamma})
\le
\gamma+\epsilon_{\rm FPR}.
\label{eq:threshold_upper_bound}
\end{align}

For the lower bound,
${\hat t}_{\gamma}$ is the smallest score among the PU instances at
which the empirical FPR is no greater than $\gamma$.
Suppose first that ${\hat t}_{\gamma}$ is not the smallest score among the PU instances, and let $t^-_{\gamma}$ be the largest
distinct score smaller than ${\hat t}_{\gamma}$.
By the definition of ${\hat t}_{\gamma}$,
\begin{align}
{\widehat {\rm FPR}}(t^-_{\gamma})>\gamma.
\end{align}
Since there is no PU instance whose score lies in
$(t^-_{\gamma},{\hat t}_{\gamma})$, the empirical FPR is constant
on this interval. Therefore,
\begin{align}
{\widehat {\rm FPR}}(t)>\gamma,
\qquad
t\in(t^-_{\gamma},{\hat t}_{\gamma}).
\end{align}
If ${\hat t}_{\gamma}$ is the smallest score among the PU instances,
then all PU instances have scores greater than $t$ for any
$t<{\hat t}_{\gamma}$.
Hence,
\begin{align}
{\widehat {\rm FPR}}(t)
=
\frac{1-\pi}{1-\pi}
=
1
>
\gamma.
\end{align}
Thus, in either case, there is a left neighborhood of
${\hat t}_{\gamma}$ in which
${\widehat {\rm FPR}}(t)>\gamma$.
This argument does not require
${\widehat {\rm FPR}}(t)$ to be monotonic.
For $t$ approaching ${\hat t}_{\gamma}$ from below,
Theorem~\ref{thm:fpr_uniform} gives
\begin{align}
{\rm FPR}(t)
>
\gamma-\epsilon_{\rm FPR}.
\end{align}
Assumption~\ref{assump:no_point_mass} ensures that
${\rm FPR}(t)$ is continuous.
Therefore, taking $t\uparrow{\hat t}_{\gamma}$ gives
\begin{align}
{\rm FPR}({\hat t}_{\gamma})
\ge
\gamma-\epsilon_{\rm FPR}.
\label{eq:threshold_lower_bound}
\end{align}
Combining Eqs.~\eqref{eq:threshold_upper_bound}
and \eqref{eq:threshold_lower_bound} gives
\begin{align}
\left|
{\rm FPR}({\hat t}_{\gamma})
-
\gamma
\right|
\le
\epsilon_{\rm FPR}.
\end{align}
Since $s\in\mathcal{S}$ and
$\gamma\in\{\alpha,\beta\}$ were arbitrary, the result holds for
every $s\in\mathcal{S}$ and $\gamma\in\{\alpha,\beta\}$.
\end{proof}

Theorem~\ref{thm:threshold_fpr} directly bounds the quantity needed for
pAUC estimation.
The estimated threshold values themselves do not need to be close to
their population values as long as the FPRs at the estimated thresholds
are close to the target values $\alpha$ and $\beta$.

\subsection{Finite-sample Error of the Empirical pAUC Estimator}
\label{subsec:theory_hard_pauc}

We next analyze the empirical pAUC estimator in
Eq.~\eqref{eq:emp_pauc_pu}.
For $b\in\mathbb{R}$ and
$a\in\mathbb{R}\cup\{+\infty\}$ with $b\le a$, define
\begin{align}
A_s(a,b)
:=
&
\mathbb{E}_{{\bf x}^{{\rm p}}\sim p^{{\rm p}}}
\mathbb{E}_{{\bf x}\sim p}
\left[
I(s({\bf x}^{{\rm p}})>s({\bf x}))
I(s({\bf x})\in[b,a])
\right],
\label{eq:A_population}
\\
B_s(a,b)
:=
&
\mathbb{E}_{{\bf x}^{{\rm p}}\sim p^{{\rm p}}}
\mathbb{E}_{{\bar{\bf x}}^{{\rm p}}\sim p^{{\rm p}}}
\left[
I(s({\bf x}^{{\rm p}})>s({\bar{\bf x}}^{{\rm p}}))
I(s({\bar{\bf x}}^{{\rm p}})\in[b,a])
\right].
\label{eq:B_population}
\end{align}
Their empirical versions are
\begin{align}
\widehat A_s(a,b)
:=
&
\frac{1}{NN^{{\rm p}}}
\sum_{n=1}^{N^{{\rm p}}}
\sum_{m=1}^{N}
I(s({\bf x}_n^{{\rm p}})>s({\bf x}_m))
I(s({\bf x}_m)\in[b,a]),
\label{eq:A_empirical}
\\
\widehat B_s(a,b)
:=
&
\frac{1}{N^{{\rm p}}(N^{{\rm p}}-1)}
\sum_{\substack{n,m=1\\n\ne m}}^{N^{{\rm p}}}
I(s({\bf x}_n^{{\rm p}})>s({\bf x}_m^{{\rm p}}))
I(s({\bf x}_m^{{\rm p}})\in[b,a]).
\label{eq:B_empirical}
\end{align}
Define
\begin{align}
Q_s(a,b)
:=
\frac{
A_s(a,b)
-
\pi B_s(a,b)
}{
(\beta-\alpha)(1-\pi)
}.
\label{eq:Q_population}
\end{align}
Using Eq.~\eqref{eq:neg_den}, we have
\begin{align}
Q_s(a,b)
=
\frac{1}{\beta-\alpha}
\mathbb{E}_{{\bf x}^{{\rm p}}\sim p^{{\rm p}}}
\mathbb{E}_{{\bf x}^{{\rm n}}\sim p^{{\rm n}}}
\left[
I(s({\bf x}^{{\rm p}})>s({\bf x}^{{\rm n}}))
I(s({\bf x}^{{\rm n}})\in[b,a])
\right].
\label{eq:Q_negative}
\end{align}
Therefore,
\begin{align}
Q_s(t_{\alpha,s},t_{\beta,s})
=
{\rm pAUC}_s(\alpha,\beta).
\label{eq:Q_true_pauc}
\end{align}

We use the following lemma.

\begin{lemma}
\label{lemma:weighted_interval}
Let $Z_1,\ldots,Z_n$ be i.i.d.\ copies of a random variable $Z$ and suppose that
\begin{align}
\sup_{t\in\mathbb{R}}
\left|
\frac{1}{n}
\sum_{i=1}^{n}I(Z_i>t)
-
\mathbb{P}(Z>t)
\right|
\le
\epsilon.
\label{eq:lemma_cdf_condition}
\end{align}
Then, for any non-increasing function
$g:\mathbb{R}\to[0,1]$,
\begin{align}
\sup_{b\le a}
\left|
\frac{1}{n}
\sum_{i=1}^{n}
g(Z_i)I(Z_i\in[b,a])
-
\mathbb{E}
\left[
g(Z)I(Z\in[b,a])
\right]
\right|
\le
2\epsilon,
\label{eq:weighted_interval_bound}
\end{align}
where the supremum is taken over $a,b\in\mathbb{R}$.
\end{lemma}

\begin{proof}
Fix a realization of $Z_1,\ldots,Z_n$ satisfying
Eq.~\eqref{eq:lemma_cdf_condition}.
Then,
\begin{align}
\left|
\frac{1}{n}
\sum_{i=1}^{n}
I(Z_i>t)
-
\mathbb{P}(Z>t)
\right|
\le
\epsilon
\qquad
\text{for all } t\in\mathbb{R}.
\end{align}
The probability of any interval can be expressed using differences
and one-sided limits of the tail probabilities in
Eq.~\eqref{eq:lemma_cdf_condition}.
Therefore,
\begin{align}
\left|
\frac{1}{n}
\sum_{i=1}^{n}
I(Z_i\in J)
-
\mathbb{P}(Z\in J)
\right|
\le
2\epsilon
\label{eq:interval_probability_bound}
\end{align}
for any interval $J$.
For example, for a closed left endpoint, we use
\begin{align}
\mathbb{P}(Z\in[b,\infty))
&=
\lim_{t\uparrow b}\mathbb{P}(Z>t),
&
\frac{1}{n}
\sum_{i=1}^{n}
I(Z_i\in[b,\infty))
&=
\lim_{t\uparrow b}
\frac{1}{n}
\sum_{i=1}^{n}
I(Z_i>t).
\end{align}

Next, since $0\le g(z)\le 1$, we use
\begin{align}
g(z)
=
\int_0^1 I(g(z)>u)\,du.
\end{align}
Multiplying both sides by $I(z\in[b,a])$ gives
\begin{align}
g(z)I(z\in[b,a])
&=
\int_0^1
I(z\in[b,a])I(g(z)>u)
\,du \nonumber\\
&=
\int_0^1 I(z\in J_u)\,du,
\end{align}
where
\begin{align}
J_u
=
[b,a]\cap\{z:g(z)>u\}.
\end{align}
Since $g$ is non-increasing, $\{z:g(z)>u\}$ is a half-line,
i.e., a set of the form $(-\infty,c)$ or $(-\infty,c]$
(possibly empty or all of $\mathbb{R}$).
Therefore, $J_u$ is an interval, possibly empty, for every
$u\in[0,1]$.
By Eq.~\eqref{eq:interval_probability_bound},
\begin{align}
\left|
\frac{1}{n}
\sum_{i=1}^{n}
I(Z_i\in J_u)
-
\mathbb{P}(Z\in J_u)
\right|
\le
2\epsilon.
\end{align}

Applying the above identity to each $Z_i$ and to $Z$, we have
\begin{align}
&
\frac{1}{n}
\sum_{i=1}^{n}
g(Z_i)I(Z_i\in[b,a])
-
\mathbb{E}
\left[
g(Z)I(Z\in[b,a])
\right]
\nonumber\\
&=
\int_0^1
\left\{
\frac{1}{n}
\sum_{i=1}^{n}
I(Z_i\in J_u)
-
\mathbb{P}(Z\in J_u)
\right\}
\,du.
\end{align}
Therefore,
\begin{align}
&
\left|
\frac{1}{n}
\sum_{i=1}^{n}
g(Z_i)I(Z_i\in[b,a])
-
\mathbb{E}
\left[
g(Z)I(Z\in[b,a])
\right]
\right|
\nonumber\\
&\le
\int_0^1
\left|
\frac{1}{n}
\sum_{i=1}^{n}
I(Z_i\in J_u)
-
\mathbb{P}(Z\in J_u)
\right|
\,du
\nonumber\\
&\le
\int_0^1 2\epsilon\,du
=
2\epsilon.
\end{align}
Since this holds for every $b\le a$, taking the supremum proves
Eq.~\eqref{eq:weighted_interval_bound}.
\end{proof}

\begin{lemma}
\label{lemma:hard_pairwise_terms}
Suppose that
$\Delta_N\le\epsilon_N$ and
$\Delta_{N^{{\rm p}}}\le\epsilon_{N^{{\rm p}}}$.
Then, for every $s\in\mathcal{S}$ and
$a,b\in\mathbb{R}$ with $b\le a$,
\begin{align}
\left|
\widehat A_s(a,b)-A_s(a,b)
\right|
&\le
2\epsilon_N+\epsilon_{N^{{\rm p}}},
\label{eq:A_error_bound}
\\
\left|
\widehat B_s(a,b)-B_s(a,b)
\right|
&\le
3\frac{N^{{\rm p}}}{N^{{\rm p}}-1}
\epsilon_{N^{{\rm p}}}
+
\frac{1}{N^{{\rm p}}-1}.
\label{eq:B_error_bound}
\end{align}
\end{lemma}

\begin{proof}
Using Eqs.~\eqref{eq:theory_survival_p} and
\eqref{eq:A_population}, we have
\begin{align}
A_s(a,b)
=
\mathbb{E}_{{\bf x}\sim p}
\left[
\overline F_s^{{\rm p}}(s({\bf x}))
I(s({\bf x})\in[b,a])
\right].
\label{eq:A_survival_form}
\end{align}
Similarly, Eq.~\eqref{eq:A_empirical} can be written as
\begin{align}
\widehat A_s(a,b)
=
\frac{1}{N}
\sum_{m=1}^{N}
\widehat{\overline F}_s^{{\rm p}}(s({\bf x}_m))
I(s({\bf x}_m)\in[b,a]).
\label{eq:A_empirical_survival_form}
\end{align}
Therefore,
\begin{align}
&
\left|
\widehat A_s(a,b)-A_s(a,b)
\right|
\nonumber\\
&\le
\left|
\frac{1}{N}
\sum_{m=1}^{N}
\left[
\widehat{\overline F}_s^{{\rm p}}(s({\bf x}_m))
-
\overline F_s^{{\rm p}}(s({\bf x}_m))
\right]
I(s({\bf x}_m)\in[b,a])
\right|
\nonumber\\
&\quad+
\left|
\frac{1}{N}
\sum_{m=1}^{N}
\overline F_s^{{\rm p}}(s({\bf x}_m))
I(s({\bf x}_m)\in[b,a])
-
A_s(a,b)
\right|.
\label{eq:A_error_decomposition}
\end{align}
By the definition of $\Delta_{N^{{\rm p}}}$, the first term is at
most $\epsilon_{N^{{\rm p}}}$.
Since $\overline F_s^{{\rm p}}(z)$ is non-increasing in $z$,
Lemma~\ref{lemma:weighted_interval}, applied to the marginal scores
$Z_m=s({\bf x}_m)$ together with
$\Delta_N\le\epsilon_N$, bounds the second term by $2\epsilon_N$.
This proves Eq.~\eqref{eq:A_error_bound}.

For $\widehat B_s(a,b)$, we first introduce the corresponding
V-statistic so that the pairwise sum can be written using the empirical
positive score survival probability:
\begin{align}
\widehat B_s^{\rm V}(a,b)
:=
\frac{1}{(N^{{\rm p}})^2}
\sum_{n=1}^{N^{{\rm p}}}
\sum_{m=1}^{N^{{\rm p}}}
I(s({\bf x}_n^{{\rm p}})>s({\bf x}_m^{{\rm p}}))
I(s({\bf x}_m^{{\rm p}})\in[b,a]).
\end{align}
Using Eq.~\eqref{eq:theory_survival_p_emp}, this can be written as
\begin{align}
\widehat B_s^{\rm V}(a,b)
=
\frac{1}{N^{{\rm p}}}
\sum_{m=1}^{N^{{\rm p}}}
\widehat{\overline F}_s^{{\rm p}}
(s({\bf x}_m^{{\rm p}}))
I(s({\bf x}_m^{{\rm p}})\in[b,a]).
\label{eq:B_V_survival_form}
\end{align}
Also,
\begin{align}
B_s(a,b)
=
\mathbb{E}_{{\bf x}^{{\rm p}}\sim p^{{\rm p}}}
\left[
\overline F_s^{{\rm p}}(s({\bf x}^{{\rm p}}))
I(s({\bf x}^{{\rm p}})\in[b,a])
\right].
\label{eq:B_survival_form}
\end{align}
As in the case of $\widehat A_s(a,b)$,
replacing $\widehat{\overline F}_s^{{\rm p}}$ in
Eq.~\eqref{eq:B_V_survival_form} with
$\overline F_s^{{\rm p}}$ gives an error of at most
$\epsilon_{N^{{\rm p}}}$.
Applying Lemma~\ref{lemma:weighted_interval} to the positive scores
gives an additional error of at most
$2\epsilon_{N^{{\rm p}}}$.
Therefore,
\begin{align}
\left|
\widehat B_s^{\rm V}(a,b)
-
B_s(a,b)
\right|
\le
3\epsilon_{N^{{\rm p}}}.
\label{eq:B_V_error_bound}
\end{align}

The diagonal terms in $\widehat B_s^{\rm V}(a,b)$ are zero since
\begin{align}
I(s({\bf x}_n^{{\rm p}})>s({\bf x}_n^{{\rm p}}))=0.
\end{align}
Therefore,
\begin{align}
\widehat B_s(a,b)
=
\frac{N^{{\rm p}}}{N^{{\rm p}}-1}
\widehat B_s^{\rm V}(a,b).
\label{eq:hard_B_UV}
\end{align}
Using Eq.~\eqref{eq:B_V_error_bound}, we have
\begin{align}
\left|
\widehat B_s(a,b)-B_s(a,b)
\right|
&\le
\frac{N^{{\rm p}}}{N^{{\rm p}}-1}
\left|
\widehat B_s^{\rm V}(a,b)-B_s(a,b)
\right|
+
\frac{B_s(a,b)}{N^{{\rm p}}-1}
\nonumber\\
&\le
3\frac{N^{{\rm p}}}{N^{{\rm p}}-1}
\epsilon_{N^{{\rm p}}}
+
\frac{1}{N^{{\rm p}}-1},
\end{align}
where we used $0\le B_s(a,b)\le1$.
This proves Eq.~\eqref{eq:B_error_bound}.
\end{proof}

\begin{theorem}
\label{thm:pauc_finite_sample}
Suppose that Assumption~\ref{assump:no_point_mass} holds,
$0\le\alpha<\beta<1$, $N^{{\rm p}}\ge2$, and
\begin{align}
\Delta_N
&\le
\epsilon_N,
&
\Delta_{N^{{\rm p}}}
&\le
\epsilon_{N^{{\rm p}}}.
\end{align}
Then, simultaneously for every $s\in\mathcal{S}$,
\begin{align}
\left|
{\widehat {\rm pAUC}}_s(\alpha,\beta)
-
{\rm pAUC}_s(\alpha,\beta)
\right|
\le
B_{N,N^{{\rm p}}},
\label{eq:pauc_finite_sample_bound}
\end{align}
where
\begin{align}
B_{N,N^{{\rm p}}}
:=
\frac{1}{
(\beta-\alpha)(1-\pi)
}
\left[
4\epsilon_N
+
\left(
1+2\pi
+
3\pi\frac{N^{{\rm p}}}{N^{{\rm p}}-1}
\right)
\epsilon_{N^{{\rm p}}}
+
\frac{\pi}{N^{{\rm p}}-1}
\right].
\label{eq:hard_pauc_bound}
\end{align}
\end{theorem}

\begin{proof}
We decompose the estimation error as
\begin{align}
&
\left|
{\widehat {\rm pAUC}}_s(\alpha,\beta)
-
{\rm pAUC}_s(\alpha,\beta)
\right|
\nonumber\\
&\le
\left|
{\widehat {\rm pAUC}}_s(\alpha,\beta)
-
Q_s({\hat t}_{\alpha,s},{\hat t}_{\beta,s})
\right|
\nonumber\\
&\quad+
\left|
Q_s({\hat t}_{\alpha,s},{\hat t}_{\beta,s})
-
Q_s(t_{\alpha,s},t_{\beta,s})
\right|.
\label{eq:pauc_error_decomposition}
\end{align}

The estimated thresholds
${\hat t}_{\alpha,s}$ and ${\hat t}_{\beta,s}$ are selected from the
PU instance scores and are therefore finite.
Thus, Lemma~\ref{lemma:hard_pairwise_terms} can be applied at these
data-dependent thresholds.
Using Eqs.~\eqref{eq:A_error_bound} and \eqref{eq:B_error_bound},
the first term in Eq.~\eqref{eq:pauc_error_decomposition} is bounded by
\begin{align}
&
\left|
{\widehat {\rm pAUC}}_s(\alpha,\beta)
-
Q_s({\hat t}_{\alpha,s},{\hat t}_{\beta,s})
\right|
\nonumber\\
&\le
\frac{1}{
(\beta-\alpha)(1-\pi)
}
\left[
2\epsilon_N
+
\epsilon_{N^{{\rm p}}}
+
3\pi
\frac{N^{{\rm p}}}{N^{{\rm p}}-1}
\epsilon_{N^{{\rm p}}}
+
\frac{\pi}{N^{{\rm p}}-1}
\right].
\label{eq:hard_empirical_error}
\end{align}
The bound in Lemma~\ref{lemma:hard_pairwise_terms} is uniform over
the finite endpoints $a$ and $b$.
Thus, it also holds when the estimated thresholds and the empirical
pAUC are calculated from the same instances.

For the second term, changing either endpoint of the score interval
affects Eq.~\eqref{eq:Q_negative} only through negative instances whose
scores are between the original and changed endpoints.
Since the pairwise indicator is bounded by one,
\begin{align}
&
\left|
Q_s({\hat t}_{\alpha,s},{\hat t}_{\beta,s})
-
Q_s(t_{\alpha,s},t_{\beta,s})
\right|
\nonumber\\
&\le
\frac{1}{\beta-\alpha}
\left[
\left|
{\rm FPR}_s({\hat t}_{\alpha,s})
-
{\rm FPR}_s(t_{\alpha,s})
\right|
+
\left|
{\rm FPR}_s({\hat t}_{\beta,s})
-
{\rm FPR}_s(t_{\beta,s})
\right|
\right].
\end{align}
Under Assumption~\ref{assump:no_point_mass},
${\rm FPR}_s(t)$ is continuous.
Thus, for finite population thresholds,
\begin{align}
{\rm FPR}_s(t_{\gamma,s})=\gamma.
\end{align}
Together with the convention
${\rm FPR}_s(+\infty)=0$ when $t_{\alpha,s}=+\infty$, we have
\begin{align}
{\rm FPR}_s(t_{\alpha,s})=\alpha,
\qquad
{\rm FPR}_s(t_{\beta,s})=\beta.
\end{align}
Then, Theorem~\ref{thm:threshold_fpr} gives
\begin{align}
&
\left|
Q_s({\hat t}_{\alpha,s},{\hat t}_{\beta,s})
-
Q_s(t_{\alpha,s},t_{\beta,s})
\right|
\nonumber\\
&\le
\frac{1}{\beta-\alpha}
\left[
\left|
{\rm FPR}_s({\hat t}_{\alpha,s})-\alpha
\right|
+
\left|
{\rm FPR}_s({\hat t}_{\beta,s})-\beta
\right|
\right]
\nonumber\\
&\le
\frac{2\epsilon_{\rm FPR}}{\beta-\alpha}.
\label{eq:hard_threshold_error}
\end{align}
Combining Eqs.~\eqref{eq:hard_empirical_error} and
\eqref{eq:hard_threshold_error}, and substituting
Eq.~\eqref{eq:epsilon_fpr}, proves
Eq.~\eqref{eq:pauc_finite_sample_bound}.
\end{proof}

For a finite score-function class with $|\mathcal{S}|=M$, the DKW
bounds in Section~\ref{subsec:theory_uniform} and a union bound show
that the conditions of Theorem~\ref{thm:pauc_finite_sample} hold
simultaneously with probability at least $1-\delta$.
Therefore,
\begin{align}
&
\sup_{s\in\mathcal{S}}
\left|
{\widehat {\rm pAUC}}_s(\alpha,\beta)
-
{\rm pAUC}_s(\alpha,\beta)
\right|
\nonumber\\
&=
O\left(
\frac{1}{
(\beta-\alpha)(1-\pi)
}
\left[
\sqrt{\frac{\log(M/\delta)}{N}}
+
\sqrt{\frac{\log(M/\delta)}{N^{{\rm p}}}}
+
\frac{1}{N^{{\rm p}}}
\right]
\right)
\label{eq:hard_finite_class_rate}
\end{align}
with probability at least $1-\delta$.
Thus, the empirical pAUC estimator is uniformly
consistent as $N,N^{{\rm p}}\to\infty$.
The same conclusion applies to infinite score-function classes when
the corresponding uniform errors in
Eq.~\eqref{eq:uniform_conditions} converge to zero.
Since the bound is uniform over $s\in\mathcal{S}$, it also applies to
a score function selected from $\mathcal{S}$ using the training data.

\subsection{Effect of Class-prior Estimation Error}
\label{subsec:theory_class_prior}

We finally analyze the effect of using an inaccurate class-prior at the
population level.
Let $\widetilde\pi\in[0,1)$ be the class-prior value used in the
proposed method, while $\pi$ is the true class-prior.
Replacing $\pi$ with $\widetilde\pi$ in Eq.~\eqref{eq:fpr_pu} gives
\begin{align}
{\rm FPR}_{\widetilde\pi,s}(t)
:=
\frac{
\overline F_s(t)
-
\widetilde\pi\overline F_s^{{\rm p}}(t)
}{
1-\widetilde\pi
}.
\label{eq:fpr_inaccurate_prior}
\end{align}

\begin{proposition}
\label{prop:class_prior_fpr}
For every $s\in\mathcal{S}$ and $t\in\mathbb{R}$,
\begin{align}
&
{\rm FPR}_{\widetilde\pi,s}(t)
-
{\rm FPR}_s(t)
\nonumber\\
&=
\frac{
\pi-\widetilde\pi
}{
1-\widetilde\pi
}
\left[
{\rm TPR}_s(t)
-
{\rm FPR}_s(t)
\right].
\label{eq:class_prior_fpr_identity}
\end{align}
Consequently,
\begin{align}
\sup_{s\in\mathcal{S},\,t\in\mathbb{R}}
\left|
{\rm FPR}_{\widetilde\pi,s}(t)
-
{\rm FPR}_s(t)
\right|
\le
\epsilon_\pi,
\qquad
\epsilon_\pi
:=
\frac{
|\widetilde\pi-\pi|
}{
1-\widetilde\pi
}.
\label{eq:class_prior_fpr_bound}
\end{align}
\end{proposition}

\begin{proof}
From Eq.~\eqref{eq:mar_train2},
\begin{align}
\overline F_s(t)
=
\pi{\rm TPR}_s(t)
+
(1-\pi){\rm FPR}_s(t).
\end{align}
Substituting this into Eq.~\eqref{eq:fpr_inaccurate_prior} gives
\begin{align}
{\rm FPR}_{\widetilde\pi,s}(t)
=
\frac{
(\pi-\widetilde\pi){\rm TPR}_s(t)
+
(1-\pi){\rm FPR}_s(t)
}{
1-\widetilde\pi
}.
\end{align}
Subtracting ${\rm FPR}_s(t)$ gives
Eq.~\eqref{eq:class_prior_fpr_identity}.
Since ${\rm TPR}_s(t),{\rm FPR}_s(t)\in[0,1]$,
Eq.~\eqref{eq:class_prior_fpr_bound} follows.
\end{proof}

For $b\in\mathbb{R}$ and
$a\in\mathbb{R}\cup\{+\infty\}$ with $b\le a$, define
\begin{align}
Q_{\widetilde\pi,s}(a,b)
:=
\frac{
A_s(a,b)
-
\widetilde\pi B_s(a,b)
}{
(\beta-\alpha)(1-\widetilde\pi)
}.
\label{eq:Q_inaccurate_prior}
\end{align}
Let
\begin{align}
\widetilde t_{\gamma,s}
:=
\inf
\left\{
t\in\mathbb{R}
:
{\rm FPR}_{\widetilde\pi,s}(t)
\le
\gamma
\right\},
\qquad
\gamma\in\{\alpha,\beta\},
\label{eq:threshold_inaccurate_prior}
\end{align}
where we set $\inf\emptyset=+\infty$.
Since $\alpha<\beta$, every threshold satisfying
${\rm FPR}_{\widetilde\pi,s}(t)\le\alpha$ also satisfies
${\rm FPR}_{\widetilde\pi,s}(t)\le\beta$.
Therefore,
\begin{align}
\widetilde t_{\beta,s}
\le
\widetilde t_{\alpha,s}.
\end{align}

\begin{theorem}
\label{thm:class_prior_pauc}
Suppose that Assumption~\ref{assump:no_point_mass} holds and
$0\le\alpha<\beta<1$.
Then, for every $s\in\mathcal{S}$,
\begin{align}
&
\left|
Q_{\widetilde\pi,s}
(
\widetilde t_{\alpha,s},
\widetilde t_{\beta,s}
)
-
{\rm pAUC}_s(\alpha,\beta)
\right|
\nonumber\\
&\le
\frac{
3|\widetilde\pi-\pi|
}{
(\beta-\alpha)(1-\widetilde\pi)
}.
\label{eq:class_prior_pauc_bound}
\end{align}
\end{theorem}

\begin{proof}
We fix $s$.
From Proposition~\ref{prop:class_prior_fpr},
\begin{align}
\sup_{t\in\mathbb{R}}
\left|
{\rm FPR}_{\widetilde\pi,s}(t)
-
{\rm FPR}_s(t)
\right|
\le
\epsilon_\pi.
\label{eq:class_prior_uniform_fixed_s}
\end{align}

Since
${\rm FPR}_{\widetilde\pi,s}(t)\to1$ as $t\to-\infty$
and $\gamma<1$, there exists a finite $T$ such that
${\rm FPR}_{\widetilde\pi,s}(t)>\gamma$ for all $t<T$.
Therefore, $\widetilde t_{\gamma,s}=-\infty$ cannot occur.
Suppose first that $\widetilde t_{\gamma,s}$ is finite.
Since the survival probabilities
$\overline F_s(t)$ and $\overline F_s^{{\rm p}}(t)$ are
right-continuous in $t$,
${\rm FPR}_{\widetilde\pi,s}(t)$ is also right-continuous.
Therefore, from the definition of
$\widetilde t_{\gamma,s}$,
\begin{align}
{\rm FPR}_{\widetilde\pi,s}
(\widetilde t_{\gamma,s})
\le
\gamma.
\end{align}
Also, for every $t<\widetilde t_{\gamma,s}$,
\begin{align}
{\rm FPR}_{\widetilde\pi,s}(t)>\gamma.
\end{align}
Together with
Eq.~\eqref{eq:class_prior_uniform_fixed_s}, these inequalities give
\begin{align}
{\rm FPR}_s(\widetilde t_{\gamma,s})
&\le
\gamma+\epsilon_\pi,
\\
{\rm FPR}_s(t)
&>
\gamma-\epsilon_\pi
\qquad
\text{for }t<\widetilde t_{\gamma,s}.
\end{align}
Assumption~\ref{assump:no_point_mass} ensures that
${\rm FPR}_s(t)$ is continuous.
Therefore, taking
$t\uparrow\widetilde t_{\gamma,s}$ in the second inequality gives
\begin{align}
{\rm FPR}_s(\widetilde t_{\gamma,s})
\ge
\gamma-\epsilon_\pi.
\end{align}
Thus,
\begin{align}
\left|
{\rm FPR}_s(\widetilde t_{\gamma,s})
-
\gamma
\right|
\le
\epsilon_\pi.
\label{eq:class_prior_threshold_bound}
\end{align}
If $\widetilde t_{\gamma,s}=+\infty$, then necessarily $\gamma=0$.
With the convention
${\rm FPR}_s(+\infty)=0$,
Eq.~\eqref{eq:class_prior_threshold_bound} also holds.
Therefore, Eq.~\eqref{eq:class_prior_threshold_bound} holds for
$\gamma\in\{\alpha,\beta\}$.

As in the proof of Theorem~\ref{thm:pauc_finite_sample},
Eq.~\eqref{eq:Q_negative} and the fact that the pairwise indicator is
bounded by one give
\begin{align}
&
\left|
Q_s(
\widetilde t_{\alpha,s},
\widetilde t_{\beta,s}
)
-
{\rm pAUC}_s(\alpha,\beta)
\right|
\nonumber\\
&\le
\frac{1}{\beta-\alpha}
\left[
\left|
{\rm FPR}_s(\widetilde t_{\alpha,s})-\alpha
\right|
+
\left|
{\rm FPR}_s(\widetilde t_{\beta,s})-\beta
\right|
\right]
\nonumber\\
&\le
\frac{2\epsilon_\pi}{\beta-\alpha}.
\label{eq:class_prior_threshold_error}
\end{align}

Next, from Eq.~\eqref{eq:mar_train2},
\begin{align}
A_s(a,b)
=
\pi B_s(a,b)
+
(1-\pi)C_s(a,b),
\end{align}
where
\begin{align}
C_s(a,b)
:=
&
\mathbb{E}_{{\bf x}^{{\rm p}}\sim p^{{\rm p}}}
\mathbb{E}_{{\bf x}^{{\rm n}}\sim p^{{\rm n}}}
\left[
I(s({\bf x}^{{\rm p}})>s({\bf x}^{{\rm n}}))
I(s({\bf x}^{{\rm n}})\in[b,a])
\right].
\end{align}
Substituting
$A_s(a,b)=\pi B_s(a,b)+(1-\pi)C_s(a,b)$
into $Q_{\widetilde\pi,s}$ and $Q_s$ gives
\begin{align}
Q_{\widetilde\pi,s}(a,b)
&=
\frac{
(\pi-\widetilde\pi)B_s(a,b)
+
(1-\pi)C_s(a,b)
}{
(\beta-\alpha)(1-\widetilde\pi)
},
\\
Q_s(a,b)
&=
\frac{C_s(a,b)}{\beta-\alpha}.
\end{align}
Therefore,
\begin{align}
&
Q_{\widetilde\pi,s}(a,b)
-
Q_s(a,b)
\nonumber\\
&=
\frac{
\pi-\widetilde\pi
}{
(\beta-\alpha)(1-\widetilde\pi)
}
\left[
B_s(a,b)-C_s(a,b)
\right].
\end{align}
Since $B_s(a,b),C_s(a,b)\in[0,1]$,
\begin{align}
\left|
Q_{\widetilde\pi,s}(a,b)
-
Q_s(a,b)
\right|
\le
\frac{\epsilon_\pi}{\beta-\alpha}.
\label{eq:class_prior_rewriting_error}
\end{align}
Applying Eq.~\eqref{eq:class_prior_rewriting_error} at
\begin{align}
(a,b)
=
(\widetilde t_{\alpha,s},\widetilde t_{\beta,s})
\end{align}
and combining it with
Eq.~\eqref{eq:class_prior_threshold_error} gives
\begin{align}
&
\left|
Q_{\widetilde\pi,s}
(
\widetilde t_{\alpha,s},
\widetilde t_{\beta,s}
)
-
{\rm pAUC}_s(\alpha,\beta)
\right|
\nonumber\\
&\le
\frac{3\epsilon_\pi}{\beta-\alpha}
=
\frac{
3|\widetilde\pi-\pi|
}{
(\beta-\alpha)(1-\widetilde\pi)
},
\end{align}
which proves Eq.~\eqref{eq:class_prior_pauc_bound}.
\end{proof}

Theorem~\ref{thm:class_prior_pauc} shows that the population-level
error caused by an inaccurate class-prior is controlled by
$|\widetilde\pi-\pi|/(1-\widetilde\pi)$.

\section{Experimental Settings}
\label{subsec:settings}

For MNIST, FMNIST, SUSY, CODRNA, IJCNN, EPSILON,
we used a three-layered feed-forward neural network with ReLU activations for the classifier (score function),
where each hidden layer consisted of $128$ units.
For SVHN and CIFAR10, we used a convolutional neural network composed of two convolutional blocks followed by a three-layer feed-forward neural network for the classifier (score function). 
The first (second) convolutional block comprised a 6 (16) filter $5 \times 5$ convolution, the ReLU activation, and a $2 \times 2$ max-pooling layer.
The numbers of nodes in the two hidden layers were
$120$ and $84$, and ReLU activations were used.
For all comparison methods, the same neural network architecture was used.
The PU mini-batch sizes $P$ and $U$ were set to $50$ and $1,000$, respectively.
For all methods, we used the Adam optimizer
\citep{kingma2014adam}
with a learning rate of $10^{-3}$. 
The maximum number of epochs was $200$.
The loss with validation PU data was used for early-stopping to mitigate overfitting.
All methods were implemented using PyTorch
\citep{paszke2017automatic},
and all experiments were conducted on a Linux server with an Intel Xeon CPU and A100 GPU.

\section{Additional Experimental Results}

\subsection{Results with $\alpha=0.1$ and $\beta=0.2$}
\label{apend:adres}

Table \ref{table:all_results_3} shows the test average pAUCs and their standard deviations with $\alpha=0.1$, $\beta=0.2$.
The proposed method performed well in many cases.
Compared to case $\alpha=0$ (Tables \ref{table:all_results_1} and \ref{table:all_results_2}),
PAUC and 2step tended to perform well.
We describe possible reasons for their good performance below.
These methods treat unlabeled data (or pseudo N data) whose scores fall within a certain range defined by $\alpha$ and $\beta$ 
as negatives for training. In general, unlabeled data with high scores are more likely to be positive. 
Therefore, when $\alpha=0$, such high score positive data tend to fall into the range, making it difficult for these methods to perform strongly. 
In contrast, when $\alpha>0$, these positive data in the unlabeled data are less likely to be included in the range, which can lead to better performance. 
Nevertheless, the proposed method performed comparably to or better than these methods.
Note that the proposed method outperformed these methods by a large margin in the case of $\alpha=0$ as shown in Tables \ref{table:all_results_1} and \ref{table:all_results_2}.

\begin{table*}[t]
\caption{Test average pAUCs and their standard deviations with $\alpha=0.1$, $\beta=0.2$ over different labeled positive data sizes $N^{\rm p} = 30$, $50$, and $100$.
Values in bold are not statistically different at the $5\%$ level from the best performing method in each row according to a paired t-test.
`\# Best' row represents the number of results of each method that are the best or comparable to it.}
\label{table:all_results_3}
\centering
\centering
\scalebox{0.78}{
\begin{tabular}{l r r r r r r r}
\hline
Data & $\pi$ & Ours & NNPU & PUAUC & PAUC & PUpAUCh & 2step \\
\hline
MNIST   & 0.1 & \textbf{0.838 $\pm$ 0.043} & 0.698 $\pm$ 0.126  & 0.823 $\pm$ 0.041  & \textbf{0.839 $\pm$ 0.048}  & \textbf{0.839 $\pm$ 0.042}  & 0.826 $\pm$ 0.036  \\
        & 0.2 & \textbf{0.821 $\pm$ 0.048} & 0.776 $\pm$ 0.063 & 0.800 $\pm$ 0.046  & 0.776 $\pm$ 0.041 & \textbf{0.818 $\pm$ 0.050} & 0.808 $\pm$ 0.050  \\
FMNIST  & 0.1 & 0.958 $\pm$ 0.021 & 0.795 $\pm$ 0.011 & 0.955 $\pm$ 0.020 & \textbf{0.964 $\pm$ 0.019} & 0.958 $\pm$ 0.021 & \textbf{0.963 $\pm$ 0.018} \\
        & 0.2 & \textbf{0.954 $\pm$ 0.024} & 0.806 $\pm$ 0.034 & 0.943 $\pm$ 0.022 & 0.910 $\pm$ 0.037 & 0.951 $\pm$ 0.025 & \textbf{0.958 $\pm$ 0.018} \\
SVHN    & 0.1 & 0.423 $\pm$ 0.098 & 0.154 $\pm$ 0.006 & 0.402 $\pm$ 0.092 & \textbf{0.451 $\pm$ 0.104} & 0.377 $\pm$ 0.092 & 0.410 $\pm$ 0.097 \\
        & 0.2 & 0.360 $\pm$ 0.104 & 0.154 $\pm$ 0.007 & 0.353 $\pm$ 0.074 & \textbf{0.391 $\pm$ 0.085} & 0.258 $\pm$ 0.063 &  0.351 $\pm$ 0.081 \\
CIFAR10 & 0.1 & \textbf{0.744 $\pm$ 0.063} & 0.127 $\pm$ 0.077 & 0.731 $\pm$ 0.069 & 0.721 $\pm$ 0.062 & \textbf{0.743 $\pm$ 0.056} & \textbf{0.752 $\pm$ 0.063} \\
        & 0.2 & \textbf{0.722 $\pm$ 0.071} & 0.162 $\pm$ 0.118 & 0.698 $\pm$ 0.080 & 0.623 $\pm$ 0.081 & 0.697 $\pm$ 0.069 & \textbf{0.712 $\pm$ 0.081} \\
SUSY    & 0.1 & \textbf{0.647 $\pm$ 0.029} & 0.557 $\pm$ 0.074 & 0.629 $\pm$ 0.035 & \textbf{0.651 $\pm$ 0.024} & \textbf{0.642 $\pm$ 0.030} & 0.632 $\pm$ 0.034 \\
        & 0.2 & \textbf{0.629 $\pm$ 0.038} & 0.578 $\pm$ 0.047 & 0.610 $\pm$ 0.043 & 0.606 $\pm$ 0.044 & \textbf{0.628 $\pm$ 0.035} & 0.617 $\pm$ 0.040 \\
CODRNA  & 0.1 & 0.959 $\pm$ 0.025 & 0.454 $\pm$ 0.319 & 0.951 $\pm$ 0.030 & 0.949 $\pm$ 0.033 & 0.959 $\pm$ 0.027 & \textbf{0.965 $\pm$ 0.023} \\        
        & 0.2 & 0.946 $\pm$ 0.037 & 0.892 $\pm$ 0.056 & 0.935 $\pm$ 0.037 & 0.853 $\pm$ 0.061 & 0.949 $\pm$ 0.031 & \textbf{0.959 $\pm$ 0.028} \\
IJCNN   & 0.1 & \textbf{0.792 $\pm$ 0.100} & 0.469 $\pm$ 0.108 & 0.778 $\pm$ 0.086 & \textbf{0.803 $\pm$ 0.078} & 0.760 $\pm$ 0.125 & \textbf{0.811 $\pm$ 0.086} \\
        & 0.2 & 0.730 $\pm$ 0.136 & 0.636 $\pm$ 0.098 & \textbf{0.750 $\pm$ 0.094} & 0.714 $\pm$ 0.086 & 0.620 $\pm$ 0.217 & \textbf{0.769 $\pm$ 0.101} \\
EPSILON   & 0.1 & \textbf{0.411 $\pm$ 0.065} & 0.177 $\pm$ 0.008 & \textbf{0.404 $\pm$ 0.067} & \textbf{0.410 $\pm$ 0.060} & \textbf{0.406 $\pm$ 0.068} & \textbf{0.399 $\pm$ 0.079}  \\
        & 0.2 & \textbf{0.360 $\pm$ 0.063} & 0.184 $\pm$ 0.015 & \textbf{0.368 $\pm$ 0.062} & \textbf{0.363 $\pm$ 0.048} & 0.295 $\pm$ 0.066 & \textbf{0.368 $\pm$ 0.071} \\
\hline
\# Best   & & 10 & 0 & 3 & 8 & 6 & 10 \\
\hline
\end{tabular}
}
\end{table*}

\subsection{Results with $\pi=0.01$ and $\pi=0.3$}
\label{apend:smaller_larger_priors}

In the main paper, we evaluated the proposed method with $\pi=0.1$ and $\pi=0.2$.
Here, we investigated the performance of the proposed method under smaller and larger class-prior values.
Table~\ref{table:all_results_smaller_larger} shows the average test pAUCs and their standard deviations with $\alpha=0$ , $\beta=0.2$. The proposed method performed well.
These results show that the proposed method is effective across a wide range of class-prior values.

\begin{table*}[t]
\caption{The results with $\pi=0.01$ and $\pi=0.3$: test average pAUCs and their standard deviations with $\alpha=0$, $\beta=0.2$ over different labeled positive data sizes $N^{\rm p} = 30$, $50$, and $100$.
Values in bold are not statistically different at the $5\%$ level from the best performing method in each row according to a paired t-test.
`\# Best' row represents the number of results of each method that are the best or comparable to it.}
\label{table:all_results_smaller_larger}
\centering
\scalebox{0.78}{
\begin{tabular}{l r r r r r r r}
\hline
Data & $\pi$ & Ours & NNPU & PUAUC & PAUC & PUpAUCh & 2step \\
\hline
MNIST   & 0.01 & \textbf{0.799 $\pm$ 0.054} & 0.338 $\pm$ 0.009  & 0.768 $\pm$ 0.055  & 0.794 $\pm$ 0.052  & \textbf{0.799 $\pm$ 0.053}  & 0.792 $\pm$ 0.055  \\
FMNIST  & 0.01 & 0.949 $\pm$ 0.021 & 0.731 $\pm$ 0.112 & 0.945 $\pm$ 0.021 & 0.947 $\pm$ 0.022 & 0.949 $\pm$ 0.021 & \textbf{0.954 $\pm$ 0.017} \\
SVHN    & 0.01 & \textbf{0.411 $\pm$ 0.089} & 0.105 $\pm$ 0.005 & 0.363 $\pm$ 0.089 & \textbf{0.410 $\pm$ 0.089} & \textbf{0.412 $\pm$ 0.083} & 0.398 $\pm$ 0.087 \\
CIFAR10 & 0.01 & 0.661 $\pm$ 0.067 & 0.081 $\pm$ 0.053 & 0.634 $\pm$ 0.079 & 0.662 $\pm$ 0.072 & 0.656 $\pm$ 0.070 & \textbf{0.679 $\pm$ 0.063} \\
SUSY    & 0.01 & \textbf{0.571 $\pm$ 0.022} & 0.295 $\pm$ 0.038 & 0.528 $\pm$ 0.041 & \textbf{0.570 $\pm$ 0.022} & \textbf{0.569 $\pm$ 0.022} & 0.552 $\pm$ 0.033 \\
CODRNA  & 0.01 & \textbf{0.882 $\pm$ 0.025} & 0.124 $\pm$ 0.031 & 0.875 $\pm$ 0.032 & \textbf{0.883 $\pm$ 0.024} & \textbf{0.882 $\pm$ 0.024} & \textbf{0.881 $\pm$ 0.027} \\
IJCNN   & 0.01 & \textbf{0.766 $\pm$ 0.070} & 0.188 $\pm$ 0.023 & 0.714 $\pm$ 0.085 & \textbf{0.765 $\pm$ 0.067} & 0.763 $\pm$ 0.070 & 0.753 $\pm$ 0.077 \\
EPSILON   & 0.01 & \textbf{0.353 $\pm$ 0.055} & 0.131 $\pm$ 0.005 & 0.324 $\pm$ 0.059 & 0.350 $\pm$ 0.052 & \textbf{0.354 $\pm$ 0.054} & 0.338 $\pm$ 0.064  \\
\hline
MNIST   & 0.3 & \textbf{0.688 $\pm$ 0.068} & 0.665 $\pm$ 0.067  & 0.663 $\pm$ 0.051  & 0.565 $\pm$ 0.061  & \textbf{0.689 $\pm$ 0.060}  & 0.664 $\pm$ 0.052  \\
FMNIST  & 0.3 & \textbf{0.923 $\pm$ 0.029} & 0.824 $\pm$ 0.070 & 0.883 $\pm$ 0.042 & 0.747 $\pm$ 0.083 & \textbf{0.922 $\pm$ 0.029} & 0.888 $\pm$ 0.042 \\
SVHN    & 0.3 & \textbf{0.239 $\pm$ 0.079} & 0.105 $\pm$ 0.005 & \textbf{0.233 $\pm$ 0.056} & \textbf{0.237 $\pm$ 0.045} & 0.145 $\pm$ 0.032 & \textbf{0.229 $\pm$ 0.056} \\
CIFAR10 & 0.3 & \textbf{0.544 $\pm$ 0.079} & 0.229 $\pm$ 0.148 & 0.513 $\pm$ 0.083 & 0.373 $\pm$ 0.058 & \textbf{0.522 $\pm$ 0.080} & 0.518 $\pm$ 0.085 \\
SUSY    & 0.3 & \textbf{0.502 $\pm$ 0.050} & 0.465 $\pm$ 0.053 & 0.478 $\pm$ 0.051 & 0.366 $\pm$ 0.051 & 0.483 $\pm$ 0.043 & 0.474 $\pm$ 0.049 \\
CODRNA  & 0.3 & \textbf{0.824 $\pm$ 0.045} & 0.804 $\pm$ 0.050 & 0.802 $\pm$ 0.060 & 0.623 $\pm$ 0.094 & \textbf{0.819 $\pm$ 0.053} & 0.808 $\pm$ 0.060 \\
IJCNN   & 0.3 & \textbf{0.580 $\pm$ 0.134} & 0.518 $\pm$ 0.136 & \textbf{0.593 $\pm$ 0.087} & 0.480 $\pm$ 0.066 & 0.415 $\pm$ 0.226 & \textbf{0.593 $\pm$ 0.134} \\
EPSILON   & 0.3 & 0.231 $\pm$ 0.054 & 0.180 $\pm$ 0.057 & \textbf{0.245 $\pm$ 0.044} & 0.213 $\pm$ 0.032 & 0.169 $\pm$ 0.046 & \textbf{0.245 $\pm$ 0.046}  \\
\hline
\# Best   & & 13 & 0 & 3 & 5 & 9 & 6 \\
\hline
\end{tabular}
}
\end{table*}

\subsection{Results using the Logistic Loss}
\label{apend:logistic}

In the main paper, we used the sigmoid function as the surrogate function in Eq. \eqref{eq:emp_pauc_pu_dif}.
Here, we evaluated the (p)AUC-based methods including the proposed method using the logistic loss as the surrogate function.
Table \ref{table:all_results_logistic} shows the results.
The results are similar to those obtained with the sigmoid function in Table \ref{table:all_results_2}, 
and the proposed method still outperforms the other methods.
These results suggest that the proposed method is effective with multiple surrogate functions.

\begin{table*}[t]
\caption{The results using the logistic loss: test average pAUCs and their standard deviations with $\alpha=0$, $\beta=0.2$ over different labeled positive data sizes $N^{\rm p} = 30$, $50$, and $100$.
Values in bold are not statistically different at the $5\%$ level from the best performing method in each row according to a paired t-test.
`\# Best' row represents the number of results of each method that are the best or comparable to it.}
\label{table:all_results_logistic}
\centering
\scalebox{0.78}{
\begin{tabular}{l r r r r r r}
\hline
Data & $\pi$ & Ours & PUAUC & PAUC & PUpAUCh & 2step \\
\hline
MNIST   & 0.1 & \textbf{0.759 $\pm$ 0.058}  & 0.727 $\pm$ 0.047  & 0.743 $\pm$ 0.057  & 0.752 $\pm$ 0.055  & 0.732 $\pm$ 0.048  \\
        & 0.2 & \textbf{0.727 $\pm$ 0.066} & 0.699 $\pm$ 0.050  & 0.652 $\pm$ 0.060 & \textbf{0.719 $\pm$ 0.057} & 0.700 $\pm$ 0.049  \\
FMNIST  & 0.1 & \textbf{0.945 $\pm$ 0.023} & 0.915 $\pm$ 0.034 & 0.933 $\pm$ 0.027 & \textbf{0.944 $\pm$ 0.024} & 0.929 $\pm$ 0.029 \\
        & 0.2 & \textbf{0.939 $\pm$ 0.024} & 0.907 $\pm$ 0.036 & 0.854 $\pm$ 0.057 & \textbf{0.936 $\pm$ 0.026} & 0.915 $\pm$ 0.034 \\
SVHN    & 0.1 & \textbf{0.346 $\pm$ 0.095} & 0.302 $\pm$ 0.070 & 0.328 $\pm$ 0.071 & 0.330 $\pm$ 0.096 & 0.308 $\pm$ 0.073 \\
        & 0.2 & \textbf{0.285 $\pm$ 0.091}  & 0.262 $\pm$ 0.061 & \textbf{0.268 $\pm$ 0.050} & 0.227 $\pm$ 0.069 & 0.259 $\pm$ 0.062 \\
CIFAR10 & 0.1 & \textbf{0.613 $\pm$ 0.070} & 0.592 $\pm$ 0.072 & 0.552 $\pm$ 0.078 & 0.602 $\pm$ 0.066 & \textbf{0.606 $\pm$ 0.068} \\
        & 0.2 & \textbf{0.579 $\pm$ 0.072} & 0.559 $\pm$ 0.079 & 0.452 $\pm$ 0.069 & \textbf{0.572 $\pm$ 0.065} & \textbf{0.564 $\pm$ 0.079} \\
SUSY    & 0.1 & \textbf{0.551 $\pm$ 0.033} & 0.516 $\pm$ 0.042 & 0.514 $\pm$ 0.048 & 0.541 $\pm$ 0.028 & 0.521 $\pm$ 0.040 \\
        & 0.2 & \textbf{0.525 $\pm$ 0.044} & 0.493 $\pm$ 0.043 & 0.427 $\pm$ 0.053 & 0.509 $\pm$ 0.042 & 0.495 $\pm$ 0.042 \\
CODRNA  & 0.1 & \textbf{0.853 $\pm$ 0.034}  & 0.849 $\pm$ 0.043 & 0.824 $\pm$ 0.048 & \textbf{0.854 $\pm$ 0.035} & \textbf{0.856 $\pm$ 0.038} \\
        & 0.2 & \textbf{0.837 $\pm$ 0.040} & 0.828 $\pm$ 0.049 & 0.707 $\pm$ 0.078 & \textbf{0.839 $\pm$ 0.039} & \textbf{0.834 $\pm$ 0.047} \\
IJCNN   & 0.1 & \textbf{0.712 $\pm$ 0.081} & 0.663 $\pm$ 0.089 & 0.685 $\pm$ 0.071 & 0.688 $\pm$ 0.100 & 0.680 $\pm$ 0.085 \\
        & 0.2 & \textbf{0.652 $\pm$ 0.122} & 0.628 $\pm$ 0.092 & 0.572 $\pm$ 0.069 & 0.595 $\pm$ 0.145 & \textbf{0.633 $\pm$ 0.094} \\
EPSILON   & 0.1 & \textbf{0.303 $\pm$ 0.058} & \textbf{0.297 $\pm$ 0.055} & 0.290 $\pm$ 0.047 & \textbf{0.305 $\pm$ 0.054} & 0.295 $\pm$ 0.056  \\
        & 0.2 & \textbf{0.272 $\pm$ 0.049} & \textbf{0.269 $\pm$ 0.048} & 0.249 $\pm$ 0.040 & \textbf{0.258 $\pm$ 0.054} & \textbf{0.269 $\pm$ 0.055} \\
\hline
\# Best   & & 16 & 2 & 1 & 8 & 6 \\
\hline
\end{tabular}
}
\end{table*}

\subsection{Results with Different Numbers of Unlabeled Training Data}
\label{apend:tru}

Figure \ref{fig:result_undata} shows the average test pAUCs with their standard errors when changing the number of unlabeled training data $N$.
The proposed method tended to perform better as $N$ increased and consistently performed well compared to other methods across different $N$.

\begin{figure*}[t]
    \centering
    \begin{minipage}{0.24\linewidth}
        \centering
        \includegraphics[width=3.3cm]{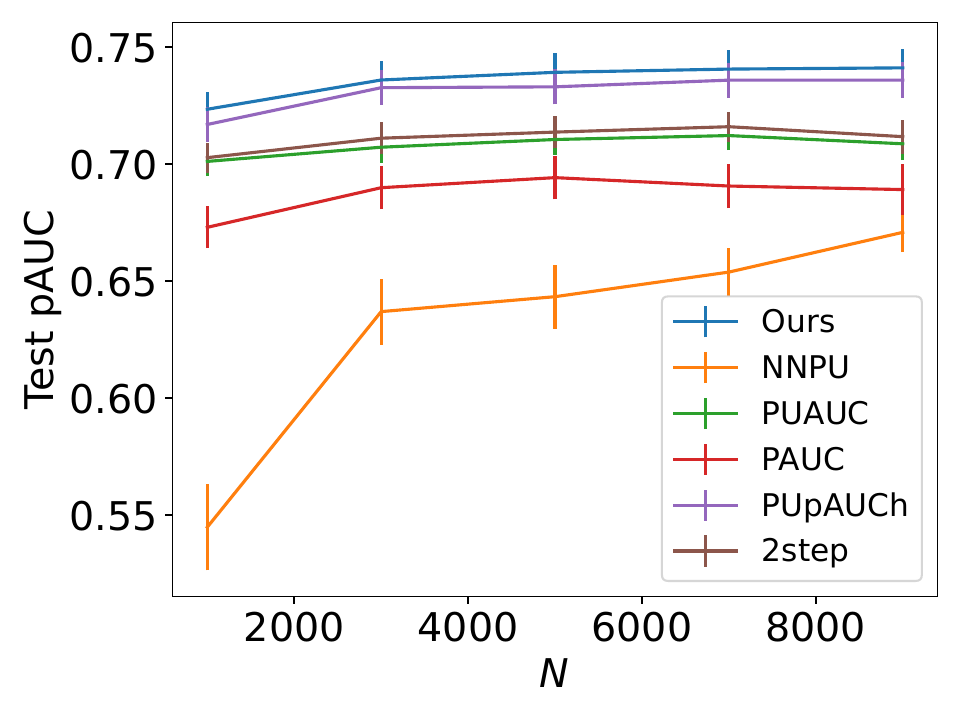}
        \subcaption{MNIST}
    \end{minipage}
    \hfill
    \begin{minipage}{0.24\linewidth}
        \centering
        \includegraphics[width=3.3cm]{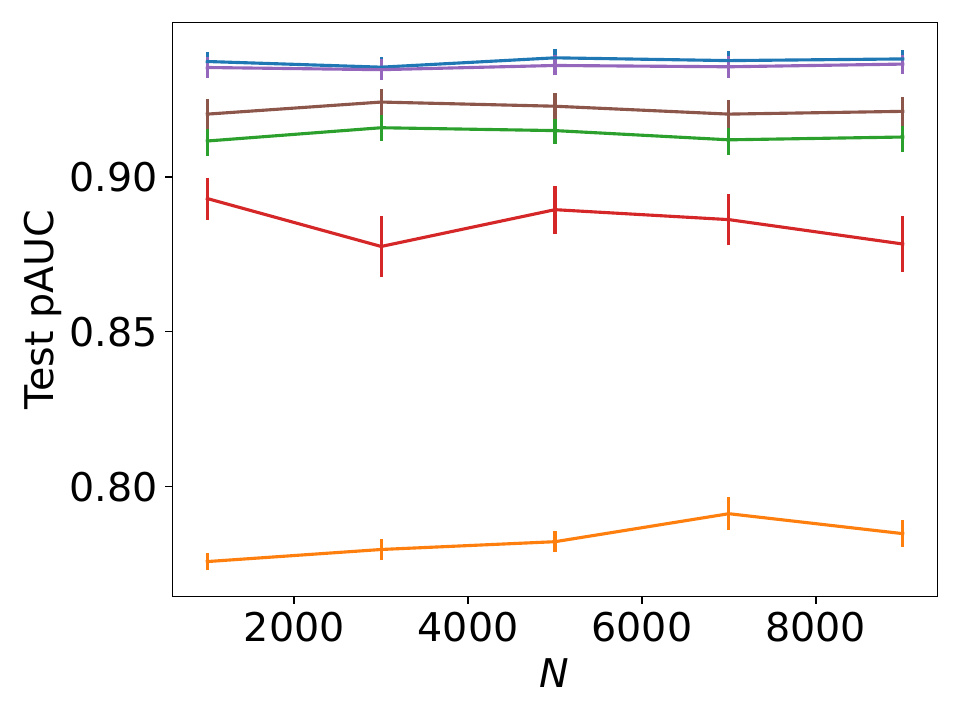}
        \subcaption{FMNIST}
    \end{minipage}
    \hfill
    \begin{minipage}{0.24\linewidth}
        \centering
        \includegraphics[width=3.3cm]{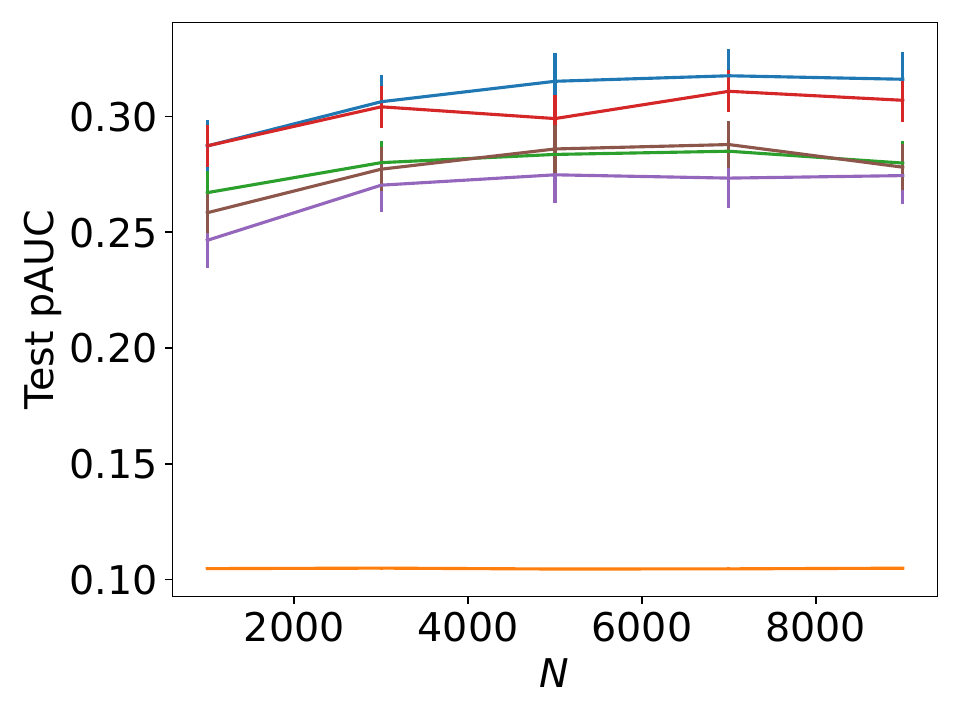}
        \subcaption{SVHN}
    \end{minipage}
    \hfill
    \begin{minipage}{0.24\linewidth}
        \centering
        \includegraphics[width=3.3cm]{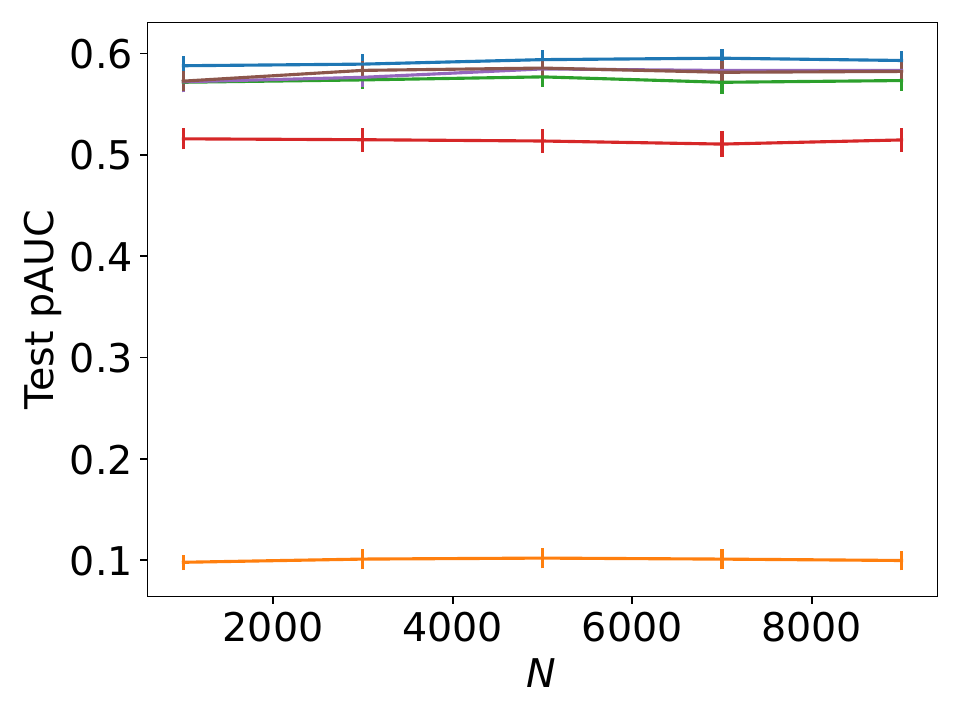}
        \subcaption{CIFAR10}
    \end{minipage}
    \hfill
    \begin{minipage}{0.24\linewidth}
        \centering
        \includegraphics[width=3.3cm]{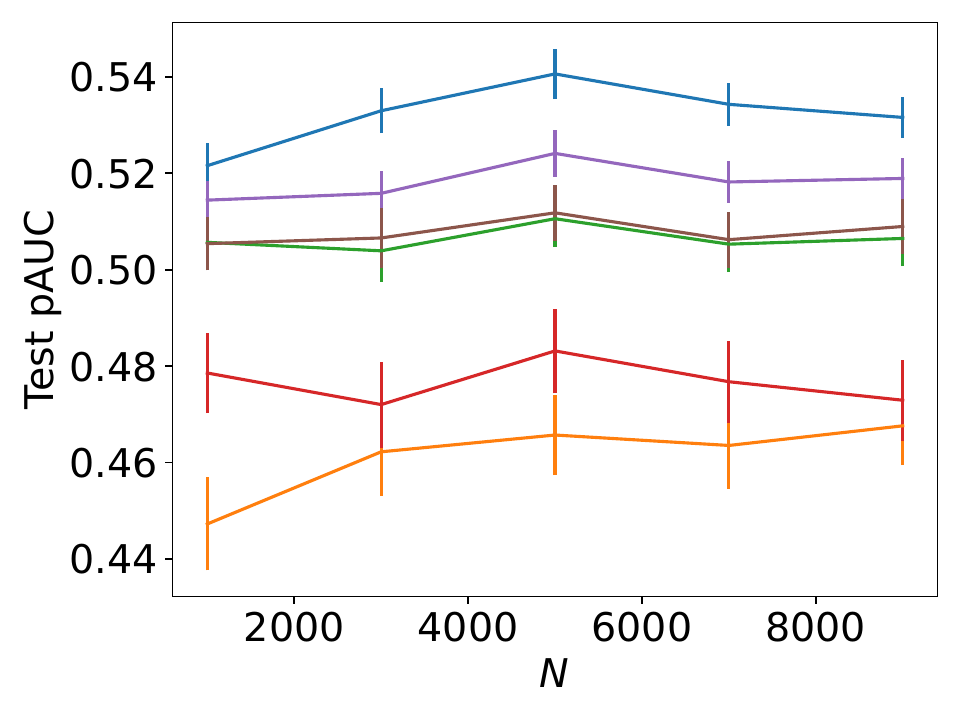}
        \subcaption{SUSY}
    \end{minipage}
    \hfill
    \begin{minipage}{0.24\linewidth}
        \centering
        \includegraphics[width=3.3cm]{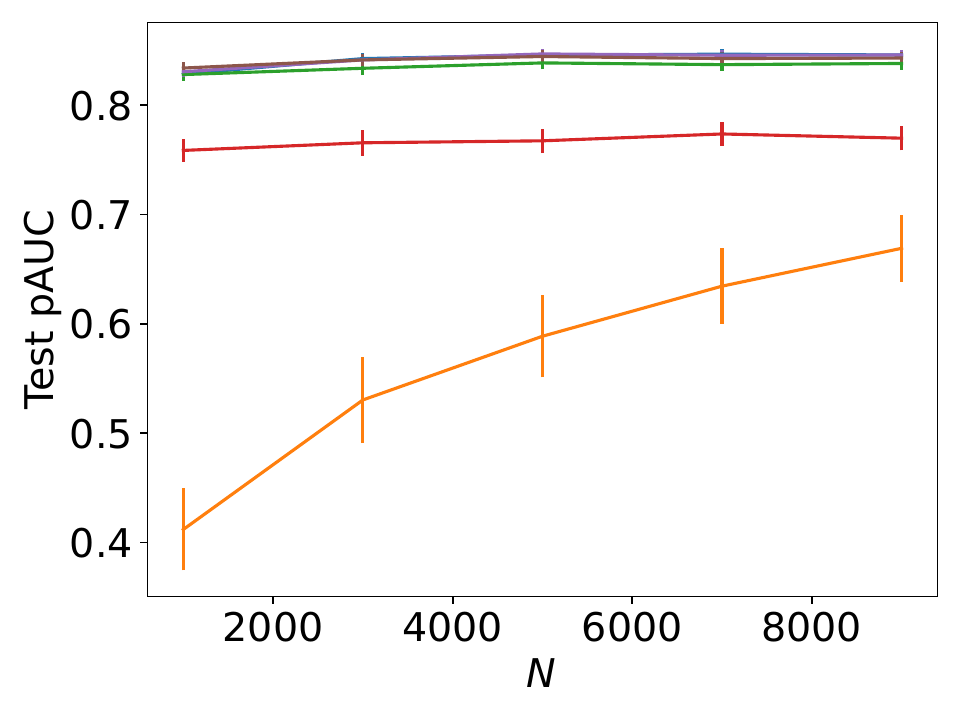}
        \subcaption{CODRNA}
    \end{minipage}
    \hfill
    \begin{minipage}{0.24\linewidth}
        \centering
        \includegraphics[width=3.3cm]{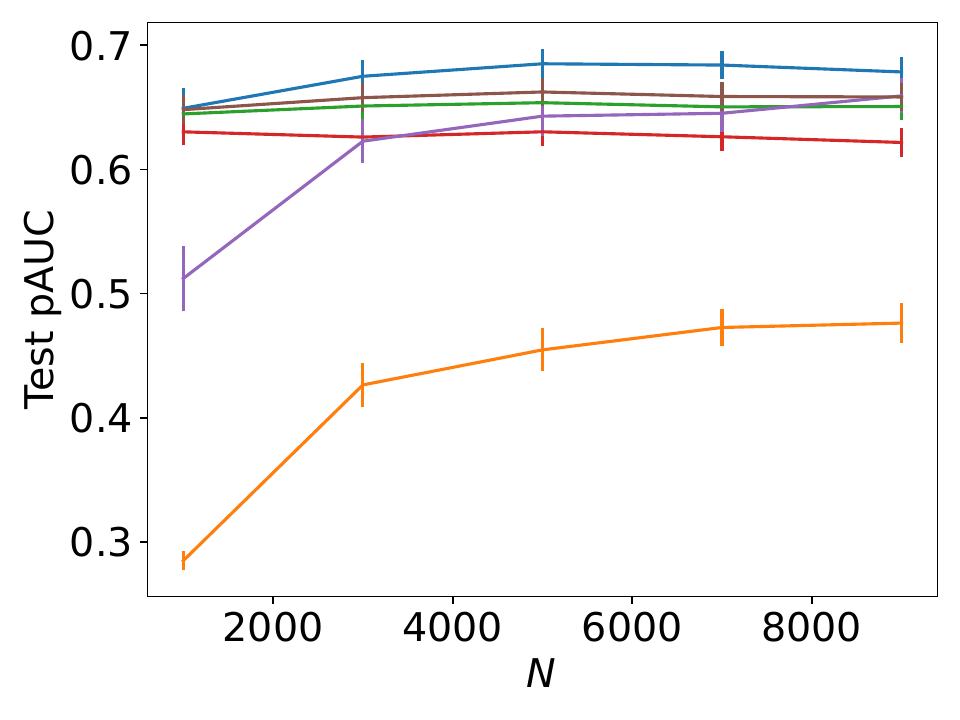}
        \subcaption{IJCNN}
    \end{minipage}
    \hfill
    \begin{minipage}{0.24\linewidth}
        \centering
        \includegraphics[width=3.3cm]{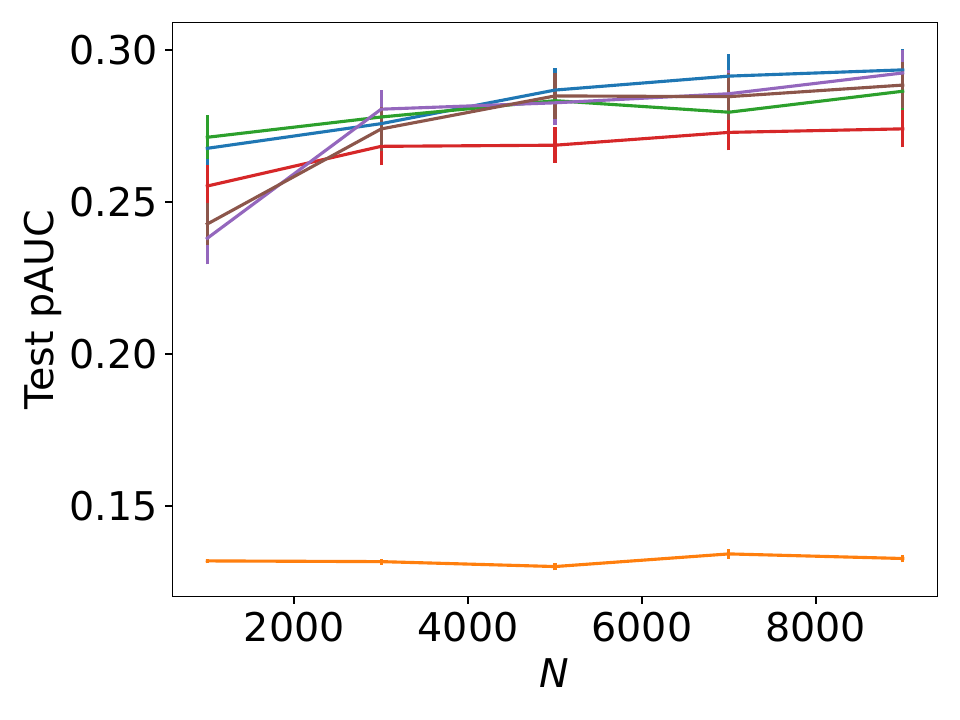}
        \subcaption{EPSILON}
    \end{minipage}
    \caption{The average test pAUCs and their standard errors with $\alpha=0$, $\beta=0.2$ over different class-priors $\pi \in \{0.1, 0.2\}$ and positive data sizes $N^{{\rm p}} \in \{ 30, 50, 100 \}$ when changing the number of unlabeled data $N$.}
    \label{fig:result_undata}
\end{figure*} 

\subsection{Computation Costs}
\label{append:comp_cost}

We investigated the training time of the proposed method on the MNIST and CIFAR10 datasets. 
We used a Linux server with a 2.20 GHz CPU.
Table \ref{table:comptime} shows the results.
Since 2step additionally learns the classifier to create pseudo PN data, its training time was long.
For the remaining methods, the training times were all similar.
This result indicates that the proposed method is sufficiently practical in terms of the computation cost.

\begin{table*}[t!]
\caption{Training time [sec] of each method on the MNIST and CIFAR10 datasets, where $(\alpha, \beta)=(0, 0.2)$, $\pi=0.2$, $N^{{\rm p}}=100$.}
\label{table:comptime}
\centering
\begin{tabular}{llcccccc}
\hline
Data & $N$ & Ours & NNPU & PUAUC & PAUC & PUpAUCh & 2step \\
\hline
MNIST & 5,000 & 90.460 & 90.054 & 90.036 & 90.978 & 89.618 & 184.040 \\
CIFAR10 & 9,000 & 305.871 & 304.184 & 302.897 & 305.471 & 302.469 & 608.849 \\
\hline
\end{tabular}
\end{table*}

\subsection{Evaluation on Real-World Medical and Advertising Datasets}
\label{sec:medical_advertising}

In Section \ref{sec:intro}, we presented examples from the medical and advertising domains to motivate our problem setting.
In this section, we evaluated the proposed method on Mammography
\citep{vanschoren2014openml} 
and Online-shoppers
\citep{sakar2019real}, 
which are real-world medical and advertising-related datasets, respectively.

Mammography is a dataset for detecting microcalcifications, which are treated as positives, and is highly class-imbalanced (i.e., the positive class-prior is $0.023$).
Due to the limited (positive) data size, we used $30$ and $20$ positive data for training and validation, respectively; $5,000$ and $1,000$ unlabeled data for training and validation, respectively; and $3,000$ test data. The unlabeled and test data were created to preserve the original class-prior value (i.e., $\pi=0.023$).
Online-shoppers is a dataset for predicting whether an online browsing session results in a purchase, with purchase sessions treated as positives. The positive class-prior is $0.155$.
This task is closely related to the advertising scenario motivating our work: it is important to identify sessions likely to result in purchases while maintaining a low FPR, since falsely targeting non-purchase sessions can waste advertising resources and degrade user experience.
For this dataset, we changed the positive training data size within $\{30, 50, 100\}$ as in the other datasets and used $1,000$ test data. The sizes of other datasets were the same as those for Mammography.
The unlabeled data were created to preserve the original class-prior (i.e., $\pi=0.155$).

Table \ref{table:mammo_online} shows the results.
On Mammography, the proposed method performed comparably to the other pAUC maximization methods (PAUC, PUpAUCh, and 2step).
Since this dataset has a very small positive class-prior, treating unlabeled instances as negative introduces only limited bias, making such approaches naturally competitive.
In contrast, on Online-shoppers, the proposed method consistently outperformed the other methods.
This dataset has a larger positive class-prior, making it more important to properly account for positive instances contained in the unlabeled data.
Such non-negligible class-priors can naturally arise when both training and test data are drawn from a candidate population enriched through upstream screening, filtering, or retrieval.
Overall, these results show that the proposed method is particularly effective in practically relevant settings where the unlabeled data contain a non-negligible proportion of positive instances, while remaining competitive even when the positive class-prior is very small.

\begin{table*}[t]
\caption{The results on Mammography and Online-shoppers: test average pAUCs and their standard deviations with $\alpha=0$, $\beta=0.1$ and $\alpha=0$, $\beta=0.2$.
Values in bold are not statistically different at the $5\%$ level from the best performing method in each row according to a paired t-test.}
\label{table:mammo_online}
\centering
\scalebox{0.76}{
\begin{tabular}{l r r r r r r r}
\hline
Data & $\beta$ & Ours & NNPU & PUAUC & PAUC & PUpAUCh & 2step \\
\hline
Mammography   & 0.1 & \textbf{0.784 $\pm$ 0.041} & 0.001 $\pm$ 0.001  & 0.754 $\pm$ 0.041 & \textbf{0.780 $\pm$ 0.045}  & \textbf{0.786 $\pm$ 0.037}  & \textbf{0.784 $\pm$ 0.037} \\
Mammography   & 0.2 & \textbf{0.821 $\pm$ 0.039} & 0.004 $\pm$ 0.004  & 0.810 $\pm$ 0.032 & \textbf{0.831 $\pm$ 0.039}  & \textbf{0.827 $\pm$ 0.036}  & \textbf{0.834 $\pm$ 0.029} \\
Online-shoppers & 0.1 & \textbf{0.329 $\pm$ 0.081} & 0.222 $\pm$ 0.072 & 0.259 $\pm$ 0.069 & 0.226 $\pm$ 0.032 & 0.308 $\pm$ 0.085 & 0.263 $\pm$ 0.071 \\
Online-shoppers & 0.2 & \textbf{0.442 $\pm$ 0.079} & 0.331 $\pm$ 0.079 & 0.371 $\pm$ 0.073 & 0.386 $\pm$ 0.063 & 0.425 $\pm$ 0.079 & 0.379 $\pm$ 0.073 \\
\hline
\end{tabular}
}
\end{table*}

\subsection{Empirical Analysis of FPRs Induced by the Estimated Thresholds}
\label{apen:fpr_eval}

We tracked the FPRs with the estimated thresholds, ${{\rm FPR}} (\hat{t}_{\alpha})$ and ${{\rm FPR}} (\hat{t}_{\beta})$, during training, and examined whether they remained close to the target values $\alpha$ and $\beta$.\footnote{We did not track the estimated thresholds themselves because they may change substantially when the scale or offset of the score function changes, even when the resulting pAUC remains unchanged.}
In this experiment, $\hat{t}_{\alpha}$ and $\hat{t}_{\beta}$ were estimated from training PU mini-batch as described in Algorithm \ref{arg}, and the FPRs induced by the estimated thresholds were calculated on different fixed $2,500$ negative instances.

Table \ref{tab:fpr_thresholds} shows the average FPRs induced by the estimated thresholds at selected training epochs for FMNIST and CODRNA, an image dataset and a tabular dataset, respectively.
For both datasets, ${\rm FPR}(\hat{t}_{\beta})$ remained close to the target value $\beta$ throughout training. In contrast, ${\rm FPR}(\hat{t}_{\alpha})$ showed larger deviations from the target $\alpha=0$, indicating that estimating the threshold corresponding to the extreme FPR ($\alpha=0$) is more challenging.
Nevertheless, increasing the number of positive training instances from $N^p=30$ to $N^p=100$ consistently reduced the deviations from the target FPRs for both thresholds.
These results suggest that the proposed method can appropriately estimate the target FPR region from PU data during training, and the estimation becomes more accurate as the number of positive training data increases.

\begin{table}[t]
    \centering
    \caption{FPRs induced by the estimated thresholds during training.}
    \label{tab:fpr_thresholds}
    \scalebox{0.85}{
    \begin{tabular}{lrrrlrrrr}
        \hline
        Data & $\alpha$ & $\beta$ & $N^p$ & FPR
        & Epoch 1 & Epoch 10 & Epoch 50 & Epoch 100 \\
        \hline
        FMNIST
        & 0.0 & 0.1 & 30
        & $\mathrm{FPR}(\hat{t}_{\alpha})$
        & 0.025 & 0.031 & 0.031 & 0.031 \\
        & 0.0 & 0.1 & 100
        & $\mathrm{FPR}(\hat{t}_{\alpha})$
        & 0.017 & 0.020 & 0.022 & 0.023 \\
        & 0.0 & 0.1 & 30
        & $\mathrm{FPR}(\hat{t}_{\beta})$
        & 0.118 & 0.115 & 0.115 & 0.116 \\
        & 0.0 & 0.1 & 100
        & $\mathrm{FPR}(\hat{t}_{\beta})$
        & 0.107 & 0.109 & 0.108 & 0.109 \\
        \cline{2-9}
        & 0.0 & 0.2 & 30
        & $\mathrm{FPR}(\hat{t}_{\alpha})$
        & 0.022 & 0.030 & 0.031 & 0.031 \\
        & 0.0 & 0.2 & 100
        & $\mathrm{FPR}(\hat{t}_{\alpha})$
        & 0.013 & 0.021 & 0.021 & 0.022 \\
        & 0.0 & 0.2 & 30
        & $\mathrm{FPR}(\hat{t}_{\beta})$
        & 0.217 & 0.208 & 0.208 & 0.208 \\
        & 0.0 & 0.2 & 100
        & $\mathrm{FPR}(\hat{t}_{\beta})$
        & 0.210 & 0.204 & 0.200 & 0.200 \\
        \hline
        CODRNA
        & 0.0 & 0.1 & 30
        & $\mathrm{FPR}(\hat{t}_{\alpha})$
        & 0.061 & 0.064 & 0.058 & 0.058 \\
        & 0.0 & 0.1 & 100
        & $\mathrm{FPR}(\hat{t}_{\alpha})$
        & 0.027 & 0.033 & 0.044 & 0.046 \\
        & 0.0 & 0.1 & 30
        & $\mathrm{FPR}(\hat{t}_{\beta})$
        & 0.140 & 0.138 & 0.129 & 0.130 \\
        & 0.0 & 0.1 & 100
        & $\mathrm{FPR}(\hat{t}_{\beta})$
        & 0.114 & 0.115 & 0.113 & 0.115 \\
        \cline{2-9}
        & 0.0 & 0.2 & 30
        & $\mathrm{FPR}(\hat{t}_{\alpha})$
        & 0.058 & 0.058 & 0.056 & 0.057 \\
        & 0.0 & 0.2 & 100
        & $\mathrm{FPR}(\hat{t}_{\alpha})$
        & 0.027 & 0.027 & 0.041 & 0.041 \\
        & 0.0 & 0.2 & 30
        & $\mathrm{FPR}(\hat{t}_{\beta})$
        & 0.223 & 0.216 & 0.214 & 0.215 \\
        & 0.0 & 0.2 & 100
        & $\mathrm{FPR}(\hat{t}_{\beta})$
        & 0.208 & 0.201 & 0.204 & 0.204 \\
        \hline
    \end{tabular}
    }
\end{table}

\subsection{Empirical Analysis of the Heuristic Scoring Assumption}
\label{apen:iwata_assumption}

The existing pAUC maximization method, PUpAUCh, trains the score function 
by heuristically treating all unlabeled instances whose ranks based on their current scores fall within a specified range as negatives
\citep{iwata2020semi}.
Here, we evaluate the quality of these pseudo-negatives to investigate to what extent their heuristic assumption that the selected unlabeled instances are negative holds during training.

Specifically, we first define
\begin{align}
e = \frac{|\{{\bf x} \in B:y=+1\}|}{|B|},
\end{align}
where $B$ denotes the set of selected pseudo-negative instances within the mini-batch unlabeled data.
Thus, $e$ represents the proportion of selected instances that are actually positive.
If instances are selected uniformly at random from the unlabeled data, the expected proportion of positives among them is $\pi$.
Thus, we define
\begin{align}
c = \frac{e}{\pi},
\end{align}
where $c=1$ indicates that the selection is random,
$c<1$ indicates that it is better than the random selection, and $c>1$ indicates that it is worse than the random selection.
A larger $c$ means a greater violation of the PUpAUCh's heuristic scoring assumption. 

Table \ref{tab:oracle_gap} reports the average $c$ during PUpAUCh training, together with the test pAUC gap from an oracle method that explicitly uses the true negative instances contained in the unlabeled data, i.e., Oracle $-$ PUpAUCh.
Here, we fixed $\pi=0.1$, $\alpha=0$, and $\beta=0.1$.
We first observe that $c>1$ for most datasets, indicating that the pseudo-negatives selected by PUpAUCh often contain more positive instances than expected under random selection.
This result empirically shows that the heuristic scoring assumption of PUpAUCh is frequently violated during training.
Moreover, we observe a positive association between $c$ and the performance gap between Oracle and PUpAUCh (Spearman's $\rho=0.714$).
This suggests that poorer pseudo-negative selection tends to be associated with a larger performance degradation relative to the oracle.

\begin{table}[t]
    \centering
    \caption{Pseudo-negative selection quality of PUpAUCh~\citep{iwata2020semi} and the test pAUC gap between Oracle and PUpAUCh. Larger $c$ indicates greater positive contamination in the selected pseudo-negatives relative to random selection.}
    \label{tab:oracle_gap}
    \begin{tabular}{lrr}
        \hline
        Data & $c$ & Oracle $-$ PUpAUCh \\
        \hline
        MNIST   & 1.5421 & 0.08379 \\
        FMNIST  & 0.6666 & 0.01786 \\
        SVHN    & 1.5432 & 0.09771 \\
        CIFAR10 & 1.5502 & 0.09186 \\
        SUSY    & 1.1720 & 0.04063 \\
        CODRNA  & 2.2437 & 0.04952 \\
        IJCNN   & 1.6922 & 0.13128 \\
        EPSILON & 1.5265 & 0.03976 \\
        \hline
    \end{tabular}
\end{table}

\end{document}